\documentclass{article}

\usepackage[nonatbib,final]{nips_2017}

\usepackage[T1]{fontenc}    % use 8-bit T1 fonts
\usepackage{hyperref}       % hyperlinks
\usepackage{url}            % simple URL typesetting
\usepackage{booktabs}       % professional-quality tables
\usepackage{amsfonts}       % blackboard math symbols
\usepackage{nicefrac}       % compact symbols for 1/2, etc.
\usepackage{microtype}      % microtypography
\usepackage{amssymb}
\usepackage{amsmath}
\usepackage{lmodern}
\usepackage{enumerate}
\usepackage{amsthm}
\usepackage[usenames]{color}
\usepackage{capt-of}
\usepackage{graphicx} % default *.eps
\usepackage{epstopdf}
\graphicspath{{fig_paper/}}
\usepackage{wrapfig}
\usepackage[font=scriptsize,labelfont=bf]{caption}
\usepackage{multirow}
\usepackage{caption}
\usepackage{algorithm}
\usepackage{algorithmic}

\newtheorem{lemma}{Lemma}[section]
\newtheorem{proposition}{Proposition}[section]
\newtheorem{theorem}{Theorem}[section]

\newcommand{\inner}[1]{\left\langle#1\right\rangle}

\def\R{\mathbb{R}}

\newcommand{\norm}[1]{\left\|#1\right\|}

\def\ones{\mathop{\rm e}\nolimits}

\def\argmax{\mathop{\rm arg\,max}\limits}%    a math operator.
\def\minop{\mathop{\rm min}\limits}
\def\maxop{\mathop{\rm max}\limits}

\def\min{\mathop{\rm min}\nolimits}
\def\max{\mathop{\rm max}\nolimits}
\def\ones{\mathbf{1}}

\newif\ifpaper
\papertrue % or % true is long, false is short
\title{Formal Guarantees on the Robustness of a Classifier against Adversarial Manipulation}

\author{
  Matthias Hein and Maksym Andriushchenko\\
  Department of Mathematics and Computer Science\\
  Saarland University, Saarbr{\"u}cken Informatics Campus, Germany\\
}

\begin{document}

\maketitle

\begin{abstract}
Recent work has shown that state-of-the-art classifiers are quite brittle, in the sense that a small adversarial change of an originally with high confidence correctly classified 
input leads to a wrong classification again with high confidence. This raises concerns that such classifiers are vulnerable to attacks and calls into question their usage in safety-critical systems.
%While former methods have tried to make classifiers more ``robust'', we think that the usage in safety critical systems requires more formal guarantees on th
We show in this paper for the first time formal guarantees on the robustness of a classifier by giving instance-specific \emph{lower bounds} on the norm of the input
manipulation required to change the classifier decision. Based on this analysis we propose the Cross-Lipschitz regularization functional. We show that using this form of regularization in kernel methods resp. neural networks improves the robustness of the classifier with no or small loss in prediction performance. 
\end{abstract}

\section{Introduction}
The problem of adversarial manipulation of classifiers has been addressed initially in the area of spam email detection, see e.g. \cite{DalEtAl2004,LowMee2005}. 
The goal of the spammer is to manipulate the spam email (the input of the classifier) in such a way that it is not detected by the classifier.
In deep learning the problem was brought up in the seminal paper by \cite{SzeEtal2014}. They showed for state-of-the-art deep neural networks, that one can manipulate an originally correctly classified input image with a \emph{non-perceivable} small transformation so that the classifier now misclassifies this image with high confidence, see \cite{GooShlSze2015} or Figure \ref{fig:adv} for an illustration. This property calls into question the usage of neural networks and other classifiers showing this behavior in safety critical systems, as they are vulnerable to attacks. On the other
hand this also shows that the concepts learned by a classifier are still quite far away from the visual perception of humans.
Subsequent research has found fast ways to generate adversarial samples with high probability \cite{GooShlSze2015,HuaEtAl2016,MooFawFro2016}
and suggested to use them during training as a form of data augmentation to gain more robustness. However, it turns out that the so-called adversarial training does not settle the problem 
as one can yet again construct adversarial examples for the final classifier. Interestingly, it has recently been shown that there exist universal adversarial changes which when
applied lead, for every image, to a wrong classification with high probability \cite{MooEtAl2016}. While one needs access to the neural network model for the generation of adversarial changes, it has
been shown that adversarial manipulations generalize across neural networks \cite{PapEtAl2016a,LiuEtAl2016,KurGooBen2016a}, which means that neural network classifiers can be attacked
even as a black-box method. The most extreme case has been shown recently \cite{LiuEtAl2016}, where they attack the commercial system Clarifai, which is a black-box system as neither the underlying classifier nor the training data are known. Nevertheless, they could successfully generate adversarial images with an existing network and fool this commercial system. This emphasizes that there are indeed severe security issues with modern neural networks.
While countermeasures have been proposed \cite{GuRig2015,GooShlSze2015,ZheEtAl2016,PapEtAl2016a,HuaEtAl2016,BasEtAl2016}, none of them provides a guarantee of preventing this behavior \cite{CarWag2017}. One might think that generative adversarial neural networks should be resistant to this problem, but it has recently been shown \cite{KosFisSon2017} that they can also be attacked by adversarial manipulation of input images.

%\begin{figure}
%\centering
%$\vcenter{
%%$$\includegraphics[width=0.46\textwidth]{adversarial-smallchange.pdf}$$%\includegraphics[width=0.46\textwidth]{adversarial-rubbish-small.pdf}$$
%$\includegraphics[width=0.3\textheight]{adversarial-smallchange.pdf}$%\qquad $\includegraphics[width=0.3\textheight]{adversarial-rubbish-small.pdf
%}$
%\caption{\label{fig:adv}Left: original image (left), applied transformation (middle), manipulated image (right). The manipulated image cannot be distinguished from
%the original image which has been correctly classified. However, the classifier predicts a wrong label for the manipulated image with high confidence (taken from \cite{ GooShlSze2015}).} 
%%Right: Generated images which are unrelated to the classes (so called ``rubbish images'') are classified with high confidence ($>0.9$) with the labels given below the image (taken from \cite{NguYosClu2015}).}
%\end{figure}

In this paper we show for the first time instance-specific formal guarantees on the robustness of a classifier against adversarial manipulation. That means we provide \emph{lower bounds} on the norm of the change of the input required to alter the classifier decision or said
otherwise: we provide a guarantee that the classifier decision does not change in a certain ball around the considered instance. We exemplify our technique for two widely used family of classifiers: kernel methods and neural networks. Based on the analysis we propose a new regularization functional, which we call \emph{Cross-Lipschitz Regularization}. This regularization functional can be used in kernel methods and neural networks. We show that using Cross-Lipschitz regularization improves both the formal guarantees of the resulting classifier (lower bounds) as well as the change required for adversarial  manipulation (upper bounds) while maintaining similar prediction performance achievable with other forms of regularization. While there exist fast ways to generate adversarial samples  \cite{GooShlSze2015,HuaEtAl2016,MooFawFro2016} without constraints, 
we provide algorithms based on the first order approximation of the classifier which generate adversarial samples satisfying box constraints in $O(d\log d)$, where $d$ is the input dimension.

%%%%%%%%%%%%%%%%%%%%%%%%%%%%%%%%%%%%%%%%%%%%%%%%%%%%%%%%%%%%%%%%%%%%%%%%%%%%%%%%%%%%%%%%%%%%%%%%%%%%%%%%%%%%%%%
%%%%%%%%%%%%%%%%%%%%%%%%%%%%%%%%%%%%%%%%%%%%%%%%%%%%%%%%%%%%%%%%%%%%%%%%%%%%%%%%%%%%%%%%%%%%%%%%%%%%%%%%%%%%%%%
%%%%%%%%%%%%%%%%%%%%%%%%%%%%%%%%%%%%%%%%%%%%%%%%%%%%%%%%%%%%%%%%%%%%%%%%%%%%%%%%%%%%%%%%%%%%%%%%%%%%%%%%%%%%%%%

\section{Formal Robustness Guarantees for Classifiers}\label{sec:robg}
In the following we consider the multi-class setting for $K$ classes and $d$ features where one has a classifier $f:\R^d \rightarrow \R^K$ and a point $x$ is classified via $c=\argmax_{j=1,\ldots,K} f_j(x)$.
We call a classifier robust at $x$ if small changes of the input do not alter the decision. 
%We have in mind the situation of Figure \ref{fig:adv} where a small, potentially even non-perceivable change of the input
%should \emph{not} change the classifier decision, in particular if the original image has been classified with high confidence. 
Formally,  the problem can be described as follows \cite{SzeEtal2014}. 
Suppose that the classifier outputs class $c$ for input $x$, that is $f_c(x)>f_j(x)$  for $j\neq c$ (we assume the decision is unique).  The problem of generating an input $x+\delta$ 
such that the classifier decision %$f: \mathbb{R}^d \rightarrow  \mathbb{R}^K$  
changes, can be formulated as
\begin{align}\label{eq:advopt}
 \minop_{\delta \in \mathbb{R}^d} \; \norm{\delta}_p, \qquad \textrm{s.th.}  \quad    \maxop_{l\neq c} \; f_l(x+\delta) \geq f_c(x+\delta) \; \textrm{ and }\; x+\delta \in C,
\end{align}
where $C$ is a constraint set specifying certain requirements on the generated input $x+\delta$, e.g., an image has to be in $[0,1]^d$. 
Typically, the optimization problem \eqref{eq:advopt} is non-convex and thus intractable. The so generated points $x+\delta$ are called \emph{adversarial samples}.
Depending on the $p$-norm the perturbations have different characteristics: for $p=\infty$ the perturbations are small and affect all features, whereas for $p=1$ one gets sparse solutions up to the extreme case that only a single feature is changed.
In \cite{SzeEtal2014} they used $p=2$ which leads to more spread but still localized perturbations.
% wheras later on emphasis was laid on the case $p=\infty$. We think all the cases are interesting, e.g. the case $p=1$ tests for images robustness against
%the change in a few pixels which could simply arise due to sensor failures, and ideally a classifier should be robust with respect to all cases.
The striking result of \cite{SzeEtal2014,GooShlSze2015} was that for most instances in computer vision datasets, the change $\delta$ necessary to alter
the decision is astonishingly small and thus clearly the label should not change. However, we will see later that our new regularizer leads to robust classifiers in the sense that the required adversarial change is so large that now also the class label changes (we have found the correct decision boundary), see Fig \ref{fig:adv}.  Already in \cite{SzeEtal2014} it is suggested to add the generated adversarial samples as a form of data augmentation during the training of neural networks in order to achieve robustness. This is denoted as \emph{adversarial training}. Later on fast ways to approximately solve \eqref{eq:advopt} were proposed in order to speed up the adversarial training process \cite{GooShlSze2015,HuaEtAl2016,MooFawFro2016}. However, in this way, given that the approximation is successful, that is $\argmax_j f_j(x+\delta)\neq c$, one gets just upper bounds
on the perturbation necessary to change the classifier decision. Also it was noted early on, that the final classifier achieved by adversarial training is again vulnerable to adversarial samples \cite{GooShlSze2015}. Robust optimization has been suggested as a measure against adversarial manipulation \cite{HuaEtAl2016,ShaYamNeg2016} which effectively boils down to adversarial training in practice. It is thus fair to say that up to date no mechanism exists which prevents the generation of adversarial samples nor can defend against it \cite{CarWag2017}.

In this paper we focus instead on robustness guarantees, that is we show that the classifier decision does not change in a small ball around the instance. Thus our guarantees hold for any method to generate adversarial samples or input transformations due to noise or sensor failure etc. Such formal guarantees are in our point of view absolutely necessary when a classifier becomes part of a safety-critical technical system such as autonomous driving. In the following we will first show how one can achieve such a guarantee and then explicitly derive bounds for
kernel methods and neural networks. We think that such formal guarantees on robustness should be investigated further and it should become standard
to report them for different classifiers alongside the usual performance measures.

\subsection{Formal Robustness Guarantee against Adversarial Manipulation}
The following guarantee holds for any classifier which is continuously differentiable with respect to the input in each output component. It is instance-specific and depends to some extent on the confidence in the decision, at least if 
we measure confidence by the relative difference $f_c(x)-\max_{j\neq c} f_j(x)$ as it is typical for the cross-entropy loss and other multi-class losses. %We discuss more aspects of the result after the theorem.
In the following we use the notation $B_p(x,R)=\{ y \in \R^d \,|\, \norm{x-y}_p \leq R\}$.
\begin{theorem}\label{th:RobG}
Let $x \in \R^d$ and   $f:\R^d \rightarrow \R^K$ be a multi-class classifier with continuously differentiable components and let $c = \argmax_{j=1,\ldots,K} f_j(x)$ be the class which $f$ predicts for $x$. 
Let $q \in \R$  be defined as $\frac{1}{p}+\frac{1}{q}=1$, then for all $\delta \in \R^d$ with
\[ \norm{\delta}_p \; \leq \;  \maxop_{R>0}\min\left\{ \minop_{j \neq c} \frac{f_c(x)-f_j(x)}{\maxop_{y \in B_p(x,R)} \norm{\nabla f_c(y) - \nabla f_j(y)}_q},\; R \right\} :=\alpha,\]
it holds $c=\argmax_{j=1,\ldots,K} f_j(x+\delta)$, that is the classifier decision does not change on $B_p(x,\alpha)$.
\end{theorem}
\ifpaper
\begin{proof}
By the main theorem of calculus, it holds that
\[ f_j(x+\delta)= f_j(x) + \int_0^1 \inner{\nabla f_j(x+t\delta),\delta} \, dt, \quad \textrm{ for } j=1,\ldots,K.\]
Thus, in order to achieve $f_j(x+\delta) \geq f_c(x+\delta)$, it has to hold that
\begin{align*}
0 \leq f_c(x)-f_j(x) &\leq \int_0^1 \inner{\nabla f_j(x+t\delta) - \nabla f_c(x+t\delta),\delta} dt\\
                     &\leq \norm{\delta}_p \int_0^1 \norm{\nabla f_j(x+t\delta) - \nabla f_c(x+t\delta)}_q \,dt,
\end{align*}
where the first inequality holds as $f_c(x)\geq f_j(x)$ for all $j=1,\ldots,K$ and in the last step we have used H{\"o}lder inequality together with the fact that the $q$-norm is dual to the $p$-norm, where $q$ is defined via $\frac{1}{p}+\frac{1}{q}=1$. Thus the minimal norm of the change $\delta$ required to change the classifier decision from $c$ to $j$ satisfies
\[ \norm{\delta}_p \geq \frac{f_c(x)-f_j(x)}{\int_0^1 \norm{\nabla f_j(x+t\delta) - \nabla f_c(x+t\delta)}_q \,dt}.\]
We upper bound the denominator over some fixed ball $B_p(x,R)$. Note that by doing this, we can only make assertions of perturbations $\delta \in B_p(0,R)$ and thus the upper bound in the guarantee is at most $R$. It holds
\[ \sup_{\delta \in B_p(0,R)} \int_0^1 \norm{\nabla f_j(x+t\delta) - \nabla f_c(x+t\delta)}_q \,dt \; \leq \; \max_{y \in B_p(x,R)} \norm{\nabla f_j(y) - \nabla f_c(y)}_q.\]
Thus we get the lower bound for the minimal norm of the change $\delta$ required to change the classifier decision from $c$ to $j$,
\[  \norm{\delta}_p \geq \min\left\{ R , \frac{f_c(x)-f_j(x)}{ \max_{y \in B_p(x,R)} \norm{\nabla f_j(y) - \nabla f_c(y)}_q}\right\}:=\alpha.\]
As we are interested in the worst case, we take the minimum over all $j\neq c$. Finally, the result holds for any fixed $R>0$ so that we can maximize over $R$ which yields the final result. %implies that if $\norm{\delta}_p\leq \alpha$, then the classifier decision cannot have changed.
\end{proof}
\fi
%%%%%%%%%%%%%%%%%%%%%%%%%%%%%%%%%%%%%%%%%%%%%%%%%%%%%%%%%%%%%%%%%%%%%%%%%%%%%%%%%%%%%%%%%%%%%%%%%%
%%%%%%%%%%%%%%%%%%%%%%%%%%%%%%%%%%%%%%%%%%%%%%%%%%%%%%%%%%%%%%%%%%%%%%%%%%%%%%%%%%%%%%%%%%%%%%%%%%%
% discussion: mention if global Lipschitz constant -> global bound
%                 bound is exact for linear classifiers
% integration of constraints 
Note that the bound requires in the denominator a bound on the local Lipschitz constant of all cross terms $f_c-f_j$, which we call local cross-Lipschitz constant in the following. However, we do not require to have a global bound. The problem with a global bound is
that the ideal robust classifier is basically piecewise constant on larger regions with sharp transitions between the classes. However, the global Lipschitz constant would then just be influenced by the sharp transition zones
and would not yield a good bound, whereas the local bound can adapt to regions where the classifier is approximately constant and then yields good guarantees. In \cite{SzeEtal2014,CisEtAl2017} they suggest to study the global 
Lipschitz constant\footnote{The Lipschitz constant $L$ wrt to $p$-norm of a piecewise continuously differentiable function is given as $L=\sup_{x \in \R^d} \norm{\nabla f(x)}_q$. Then it holds, $|f(x)-f(y)| \leq L \norm{x-y}_p$.} of each $f_j$, $j=1,\ldots,K$.
A small global Lipschitz constant  for all $f_j$ implies a good bound as
\begin{align}\label{eq:LipBound}  
\norm{\nabla f_j(y) - \nabla f_c(y)}_q \, \leq \, \norm{\nabla f_j(y)}_q  + \norm{\nabla f_c(y)}_q,
\end{align}
but the converse does not hold. As discussed below it turns out that our local estimates are significantly better than the suggested global estimates which implies also better robustness guarantees. 
In turn we want to emphasize that our bound is tight, that is the bound is attained, for linear classifiers $f_j(x)=\inner{w_j,x}$, $j=1,\ldots,K$. It holds
\[ \norm{\delta}_p \; = \; \minop_{j \neq c} \frac{\inner{w_c-w_j,x}}{\norm{w_c-w_j}_q}.\]
In Section \ref{sec:adv} we refine this result for the case when the input is constrained to $[0,1]^d$. In general, it is possible to integrate  constraints on the input by simply doing the maximum over 
the intersection of $B_p(x,R)$ with the constraint set e.g. $[0,1]^d$ for gray-scale images.

\subsection{Evaluation of the Bound for Kernel Methods}
Next, we discuss how the bound can be evaluated for different classifier models. For simplicity we restrict ourselves to the case $p=2$ (which implies $q=2$) 
and leave the other cases to future work. We consider the class of kernel methods, that is
the classifier has the form
\[ f_j(x) = \sum_{r=1}^n \alpha_{jr} k(x_r, x),\]
where $(x_r)_{r=1}^n$ are the $n$ training points, $k:\R^d \times \R^d \rightarrow \R$ is a positive definite kernel function and $\alpha \in \R^{K \times n}$ are the trained parameters e.g. of a SVM. 
The goal is to upper bound the term $ \max_{y \in B_2(x,R)} \norm{\nabla f_j(y) - \nabla f_c(y)}_2$ for this classifier model. A simple calculation shows
\begin{align}\label{eq:exp-kernel}
  0\leq \norm{\nabla f_j(y) - \nabla f_c(y)}_2^2  = \sum_{r,s=1}^n (\alpha_{jr}-\alpha_{cr})(\alpha_{js}-\alpha_{cs})\inner{\nabla_y k(x_r,y), \nabla_y k(x_s,y)}
\end{align}

It has been reported that kernel methods with a Gaussian kernel are robust to noise. Thus we specialize now to this class, that is $k(x,y)=e^{-\gamma \norm{x-y}^2_2}$. In this case
\[ \inner{\nabla_y k(x_r,y), \nabla_y k(x_s,y)} = 4\gamma^2 \inner{y-x_r, y-x_s} e^{-\gamma \norm{x_r-y}^2_2} e^{-\gamma \norm{x_s-y}^2_2}.\]
%%%%%%%%%%%%%%%%%%%%%%%%%%%%%%%%%%%%%%%%%%%%%%%%%%%%%%%%%%%%%%%%%%%%%%%%%%%%%%%%%%%%%%%%%%%%%%%%%%%%%%%%%%%%%%%%%%%%%%%%%%%%%%%%%%%%%%%%%%%%%%%%%%%%%%%%%%%%%%%%
\ifpaper
We will now derive lower and upper bounds on this term uniformly over $B_2(x,R)$ which allows us to derive the guarantee.
\begin{lemma}\label{le:dummy-kernel}
Let $M =  \min\Big\{\frac{\norm{2x-x_r - x_s}_2}{2},R\Big\}$, then 
\begin{align*} 
\maxop_{y \in B_2(x,R)}  \inner{y-x_r, y-x_s}  &= \inner{x-x_r,x-x_s} + R\norm{2x-x_r - x_s}_2 + R^2\\
\minop_{y \in B_2(x,R)}   \inner{y-x_r, y-x_s}  &= \inner{x-x_r,x-x_s} - M \norm{2x-x_r - x_s}_2 + M^2\\
\maxop_{y \in B_2(x,R)} e^{-\gamma \norm{x_r-y}^2_2} e^{-\gamma \norm{x_s-y}^2_2} &= e^{-\gamma \Big( \norm{x-x_r}^2_2 + \norm{x-x_s}^2_2 - 2M \norm{2x-x_r - x_s}_2 +2M^2\Big)}\\
\minop_{y \in B_2(x,R)} e^{-\gamma \norm{x_r-y}^2_2} e^{-\gamma \norm{x_s-y}^2_2} &= e^{-\gamma \Big( \norm{x-x_r}^2_2 + \norm{x-x_s}^2_2 +2 R\norm{2x-x_r - x_s}_2 + 2R^2\Big)}\\
\end{align*}
\end{lemma}
\begin{proof}
For the first part we use
\begin{align*}
\max_{y \in B_2(x,R)}  \inner{y-x_r, y-x_s} &= \max_{h \in B_2(0,R)}  \inner{x-x_r + h, x-x_s + h}\\ 
                                                              &= \inner{x-x_r, x-x_s} +  \max_{h \in B_2(0,R)}  \inner{h,2x - x_r -x_s} + \norm{h}_2^2\\
                                                              &= \inner{x-x_r,x-x_s} + R\norm{2x-x_r - x_s}_2 + R^2,
\end{align*}
where the last equality follows by Cauchy-Schwarz and noting that equality is attained as we maximize over the Euclidean unit ball. For the second part we consider
\begin{align*}
\min_{y \in B_2(x,R)}   \inner{y-x_r, y-x_s}  =& \min_{h \in B_2(0,R)}  \inner{x-x_r + h, x-x_s + h} \\
                                                              =& \inner{x-x_r, x-x_s} +  \min_{h \in B_2(0,R)}  \inner{h,2x - x_r -x_s} + \norm{h}_2^2\\
                                                               =&\inner{x-x_r, x-x_s} +  \min_{0\leq \alpha \leq R}  \, -\alpha \norm{2x - x_r -x_s}_2 + \alpha^2\\
                                                             =& \inner{x-x_r,x-x_s} - \min\Big\{\frac{\norm{2x-x_r - x_s}_2}{2},R\Big\} \norm{2x-x_r - x_s}_2  \\
                                                               &+ \Big(\min\Big\{\frac{\norm{2x-x_r - x_s}_2}{2},R\}\Big)^2,
\end{align*}
where in the second step we have separated direction and norm of the vector, optimization over the direction yields with Cauchy-Schwarz the result. Finally, the constrained convex one-dimensional
optimization problem can be  solved explicitly as $\alpha =\min\{\frac{\norm{2x-x_r - x_s}_2}{2},R\}$.
The proof of the other results follows analogously noting that 
\[ e^{-\gamma \norm{x+h-x_r}^2_2} e^{-\gamma \norm{x+h - x_s}^2_2} = e^{-\gamma \Big( \norm{x-x_r}^2_2 + \norm{x-x_s}^2_2 + 2\inner{h,2x-x_r-x_s} + 2\norm{h}^2_2\Big)}.\]
\end{proof}
Using this lemma it is easy to derive the final result.
\else
We derive the following bound
\fi
\begin{proposition}\label{pro:kernel-bound}
Let $\beta_r=\alpha_{jr}-\alpha_{cr}$, $r=1,\ldots,n$ and define  $M =  \min\Big\{\frac{\norm{2x-x_r - x_s}_2}{2},R\Big\}$ and $S=\norm{2x-x_r - x_s}_2$. Then
\begin{align*}
      &\max_{y \in B_2(x,R)} \norm{\nabla f_j(y) - \nabla f_c(y)}_2 \leq  2\gamma \\
 \Bigg( \sum_{\stackrel{r,s=1}{\beta_r \beta_s\geq 0}}^n \beta_r \beta_s & \Big[ \max\{\inner{x-x_r,x-x_s} + R S + R^2,\,0\} e^{-\gamma \big( \norm{x-x_r}^2_2 + \norm{x-x_s}^2_2 - 2M S +2M^2\big)}\\
           &+ \min\{\inner{x-x_r,x-x_s} + R S+ R^2,\,0 \}  e^{-\gamma \big( \norm{x-x_r}^2_2 + \norm{x-x_s}^2_2 +2 R S  + 2R^2\big)}\Big]\\
    + \sum_{\stackrel{r,s=1}{\beta_r \beta_s< 0}}^n \beta_r \beta_s  &\Big[ \max\{\inner{x-x_r,x-x_s} - M S + M^2,\, 0\} e^{-\gamma \big( \norm{x-x_r}^2_2 + \norm{x-x_s}^2_2 +2 R S + 2R^2\big)}\\
    &+\min\{\inner{x-x_r,x-x_s} - M S + M^2,\,0\}e^{-\gamma \big( \norm{x-x_r}^2_2 + \norm{x-x_s}^2_2 - 2MS  +2M^2\big)}\Big] \Bigg)^\frac{1}{2}
\end{align*}
\end{proposition}
\ifpaper
\begin{proof}
We bound each term in the sum in Equation \ref{eq:exp-kernel} separately using that $a c \leq b d$ if $b,c \geq 0$, $a\leq b$ and $c\leq d$ or $b \leq 0$, $d\geq 0$, $a\leq b$ and $c\geq d$, where $c,d$ correspond to the exponential
terms and $a,b$ to the upper bounds of the inner product. Similarly, $ab \geq cd$ if $b,c\geq 0$, $a\geq b$ and $c\geq d$ or $a\leq 0$, $d\geq 0$, $a\geq b$ and $c \leq d$. The individual upper and lower bounds are taken from
Lemma \ref{le:dummy-kernel}.
\end{proof}
\fi
While the bound leads to non-trivial estimates as seen in Section \ref{sec:exp}, the bound is not very tight. The reason is that the sum is bounded elementwise, which is quite pessimistic. We think that better bounds are possible but have to postpone this to future work.

\subsection{Evaluation of the Bound for Neural Networks}
We derive the bound for a neural network with one hidden layer. In principle, the technique we apply below can be used for arbitrary layers but the computational complexity increases rapidly. The problem is that in the directed
network topology one has to consider almost each path separately to derive the bound. Let $U$ be the number of hidden units and $w,u$ are the weight matrices of the output resp. input layer. We assume that the activation function
$\sigma$ is continuously differentiable and assume that the derivative $\sigma'$ is monotonically increasing. Our prototype activation function we have in mind and which we use later on in the experiment is the differentiable 
approximation, $\sigma_{\alpha}(x)=\frac{1}{\alpha}\log(1+e^{\alpha x})$ of the ReLU activation function $\sigma_{\mathrm{ReLU}}(x)=\max\{0,x\}$. Note that $\lim_{\alpha \rightarrow \infty} \sigma_\alpha(x)=\sigma_{\mathrm{ReLU}}(x)$
and $\sigma'_\alpha(x)=\frac{1}{1+e^{-\alpha x}}$. The output of the neural network can be written as
\[ f_j(x) = \sum_{r=1}^U w_{jr}\, \sigma\Big(\sum_{s=1}^d u_{rs} x_s\Big), \quad j=1,\ldots,K,\]
where for simplicity we omit any bias terms, but it is straightforward to consider also models with bias.
A direct computation shows that 
\begin{align}\label{eq:dummyNN}
  \norm{\nabla f_j(y) - \nabla f_c(y)}_2^2  = \sum_{r,m=1}^U (w_{jr}-w_{cr})(w_{jm}-w_{cm})\sigma'(\inner{u_r,y})\sigma'(\inner{u_m,y})\sum_{l=1}^d u_{rl}u_{ml},
\end{align}
where $u_r \in \R^{d}$ is the $r$-th row of the weight matrix $u \in \R^{U \times d}$.
The resulting bound is given in the following proposition.
\begin{proposition}\label{pro:boundNN}
Let $\sigma$ be a continuously differentiable activation function with $\sigma'$ monotonically increasing. Define 
$\beta_{rm} =  (w_{jr}-w_{cr})(w_{jm}-w_{cm})\sum_{l=1}^d u_{rl}u_{ml}$.
Then
\begin{align*}
\max_{y \in B_2(x,R)} \norm{\nabla f_j(y) - \nabla f_c(y)}_2\\
\leq \Big[ \sum_{r,m=1}^U  \max\{\beta_{rm},0\}&\sigma'\big(\inner{u_r,x} + R \norm{u_r}_2\big) \sigma'\big(\inner{u_m,x} + R \norm{u_m}_2\big) \\+  \min\{\beta_{rm},0\}&\sigma'\big(\inner{u_r,x}-R \norm{u_r}_2\big) \sigma'\big(\inner{u_m,x} - R \norm{u_m}_2\big) \Big]^\frac{1}{2}
\end{align*}
\end{proposition}
\ifpaper
\begin{proof}
The proof is based on the fact that due the monotonicity of $\sigma'$ and with Cauchy-Schwarz,
\[ \max_{y \in B_2(x,R)} \sigma'(\inner{u_r,y}) = \max_{h \in B_2(0,R)} \sigma'(\inner{u_r,x}+\inner{u_r,h}) = \sigma'(\inner{u_r,x} + R \norm{u_r}_2).\]
Similarly, one gets
\[ \min_{y \in B_2(x,R)} \sigma'(\inner{u_r,y}) = \min_{h \in B_2(0,R)} \sigma'(\inner{u_r,x}+\inner{u_r,h}) = \sigma'(\inner{u_r,x} - R \norm{u_r}_2).\]
The rest of the result follows by element-wise bounding the terms in the sum in Equation \ref{eq:dummyNN}.
\end{proof}
\fi
As discussed above the global Lipschitz bounds of the individual classifier outputs, see \eqref{eq:LipBound}, lead to an upper bound of our desired local 
cross-Lipschitz constant. In the experiments below our local bounds on the Lipschitz constant are up to 8 times smaller, than what one would achieve via the global Lipschitz bounds of \cite{SzeEtal2014}. This shows that their global approach is much too rough to get meaningful robustness
guarantees.

%%%%%%%%%%%%%%%%%%%%%%%%%%%%%%%%%%%%%%%%%%%%%%%%%%%%%%%%%%%%%%%%%%%%%%%%%%%%%%%%%%%%%%%%%%%%%%%%%%%%%%%%%%%%%%%
%%%%%%%%%%%%%%%%%%%%%%%%%%%%%%%%%%%%%%%%%%%%%%%%%%%%%%%%%%%%%%%%%%%%%%%%%%%%%%%%%%%%%%%%%%%%%%%%%%%%%%%%%%%%%%%
%%%%%%%%%%%%%%%%%%%%%%%%%%%%%%%%%%%%%%%%%%%%%%%%%%%%%%%%%%%%%%%%%%%%%%%%%%%%%%%%%%%%%%%%%%%%%%%%%%%%%%%%%%%%%%%

\section{The Cross-Lipschitz Regularization Functional}\label{sec:reg}
We have seen in Section \ref{sec:robg} that if 
\begin{align}\label{eq:lip} 
\maxop_{j \neq c} \maxop_{y \in B_p(x,R)} \norm{\nabla f_c(y) - \nabla f_j(y)}_q,
\end{align}
is small and $f_c(x)-f_j(x)$ is large, then we get good robustness guarantees. The latter property is typically
already optimized in a multi-class loss function. We consider for all methods in this paper the cross-entropy loss
so that the differences in the results only come from the chosen function class (kernel methods versus neural networks)
and the chosen regularization functional. The cross-entropy loss $L:\{1,\ldots,K\} \times \R^K \rightarrow \R$  is given as
\[ L(y,f(x)) = - \log\Big(\frac{e^{f_y(x)}}{\sum_{k=1}^K e^{f_k(x)}}\Big) = \log\Big(1 + \sum_{k\neq y}^K e^{f_k(x)-f_y(x)}\Big).\]
In the latter formulation it becomes apparent that the loss tries to make the difference $f_y(x)-f_k(x)$ as large as possible for all $k=1,\ldots,K$.

As our goal are good robustness guarantees it is natural to consider a proxy of the quantity in \eqref{eq:lip} for regularization.
We define the \textbf{Cross-Lipschitz Regularization} functional as 
\begin{align}\label{eq:crosslip}
\Omega(f) = \frac{1}{nK^2}\sum_{i=1}^n \sum_{l,m=1}^K \norm{\nabla f_l(x_i) - \nabla f_m(x_i)}_2^2,
\end{align}
where the $(x_i)_{i=1}^n$ are the training points. The goal of this regularization functional is to make the \emph{differences} of the classifier functions
at the data points as constant as possible. In total by minimizing
\begin{align}\label{eq:L-Alg}
 \frac{1}{n}\sum_{i=1}^n L\big(y_i,f(x_i)\big) + \lambda \Omega(f),
\end{align}
over some function class we thus try to maximize $f_c(x_i)-f_j(x_i)$ and at the same time keep $\norm{\nabla f_l(x_i) - \nabla f_m(x_i)}_2^2$ small uniformly over all
classes. This automatically enforces robustness of the resulting classifier. It is important to note that this regularization functional is coherent with the loss
as it shares the same degrees of freedom, that is adding the same function $g$ to all outputs: $f'_j(x)=f_j(x)+g(x)$ leaves loss and regularization functional invariant. This is the main difference to \cite{CisEtAl2017}, where they enforce the global Lipschitz constant to be smaller than one.
%Note that this is not the case when one would penalize in the traditional way the derivatives of the classifier functions e.g. $\sum_{i=1}^n \sum_{j=1}^K %\norm{\nabla f_j(x_i)}^2_2$. 

\subsection{Cross-Lipschitz Regularization in Kernel Methods}
In kernel methods one uses typically the regularization functional induced by the kernel which is given as the squared norm of the function, $f(x)=\sum_{i=1}^n \alpha_i k(x_i,x)$, in the corresponding reproducing
kernel Hilbert space $H_k$, $\norm{f}^2_{H_k} = \sum_{i,j=1}^n \alpha_i \alpha_j k(x_i,x_j).$
In particular, for translation invariant kernels one can make directly a connection to penalization of derivatives of the function $f$ via the Fourier transform, see 
\cite{SchSmo2002}. However, penalizing higher-order derivatives is irrelevant for achieving robustness. Given the kernel expansion of $f$,
one can write the Cross-Lipschitz regularization function as
\begin{align*}
 \Omega(f) = \frac{1}{nK^2}\sum_{i,j=1}^n \sum_{l,m=1}^K \sum_{r,s=1}^n (\alpha_{lr}-\alpha_{mr})(\alpha_{ls}-\alpha_{ms})\inner{\nabla_y k(x_r,x_i), \nabla_y k(x_s,x_i)}
\end{align*}
$\Omega$ is convex in $\alpha \in \R^{K \times n}$ as $k'(x_r,x_s)=\inner{\nabla_y k(x_r,x_i), \nabla_y k(x_s,x_i)}$ is a positive definite kernel for any $x_i$ and with the convex cross-entropy loss the learning problem in \eqref{eq:L-Alg} is convex. %We use L-BFGS for solving the resulting convex optimization problem.

\subsection{Cross-Lipschitz Regularization  in Neural Networks}
The standard way to regularize neural networks is \emph{weight decay}; that is, the squared Euclidean norm of all weights is added to the objective. More recently dropout \cite{SriEtAl2014}, which can be seen as a form of stochastic regularization, has been introduced. Dropout can also be interpreted
as a form of regularization of the weights \cite{SriEtAl2014,HelLon2015}. It is interesting to note that classical regularization functionals which penalize derivatives of
the resulting classifier function are not typically used in deep learning, but see \cite{DruCun1992,SchHoc1995}. As noted above we restrict ourselves to one hidden layer neural networks to simplify notation, that is, $f_j(x) = \sum_{r=1}^U w_{jr}\, \sigma\big(\sum_{s=1}^d u_{rs} x_s\big), \quad j=1,\ldots,K.$
Then we can write the Cross-Lipschitz regularization as
\begin{align*}
 \Omega(f) %&=  \sum_{i,j=1}^n \sum_{l,m=1}^K \sum_{r,s=1}^U (w_{lr}-w_{mr})(w_{ls}-w_{ms})\sigma'(\inner{u_r,x_i})\sigma'(\inner{u_s,x_i})\sum_{l=1}^d u_{rl}u_{sl}\\
           &=  \frac{2}{nK^2} \sum_{r,s=1}^U \Big(\sum_{l=1}^K w_{lr}w_{ls} - \sum_{l=1}^K w_{lr} \sum_{m=1}^K w_{ms}\Big) \sum_{i,j=1}^n \sigma'(\inner{u_r,x_i})\sigma'(\inner{u_s,x_i}) \sum_{l=1}^d u_{rl}u_{sl}
\end{align*}
which leads to an expression which can be fast evaluated using vectorization. Obviously, one can also implement the Cross-Lipschitz Regularization
also for all standard deep networks.

%%%%%%%%%%%%%%%%%%%%%%%%%%%%%%%%%%%%%%%%%%%%%%%%%%%%%%%%%%%%%%%%%%%%%%%%%%%%%%%%%%%%%%%%%%%%%%%%%%%%%%%%%%%%%%%
%%%%%%%%%%%%%%%%%%%%%%%%%%%%%%%%%%%%%%%%%%%%%%%%%%%%%%%%%%%%%%%%%%%%%%%%%%%%%%%%%%%%%%%%%%%%%%%%%%%%%%%%%%%%%%%
%%%%%%%%%%%%%%%%%%%%%%%%%%%%%%%%%%%%%%%%%%%%%%%%%%%%%%%%%%%%%%%%%%%%%%%%%%%%%%%%%%%%%%%%%%%%%%%%%%%%%%%%%%%%%%%
\section{Box Constrained Adversarial Sample Generation}\label{sec:adv}
The main emphasis of this paper are robustness guarantees without resorting to particular ways how to generate adversarial samples. On the other hand while Theorem \ref{th:RobG}
gives lower bounds on the required input transformation, efficient ways to approximately solve the adversarial sample generation in \eqref{eq:advopt} are helpful to get upper bounds
on the required change. Upper bounds allow us to check how tight our derived lower bounds are. As all of our experiments will be concerned with images, it is reasonable that our adversarial samples are also images. However, up to our knowledge, the current main techniques to generate adversarial samples \cite{GooShlSze2015,HuaEtAl2016,MooFawFro2016} integrate box constraints by clipping the results to $[0,1]^d$. We provide in the following fast algorithms to generate adversarial samples
which lie in $[0,1]^d$. The strategy is similar to \cite{HuaEtAl2016}, where they use a linear approximation of the classifier to derive
adversarial samples with respect to different norms. Formally, 
\[ f_j(x+\delta) \approx f_j(x) + \inner{\nabla f_j(x),\delta}, \quad j=1,\ldots,K.\]
Assuming that the linear approximation holds, the optimization problem \eqref{eq:advopt} integrating box constraints for changing class $c$ into $j$ becomes
\begin{align}\label{opt:lp-box} 
   \min_{\delta \in \R^d}\; & \norm{\delta}_p\\
   \textrm{ sbj. to: }      & f_j(x)-f_c(x) \geq \inner{\nabla f_c(x)-\nabla f_j(x),\delta} \nonumber\\
                            & 0\leq x_j + \delta_j \leq 1 \nonumber
\end{align}
In order to get the minimal adversarial sample we have to solve this for all $j\neq c$ and take the one with minimal $\norm{\delta}_p$. This yields the
minimal adversarial change for linear classiifers.
Note that \eqref{opt:lp-box} is a convex optimization problem, which can be reduced to a one-parameter problem in the dual. This allows to derive the following result (proofs and algorithms are in the supplement).
\begin{proposition}\label{pro:generate}
Let $p \in \{1,2,\infty\}$, then  \eqref{opt:lp-box} can be solved in $O(d \log d)$ time.
\end{proposition}
%%%%%%%%%%%%%%%%%%%%%%%%%%%%%%%%%%%%%%%%%%%%%%%%%%%%%%%%%%%%%%%%%%%%%%%%%%%%%%%%%%%%%%%%%%%%%%%%%%%%%%%%%%%%%%%%%%%%%%%%%%%%%%%%%%%%%%%%%%%%%%%%%%%%%%%%%%%%55
\ifpaper
We start with problem for $p=2$ which is given as:
\begin{align}\label{opt:l2-box} 
   \min_{\delta \in \R^d}\; & \norm{\delta}_2\\
   \textrm{ sbj. to: }       & f_j(x)-f_c(x) \geq \inner{\nabla f_c(x)-\nabla f_j(x),\delta} \nonumber\\
                            & 0\leq x_j + \delta_j \leq 1 \nonumber
\end{align}
\begin{lemma}\label{le:proj-l2}
Let $m=\argmax_j f_j(x)$ and define $v=\nabla f_m(x)-\nabla f_j(x)$ and $0>c=f_j(x)-f_m(x)$. If a solution of problem \eqref{opt:l2-box} exists, then it is given as
\[ \delta_r = \begin{cases}    -\lambda v_r, & \textrm{ if } -x_r \leq -\lambda v_r \leq 1-x_r,\\
                               1-x_r ,              & \textrm{ if }  -\lambda v_r \geq  1-x_r,\\
                               -x_r,                & \textrm{ if } -x_r \geq -\lambda v_r.\end{cases}\]
The optimal $\lambda\geq 0$ can be obtained by solving 
\[ c=\inner{v,\delta} = -\lambda \sum_{-x_r \leq -\lambda v_r \leq 1-x_r} v_r^2 + \sum_{-\lambda v_r > 1-x_r} v_r (1-x_r) - \sum_{-\lambda v_r < -x_r} v_r x_r.\]
If Problem \ref{opt:l2-box} is infeasible, then this equation has no solution for $\lambda \geq 0$. In both the feasible and infeasible case the solution can
be found in $O(d \log d)$.
The algorithm is given in Algorithm \ref{alg:L2}.
\end{lemma}
\begin{proof}
The Lagrangian is given by 
\[ L(\delta,\lambda,\alpha,\beta) = \frac{1}{2}\norm{\delta}^2_2 + \lambda  \big(\inner{v,\delta}-c\big) + \inner{\alpha,x+\delta-\ones} - \inner{\beta,x+\delta}.\]
The KKT conditions become
\begin{align*}
   \delta + \lambda v + \alpha - \beta &=0\\
   \alpha_r (x_r + \delta_r - 1) &=0, \quad \forall r=1,\ldots,d\\
   \beta_r  (x_r + \delta_r    ) &=0, \quad \forall r=1,\ldots,d\\
   \lambda (\inner{v,\delta}-c) &=0\\
   \alpha_r &\geq 0, \quad \forall r=1,\ldots,d\\
   \beta_r &\geq 0, \quad \forall r=1,\ldots,d\\
   \lambda &\geq 0.
\end{align*}
We deduce that if $\beta_r>0$ then $\alpha_r=0$ which implies 
\[ \delta_r = -x_r = -\lambda v_r + \beta_r \quad \Longrightarrow \quad \beta_r = \max\{0,-x_r+\lambda v_r\}.\]
Similarly, if $\alpha_r>0$ then $\beta_r=0$ which implies
\[ \delta_r =1-x_r = -\lambda v_r - \alpha_r \quad \Longrightarrow \quad \alpha_r = \max\{0, x_r-1-\lambda v_r\}.\]
It follows
\[ \delta_r = \begin{cases} -\lambda v_r & \textrm{ if } -x_r < -\lambda v_r < 1-x_r\\ 1-x_r & \textrm{ if } -\lambda v_r > 1-x_r \\ -x_r & \textrm{ if } -\lambda v_r <-x_r\end{cases}.\]
We can determine $\lambda$ by inspecting $\inner{v,\delta}$ which is given as
\[ \inner{v,\delta} = -\lambda \sum_{-x_r \leq -\lambda v_r \leq 1-x_r} v_r^2 + \sum_{-\lambda v_r > 1-x_r} v_r (1-x_r) - \sum_{-\lambda v_r < -x_r} v_r x_r.\]
Note that $\lambda \geq 0$ and $1-x_r\geq 0$ and thus  $-\lambda v_r > 1-x_r$ implies $v_r<0$, thus $v_r (1-x_r)\leq 0$ and similarly $-\lambda v_r < -x_r$
implies $v_r>0$ and thus also $-v_r x_r<0$. Note that the term $\inner{v,\delta}$ is monotonically decreasing as $\lambda$ is increasing. Thus one can sort
$\max\{ \frac{x_r-1}{v_r},\frac{x_r}{v_r}\}$ in increasing order and they represent the thresholds when the summation changes. Then we compute $\inner{v,\delta}$ for
all of these thresholds and determine the largest threshold $\lambda^*$ such that $\inner{v,\delta}\leq c$. This fixes the index sets of all sums. Then we determine $\lambda$ by
\[ \lambda = \frac{ \sum\limits_{-\lambda^* v_r > 1-x_r} v_r (1-x_r) - \sum\limits_{-\lambda^* v_r < -x_r} v_r x_r-c}{\sum\limits_{-x_r \leq -\lambda^* v_r \leq 1-x_r} v_r^2}.\] 
In total sorting takes time $O(d \log d)$ and solving for $\lambda^*$ has complexity $O(d)$.
\end{proof}
\begin{algorithm}
\caption{\label{alg:L2} Computation of box-constrained adversarial samples wrt to $\norm{\cdot}_2$-norm.}
\begin{algorithmic}
\STATE INPUT: $c=f_j(x)-f_m(x)$ and $v=\nabla f_m(x)-\nabla f_j(x)$ (j, desired class, m original class)
\STATE sort $\gamma_r = \max\{ \frac{x_r-1}{v_r},\frac{x_r}{v_r}\}$ in increasing order $\pi$
\STATE s=0; $\rho=0$
\WHILE{$\rho>c$}
\STATE $s \leftarrow s + 1$
\STATE compute $\rho=\inner{v,\delta(\gamma_{\pi_s})}$, where $\delta(\lambda)$ is the function defined in Lemma \ref{le:proj-l2}
\ENDWHILE
\IF{$\rho \leq c$}
\STATE compute $I_m = \{r \,|\,-x_r \leq -\gamma_{\pi_{s-1})}v_r \leq 1-x_r\}$, 
\STATE compute $I_u=\{r\,|\, -\gamma_{\pi_{s-1}} v_r > 1-x_r\}$, 
\STATE compute $I_l=\{r\,|\,-\gamma_{\pi_{s-1}} v_r < -x_r\}$
\STATE $\lambda = \frac{ \sum\limits_{r \in I_u} v_r (1-x_r) - \sum\limits_{r \in I_l} v_r x_r-c}{\sum\limits_{r \in I_m} v_r^2}.$
\STATE $\delta = \max\{-x_r,\min\{-\lambda v_r, 1-x_r\}\}$
\ELSE
\STATE Problem has no feasible solution
\ENDIF
\end{algorithmic}
\end{algorithm}
Next we consider the case $p=1$.
\begin{align}\label{opt:l1-box} 
   \min_{\delta \in \R^d}\; & \norm{\delta}_1\\
   \textrm{ sbj. to: }       & f_j(x)-f_c(x) \geq \inner{\nabla f_c(x)-\nabla f_j(x),\delta}\nonumber\\
                            & 0\leq x_j + \delta_j \leq 1\nonumber
\end{align}
\begin{lemma}
Let $m=\argmax_j f_j(x)$ and define $v=\nabla f_m(x)-\nabla f_j(x)$ and $c=f_j(x)-f_m(x)<0$, then the solution of the problem \eqref{opt:l1-box} can be found by Algorithm \ref{alg:L1}.
\begin{algorithm}
\caption{\label{alg:L1} Computation of box-constrained adversarial samples wrt to $\norm{\cdot}_1$-norm.}
\begin{algorithmic}
\STATE INPUT: $c=f_j(x)-f_m(x)$ and $v=\nabla f_m(x)-\nabla f_j(x)$ (j, desired class, m original class)
\STATE sort $|v_i|$ in decreasing order
\STATE s:=0, g:=0,
\WHILE{$g>c$}
\STATE $s \leftarrow s + 1$
\IF{$v_{\pi_s}>0$}
\STATE $\delta_{\pi_s}=-x_{\pi_s}$
\ELSE 
\STATE $\delta_{\pi_s}=1-x_{\pi_s}$
\ENDIF
\STATE $g \leftarrow g + \delta_{\pi_s}v_{\pi_s}$
\ENDWHILE
\IF{$g \leq c$}
\STATE $g \leftarrow g - v_{\pi_s}\delta_{\pi_s}$
\STATE $\delta_{\pi_s} \leftarrow \frac{c-g}{v_{\pi_s}}$
\ELSE
\STATE Problem has no feasible solution
\ENDIF
\end{algorithmic}
\end{algorithm}
\end{lemma}
\begin{proof}
The result is basically obvious but we derive it formally. First of all we rewrite \eqref{opt:l1-box} as a linear program.
\begin{align}\label{opt:l1-box-LP} 
   \minop_{\delta \in \R^d}\; & \sum_{i=1}^d t_i\\
   \textrm{ sbj. to: }       & c \geq \inner{v,\delta} \nonumber\\
                            & -x_j \leq \delta_j \leq 1-x_j \nonumber\\
                            & -t_i \leq \delta_i \leq t_i\nonumber\\
                            &  t_i\geq 0 \nonumber
\end{align}
The Lagrangian of this problem is
\begin{align*}
L(\delta,t,\alpha,\beta,\gamma,\theta,\kappa,\lambda) &= 
\inner{t,\ones} + \lambda(\inner{v,\delta}-c) - \inner{\alpha,\delta+x} + \inner{\beta,\delta -\ones +x}  -\inner{\gamma,\delta+t} + \inner{\theta,\delta-t} - \inner{\kappa,t}\\
&= \inner{t,\ones-\gamma-\theta-\kappa} + \inner{\delta,\beta-\alpha+\lambda v+\theta-\gamma} +\inner{\beta-\alpha,x} - \inner{\beta,\ones} - \lambda c
\end{align*}
Minimization of the Lagrangian over $t$ resp. $\delta$ leads only to a non-trivial result if
\begin{align*}
   \ones-\gamma-\theta-\kappa &=0\\
   \beta-\alpha+\lambda v+\theta-\gamma&=0
\end{align*}
We get the dual problem
\begin{align}\label{opt:l1-box-LP-dual} 
   \maxop_{\alpha,\beta,\theta,\gamma,\kappa,\lambda}\; & \inner{\beta-\alpha,x} - \inner{\beta,\ones} - \lambda c\\
   \textrm{ sbj. to: }       & \ones-\gamma-\theta-\kappa =0 \nonumber\\
                            & \beta-\alpha+\lambda v+\theta-\gamma=0\nonumber\\
                            & \alpha\geq 0, \; \beta\geq 0, \; \theta \geq 0,\;\gamma\geq 0,\;\kappa\geq 0,\;\lambda\geq 0
\end{align}
Using the equalities we can now simplify the problem by replacing $\alpha$ and using the fact that $\kappa$ is not part of the objective, the positivity just induces an additional constraint. We get
\[ \alpha = \beta +\lambda v+\theta-\gamma.\]
Plugging this into the problem \eqref{opt:l1-box-LP-dual} we get
\begin{align}\label{opt:l1-box-LP-dual} 
   \maxop_{\beta,\theta,\gamma,\lambda}\; & -\inner{\lambda v+\theta-\gamma,x} - \inner{\beta,\ones} - \lambda c\\
    \textrm{ sbj. to: }      & \ones-\gamma-\theta \geq 0 \nonumber\\
                            & \beta +\lambda v+\theta-\gamma\geq 0\nonumber\\
                            & \beta\geq 0, \; \theta \geq 0,\;\gamma\geq 0,\;\lambda\geq 0
\end{align}
We get the constraint $\beta \geq \max\{0,-\lambda v-\theta+\gamma\}$ (all the inequalities and functions are taken here componentwise) and thus we can explicitly maximize over $\beta$
\begin{align}\label{opt:l1-box-LP-dual} 
   \maxop_{\theta,\gamma,\lambda}\; & -\inner{\lambda v+\theta-\gamma,x} - \sum_{i=1}^d \max\{0,-\lambda v_i-\theta_i+\gamma_i\} - \lambda c\\
    \textrm{ sbj. to: }      & \gamma+\theta \leq \ones \nonumber\\
                            & \theta \geq 0,\;\gamma\geq 0,\;\lambda\geq 0
\end{align}
As $0\leq x_i \leq 1$ it holds for all $\theta\geq 0,\gamma\geq0$ 
\[ \inner{-\lambda v-\theta+\gamma,x} - \sum_{i=1}^d \max\{0,-\lambda v_i-\theta_i+\gamma_i\} \leq 0.\]
The maximum is attained if $\gamma_i-\theta_i-\lambda v_i=0$ resp. with the constraints on $\gamma_i,\theta_i$ this is equivalent to $-1\leq \lambda v_i\leq 1$.
Suppose that $\lambda v_i>1$ then the maximum is attained for $\gamma_i=1$ and $\theta_i=0$, and for $\lambda v_i <-1$ the maximum is attained for $\gamma_i=0$ and $\theta_i=1$.
Thus by solving explicitly for $\theta$ and $\gamma$ we obtain 
\begin{align}\label{opt:l1-box-LP-dual2} 
   \maxop_{\lambda}\; & \sum_{\lambda v_i > 1} (1-\lambda_i v_i)x_i + \sum_{\lambda v_i < -1} (1-\lambda_i v)(x_i-1)- \lambda c\\
   \textrm{ sbj. to: } & \lambda\geq 0
\end{align}
Note that the first two terms are decreasing with $\lambda$ and the last term is increasing with $\lambda$. Let $\lambda^*$ be the optimum, then
we have the following characterization
\begin{align*}
    -1 < \lambda v_i < 1 \quad &\Longrightarrow \quad \gamma_i+\theta_i<1 \quad \Longrightarrow \kappa_i >0 \quad \Longrightarrow t_i=0 \Longrightarrow \delta_i=0\\
    \lambda v_i > 1 \quad &\Longrightarrow \quad\gamma_i=1,\theta_i=0,\beta_i=0,\alpha_i>0 \quad\Longrightarrow \quad \delta_i =-x_i,\\
    \lambda v_i <-1 \quad &\Longrightarrow \quad\gamma_i=0,\theta_i=1,\beta_i>0 \quad\Longrightarrow \quad \delta_i=1-x_i,
\end{align*} 
The cases $|\lambda v_i|=1$ are undetermined but given that $\lambda>0$ the remaining values can be fixed by solving for $c=\inner{v,\delta}$.
The time complexity is again determined by the initial sorting step of $O(d \log d)$. The following linear scan requires $O(d)$.
\end{proof}
Finally, we consider the case $p=\infty$.
\begin{align}\label{opt:linf-box} 
   \minop_{\delta \in \R^d}\; & \norm{\delta}_\infty\\
   \textrm{ sbj. to: }        & f_j(x)-f_c(x) \geq \inner{\nabla f_c(x)-\nabla f_j(x),\delta}\nonumber\\
                            & 0\leq x_j + \delta_j \leq 1\nonumber
\end{align}
\begin{lemma}
Let $m=\argmax_j f_j(x)$ and define $v=\nabla f_m(x)-\nabla f_j(x)$ and $c=f_j(x)-f_m(x)<0$, then the solution of \eqref{opt:linf-box} can be found by solving for
$t\geq 0$,
\[ \sum_{v_i>0} v_i \max\{-t,-x_r\} + \sum_{v_i<0} v_i \min\{t,1-x_r\}=c.\]
which is done in Algorithm \ref{alg:Linf}.
\end{lemma}
\begin{proof}
We can rewrite the optimization problem into a linear program
\begin{align}\label{opt:linf-box1} 
   \minop_{t \in \R, \delta \in \R^d}\; & t\\
   \textrm{ sbj. to: }       & c \geq \inner{v,\delta}\nonumber\\
                            & -x_j \leq \delta_j \leq 1-x_j, \quad j=1,\ldots,d\nonumber\\
                            & -t \leq \delta_j \leq t, \quad j=1,\ldots,d\nonumber\\
                            & t \geq 0
\end{align}
Thus we have $\max\{-t,-x_r\} \leq \delta_r \leq \min\{t,1-x_r\}$ for $r=1,\ldots,d$. Then it holds
\begin{align*}
\inner{v,\delta}&\geq \sum_{v_r>0} v_r \max\{-t,-x_r\} + \sum_{v_r<0} v_r \min\{t,1-x_r\}\\
                &= - \sum_{v_r>0,t\geq x_r} v_r x_r - t \sum_{v_r>0,t<x_r} v_r + t \sum_{v_r<0,t<1-x_r} v_r + \sum_{v_r<0,t\geq 1-x_r} v_r (1-x_r),
\end{align*}
and the lower bound can be attained if $\delta_r=\min\{t,1-x_r\}$ for $v_r<0$ and $\delta_r = \max\{-t,-x_r\}$ for $v_r>0$, $\delta_r=0$ if $v_r=0$.
Note that both terms are monotonically decreasing with $t$. The algorithm \ref{alg:Linf} has complexity $O(d\log d)$ due to the initial sorting step
followed by steps of complexity $O(d)$.
\end{proof}

\begin{algorithm}
\caption{\label{alg:Linf}Computation of box-constrained adversarial samples wrt to $\norm{\cdot}_\infty$-norm.}
\begin{algorithmic}
\STATE INPUT: $c=f_j(x)-f_m(x)$ and $v=\nabla f_m(x)-\nabla f_j(x)$ (j, desired class, m original class)
\STATE $d_+: = |\{l \,|\,v_l >0\}|$ and $d_- := |\{l \,|\, v_l < 0 \}|$
\STATE sort $\{x_l \,|\, v_l>0\}$ in increasing order $\pi$, sort $\{1-x_l\,|\, v_l<0\}$ in increasing order $\rho$
\STATE s:=1, r:=1, g:=0, t:=0, $\kappa_+=\sum_{v_l>0} v_l$, $\kappa_-=\sum_{v_l<0} v_l$, $\gamma_+=\gamma_-=0$.
\WHILE{$g>c$ AND ($s\leq d_+$ OR $t\leq d_-$)}
\IF{$x_{\pi_s} < 1-x_{\rho_r}$}
\STATE $t=x_{\pi_s}$, $\quad\kappa_+ \leftarrow \kappa_+ - v_{\pi_s}$, $\quad\gamma_+ \leftarrow \gamma_+ - v_{\pi_s}x_{\pi_s}$,
\STATE $s \leftarrow s+1$
\ELSE 
\STATE $t=1-x_{\rho_r}$,  $\quad\kappa_- \leftarrow \kappa_- - v_{\rho_r}$, $\quad\gamma_- \leftarrow \gamma_- - v_{\rho_r}(1-x_{\rho_r})$,
\STATE $t \leftarrow t + 1$
\ENDIF 
\STATE $g=\gamma_+ + \gamma_- + t\,(\kappa_+ - \kappa_-)$
\ENDWHILE
\IF{$g \leq c$}
\STATE undo last step
\STATE $t = (c - \gamma_+ - \gamma_-)/(\kappa_- - \kappa_+)$
\STATE compute $\delta_r = \begin{cases} \min\{t, 1-x_r\} & \textrm{ for } v_r>0\\ \max\{-t,-x_r\} & \textrm{ for }v_r<0\end{cases}$
\ELSE
\STATE Problem has no feasible solution
\ENDIF
\end{algorithmic}
\end{algorithm}
\else
\fi
For nonlinear classifiers a change of the decision is not guaranteed and thus we use later on a binary search with a variable $c$ instead of $f_c(x)-f_j(x)$.

%%%%%%%%%%%%%%%%%%%%%%%%%%%%%%%%%%%%%%%%%%%%%%%%%%%%%%%%%%%%%%%%%%%%%%%%%%%%%%%%%%%%%%%%%%%%%%%%%%%%%%%%%%%%%%%
%%%%%%%%%%%%%%%%%%%%%%%%%%%%%%%%%%%%%%%%%%%%%%%%%%%%%%%%%%%%%%%%%%%%%%%%%%%%%%%%%%%%%%%%%%%%%%%%%%%%%%%%%%%%%%%
%%%%%%%%%%%%%%%%%%%%%%%%%%%%%%%%%%%%%%%%%%%%%%%%%%%%%%%%%%%%%%%%%%%%%%%%%%%%%%%%%%%%%%%%%%%%%%%%%%%%%%%%%%%%%%%

\section{Experiments}\label{sec:exp}
The goal of the experiments is the evaluation of the robustness of the resulting classifiers and not
necessarily state-of-the-art results in terms of test error. In all  cases we compute the robustness guarantees from Theorem \ref{th:RobG}
(lower bound on the norm of the minimal change required to change the classifier decision), where we optimize over $R$ using binary search, and adversarial samples with the algorithm for the $2$-norm from Section \ref{sec:adv} (upper bound on the norm of the minimal change required to change the classifier decision),
where we do a binary search in the classifier output difference in order to find a point on the decision boundary. Additional experiments can be found in the supplementary material.
\paragraph{Kernel methods:} We optimize the cross-entropy loss once with the standard regularization (Kernel-LogReg)
and with Cross-Lipschitz regularization (Kernel-CL). Both are convex optimization problems and we use L-BFGS
to solve them.  We use the Gaussian kernel $k(x,y)=e^{-\gamma \norm{x-y}^2}$
where $\gamma=\frac{\alpha}{\rho^2_{\mathrm{KNN40}}}$ and $\rho_{\mathrm{KNN40}}$ is the mean of the 40 nearest neighbor distances on the training set
and $\alpha \in \{0.5,1,2,4\}$. We show the results for MNIST (60000 training and 10000 test samples). However, we have checked that parameter selection using a subset of 50000 images from the
training set and evaluating on the rest yields indeed the parameters which give the best test errors when trained on the full set.
The regularization parameter is chosen in $\lambda \in \{10^{-k}| k \in \{5,6,7,8\}\}$ for Kernel-SVM and $\lambda \in \{10^{-k}\,|\, k \in \{0,1,2,3\}\}$
for our Kernel-CL. The results of the optimal parameters are given in the following table and the performance of all parameters is shown in 
Figure \ref{exp:Kernel}. Note that due to the high computational complexity we could evaluate the robustness guarantees only for the optimal parameters.
\begin{center}
%\begin{minipage}{0.49\textwidth}
%{\small %\begin{left}
%\begin{tabular}{l|l|l|l|l}
%          &   & avg.                 & avg.                               & avg. \\
%          & &$\norm{\cdot}_2$  &  $\norm{\cdot}_\infty$     &  $\norm{\cdot}_2$ \\ 
%         &  test        & adv.                      & adv.                         & rob. \\ 
%         &  error        & samples               & samples     & guar. \\ 
%\hline
%No Reg.  &  2.23\%   &  4.16  &  0.41  &  0.037 \\
%($\lambda=0$) & & & \\
%K-SVM   &  1.48\%   &  4.14  &  0.36  &  0.058 \\
%K-CL    &  1.44\%   &  5.52  &  0.52  &  0.045
%\end{tabular}
\begin{minipage}{0.49\textwidth}
{\small %\begin{left}
\begin{tabular}{l|l|l|l}
          &   &                                              &  \\
          & &     avg.    $\norm{\cdot}_2$ & avg.$\norm{\cdot}_2$ \\ 
         &  test                      & adv.                         & rob. \\ 
         &  error                    & samples     & guar. \\ 
\hline
No Reg.  &  2.23\%   &  2.39 &  0.037 \\
($\lambda=0$) & & & \\
K-SVM   &  1.48\%   &  1.91 &  0.058 \\
K-CL    &  1.44\%   &   3.12 &  0.045
\end{tabular}
%\end{left}
} 
\end{minipage}
\begin{minipage}{0.45\textwidth}
%\begin{figure} 
\includegraphics[width=\textwidth]{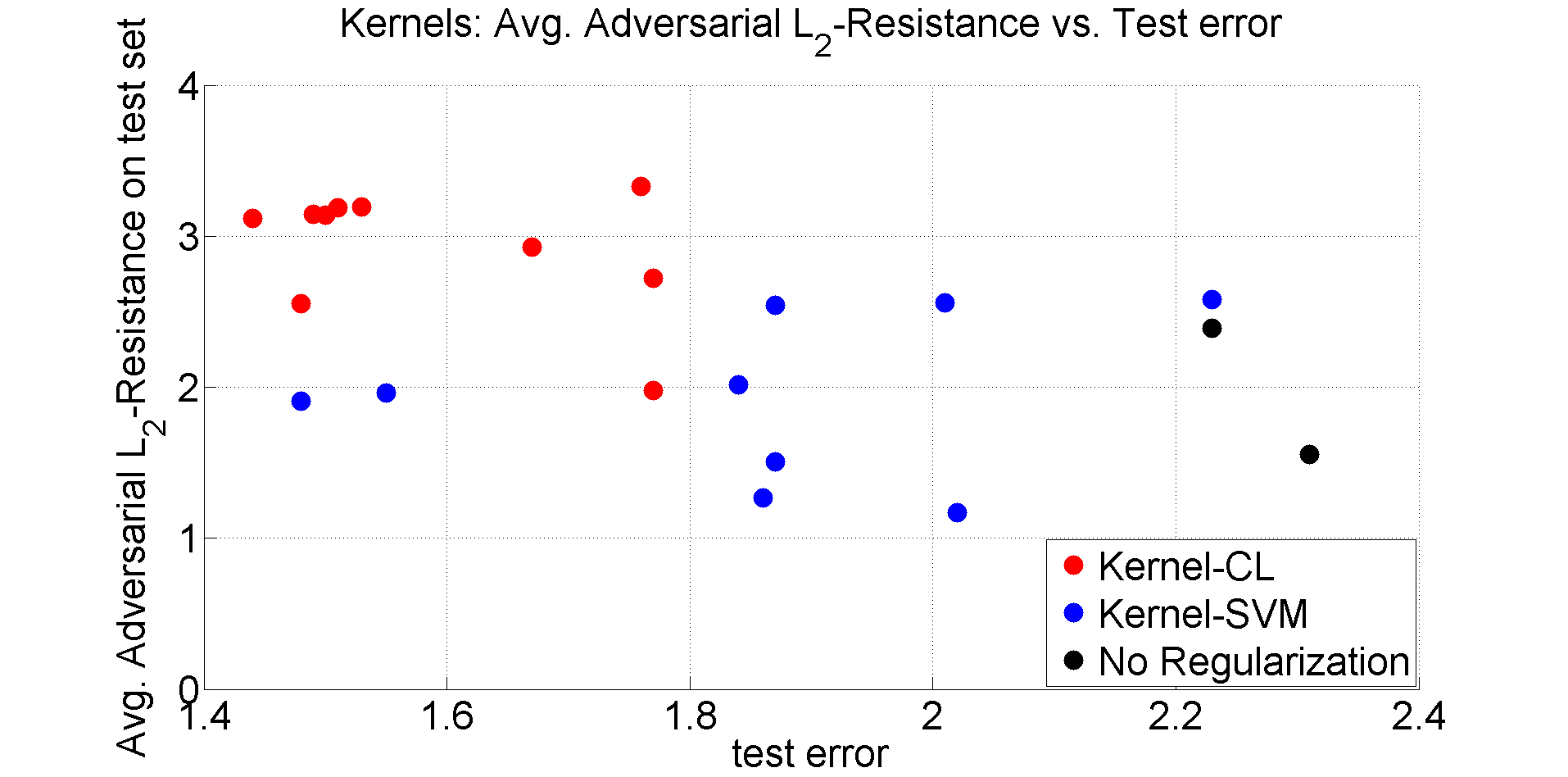}

%\end{figure}
\end{minipage}
\captionof{figure}{\label{exp:Kernel} \textbf{Kernel Methods:} Cross-Lipschitz regularization achieves both better test error and robustness
against adversarial samples (upper bounds, larger is better) compared to the standard regularization. The robustness guarantee is weaker than for neural networks but this is most likely due to the relatively loose bound.}
\end{center}

%%%%%%%%%%%%%%%%%%%%%%%%%%%%%%%%%%%%%
\paragraph{Neural Networks:} Before we demonstrate how upper and lower bounds improve using cross-Lipschitz regularization, we first want to highlight
the importance of the usage of the local cross-Lipschitz constant in Theorem \ref{th:RobG} for our robustness guarantee. 
\paragraph{Local versus global Cross-Lipschitz constant:} While no robustness guarantee has been proven before, it has been discussed in \cite{SzeEtal2014} that penalization of the global Lipschitz constant should improve robustness, see also \cite{CisEtAl2017}. For that purpose they derive the Lipschitz constants of several different layers and use the fact that the Lipschitz constant of a composition of functions is upper bounded by the product of the Lipschitz constants of the functions. In analogy, this would mean that the term $\sup_{y \in B(x,R)} \norm{\nabla f_c(y)-\nabla f_j(y)}_2$, which we have upper bounded in Proposition \ref{pro:boundNN}, in the denominator in Theorem \ref{th:RobG} could 
be replaced\footnote{Note that then the optimization of $R$ in Theorem \ref{th:RobG} would be unnecessary.} by the global Lipschitz constant of $g(x):=f_c(x)-f_j(x)$.  which is given as
$\sup_{y \in \R^d} \norm{\nabla g(x)}_2 = \sup_{x \neq y} \frac{|g(x)-g(y)|}{\norm{x-y}_2}.$
We have with $\norm{U}_{2,2}$ being the largest singular value of $U$,
\begin{align*} |g(x)-g(y)| &= \inner{w_c-w_j,\sigma(Ux)-\sigma(Uy)} \leq \norm{w_c-w_j}_2 \norm{\sigma(Ux)-\sigma(Uy)}_2 \\&\leq \norm{w_c-w_j}_2\norm{U(x-y)}_2
\leq \norm{w_c-w_j}_2 \norm{U}_{2,2} \norm{x-y}_2,\end{align*}
where we used that $\sigma$ is contractive as $\sigma'(z)=\frac{1}{1+e^{-\alpha z}}$ and thus we get
\[  \sup_{y \in \R^d} \norm{\nabla f_c(x)-\nabla f_j(x)}_2 \leq  \norm{w_c-w_j}_2 \norm{U}_{2,2}.\]
The advantage is clearly that this global Cross-Lipschitz constant can just be computed once and by using it in Theorem \ref{th:RobG} one can evaluate the guarantees very quickly. However, it turns out that one gets significantly better robustness guarantees by using the local Cross-Lipschitz constant in terms of the
bound derived in Proposition \ref{pro:boundNN} instead of the just derived global Lipschitz constant. Note that the optimization over $R$ in Theorem \ref{th:RobG} is done using a binary search, noting that the bound of the local Lipschitz constant in Proposition \ref{pro:boundNN} is monotonically decreasing in $R$.
We have the following comparison in Table \ref{tab:LocalGlobal}. We want to highlight that the robustness guarantee with the global Cross-Lipschitz constant was \emph{always} worse than when using the local Cross-Lipschitz constant across all regularizers and data sets. Table \ref{tab:LocalGlobal} shows that the guarantees using the local Cross-Lipschitz can be up to eight times better than for the global one. As these are just one hidden layer networks, it is obvious
that robustness guarantees for deep neural networks based on the global Lipschitz constants will be too coarse to be useful.
\begin{table}
\begin{center}
\begin{tabular}{|@{\;}c@{\;}|@{\;}c@{\;}|@{\;}c@{\;}|@{\;}c@{\;}||@{\;}c@{\;}|@{\;}c@{\;}|@{\;}c@{\;}|@{\;}c@{\;}|}
\multicolumn{4}{c}{MNIST (plain)} & \multicolumn{4}{c}{CIFAR10 (plain)}\\
None & Dropout & Weight Dec. & Cross Lip. & None & Dropout & Weight Dec. & Cross Lip. \\
\hline
0.69 & 0.48 & 0.68  & 0.21 & 0.22 & 0.13 & 0.24 & 0.17\\
\hline
\end{tabular}
\end{center}
\caption{\label{tab:LocalGlobal} We show the average ratio $\frac{\alpha_{\textrm{global}}}{\alpha_{\textrm{local}}}$ of the robustness guarantees $\alpha_{\textrm{global}}, \alpha_{\textrm{local}}$ from Theorem \ref{th:RobG} on the test data for MNIST and CIFAR10
and different regularizers. The guarantees using the local Cross-Lipschitz constant are up to eight times better than with the global one.} 
\end{table}  

\paragraph{Experiments:} We use a one hidden layer network with 1024 hidden units and the softplus activation function with $\alpha=10$. Thus
the resulting classifier is continuously differentiable. We compare three different regularization techniques: weight decay, dropout
and our Cross-Lipschitz regularization. Training is done with SGD. For each method we have adapted the learning rate (two per method) and regularization parameters (4 per method) so that all methods achieve good performance. We do experiments for MNIST and CIFAR10
in three settings: plain, data augmentation and adversarial training. The exact settings of the parameters and the augmentation
techniques are described 
\ifpaper below.\else in the supplementary material.\fi The results for MNIST are shown in Figure \ref{exp:NN-MNIST} and the
results for CIFAR10 are in \ifpaper Figure \ref{exp:NN-CIFAR}.\else the supplementary material.\fi For MNIST there is a clear trend that our Cross-Lipschitz regularization improves the 
robustness of the resulting classifier while having competitive resp. better test error. It is surprising that data augmentation does not lead to more robust models. However, adversarial training improves the guarantees as well as adversarial resistance. For CIFAR10 the picture is mixed, our CL-Regularization performs well for the augmented task in test error and upper bounds but is not significantly better in the robustness guarantees. The problem
might be that the overall bad performance due to the simple model is preventing a better behavior. Data augmentation leads to better test error but the robustness properties (upper and lower bounds) are basically unchanged. Adversarial training slightly improves performance compared to the plain setting 
and improves upper and lower bounds in terms of robustness. We want to highlight that our guarantees (lower bounds) and the upper bounds from the adversarial samples
are  not too far away.
\ifpaper
$\;$\\
For MNIST (all settings) the learning rate is for all methods chosen from $\{0.2,0.5\}$.
The regularization parameters for weight decay are chosen from $\{10^{-5},10^{-4},10^{-3},10^{-2}\}$, for Cross-Lipschitz
from $\{10^{-5},10^{-4},5*10^{-4},10^{-3}\}$ and the dropout probabilities are taken from $\{0.4,0.5,0.6,0.7\}$.
For CIFAR10 the learning rate is for all methods chosen from $\{0.04,0.1\}$, the regularization parameters for weight decay
and Cross-Lipschitz are $\{10^{-5},10^{-4},5*10^{-4},10^{-3}\}$ and  dropout probabilities are taken from $\{0.5,0.6,0.7,0.8\}$.
For CIFAR10 with data augmentation we choose the learning rate for all methods from $\{0.04,0.1\}$,  the regularization parameters for weight decay are $\{10^{-6}, 10^{-5},10^{-4},10^{-3}\}$
and for Cross-Lipschitz $\{10^{-5},10^{-4},5*10^{-4},10^{-3}\}$ and the dropout probabilities are taken from $\{0.5,0.6,0.7,0.8\}$.
Data augmentation for MNIST means that we apply random rotations in the angle $[-\frac{\pi}{20},\frac{\pi}{20}]$
and random crop  from 28x28 to 24x24. For CIFAR-10 we apply the same and additionally we mirror the image (left to right)
with probability 0.5 and apply random brightness $[-0.1,0.1]$ and random contrast change $[0.6,1.4]$. In each substep we ensure
that we get an image in $[0,1]^d$ by clipping. We implemented adversarial training by generating adversarial samples wrt to the infinity norm with 
the code from Section \ref{sec:adv} and replaced 50\% of each batch as adversarial samples. Finally, we use for SGD batchsize 64
in all experiments.
\fi
\begin{figure}
\centering
%{\scriptsize 
\begin{tabular}{c|c}
 Adversarial Resistance (Upper Bound)  & Robustness Guarantee (Lower Bound)\\
  wrt to $L_2$-norm                             & wrt to $L_2$-norm\\
  \includegraphics[width=0.45\textwidth]{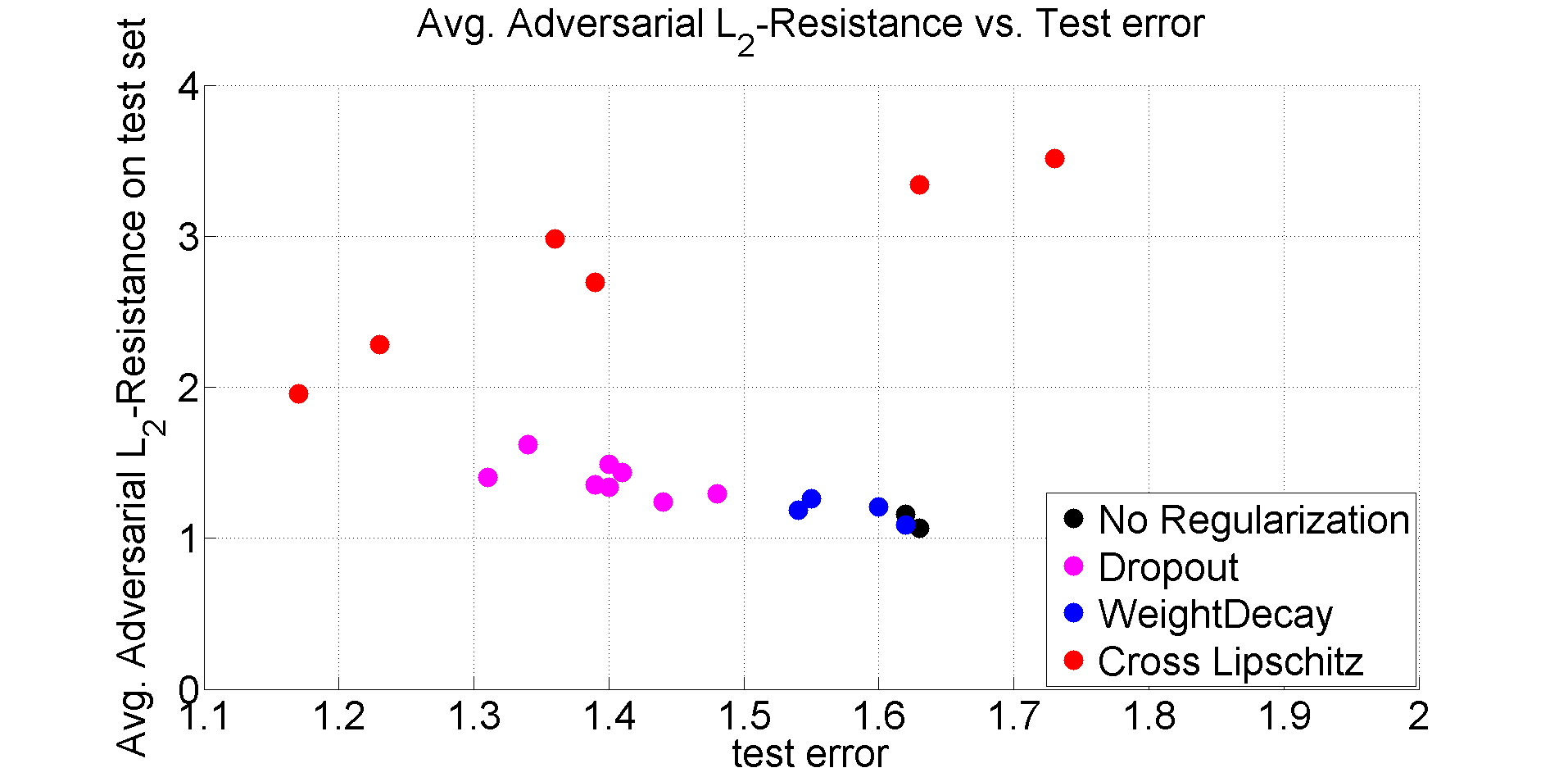}     &\includegraphics[width=0.45\textwidth]{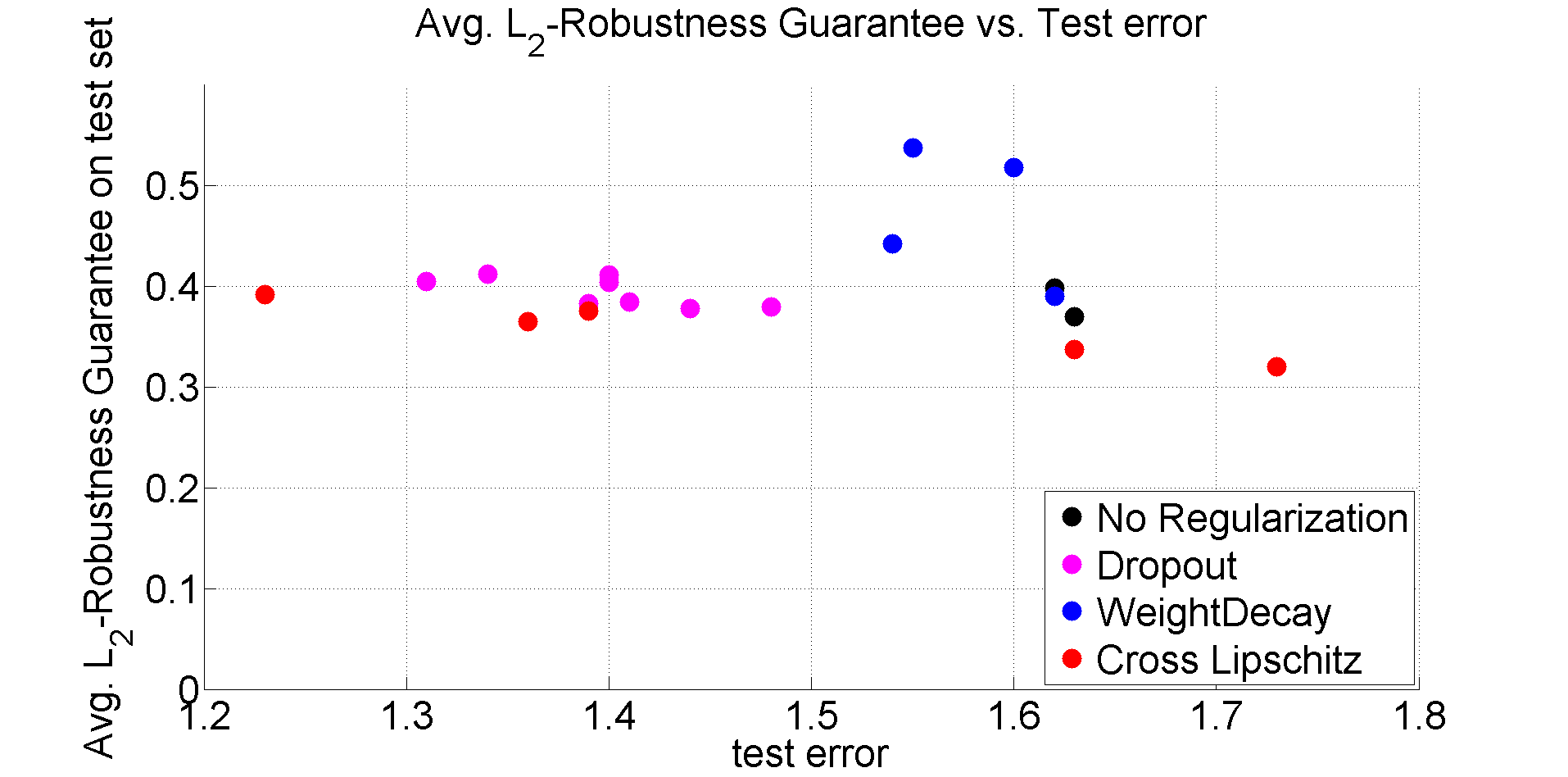}\\
  \includegraphics[width=0.45\textwidth]{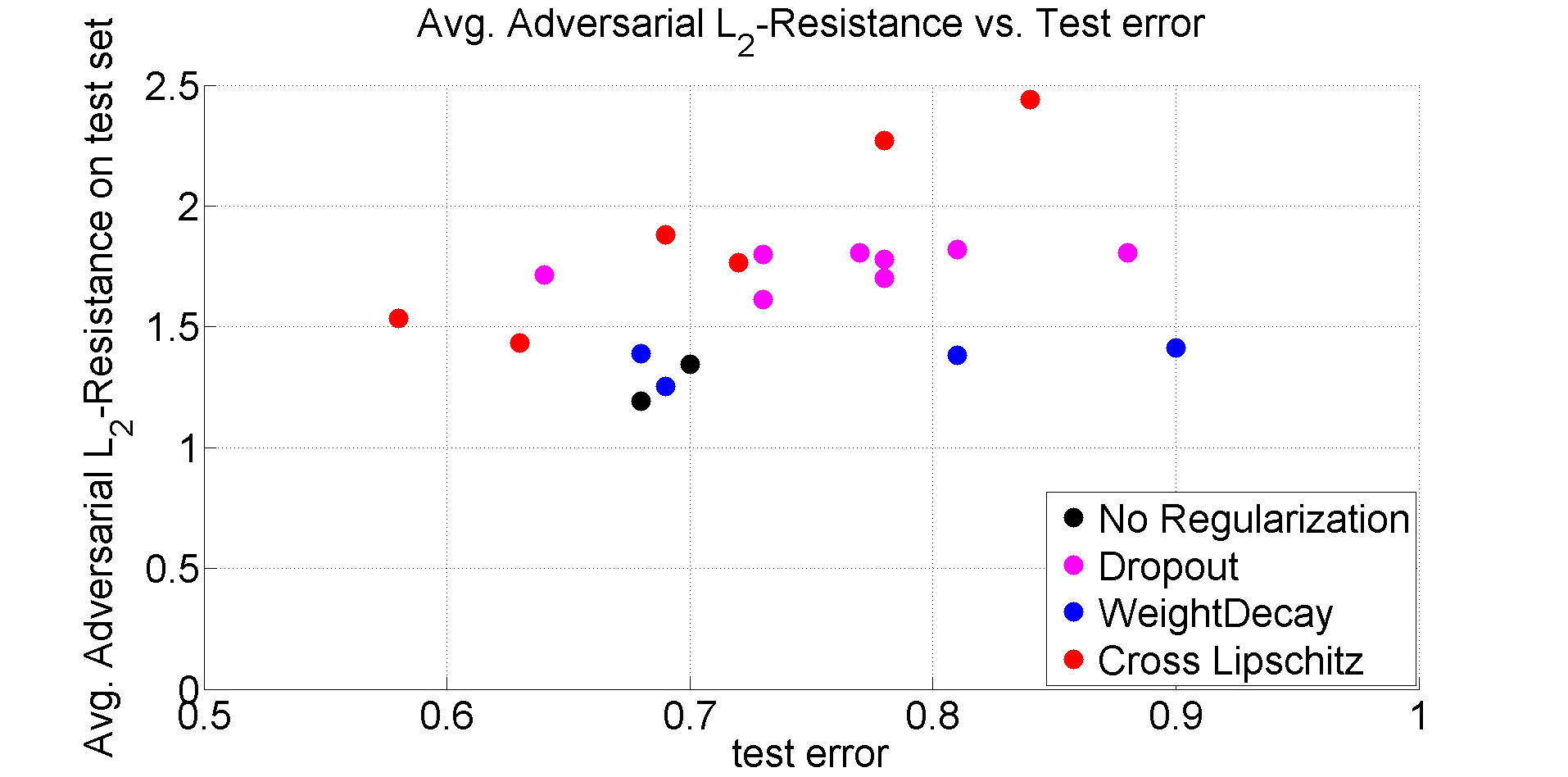}&\includegraphics[width=0.45\textwidth]{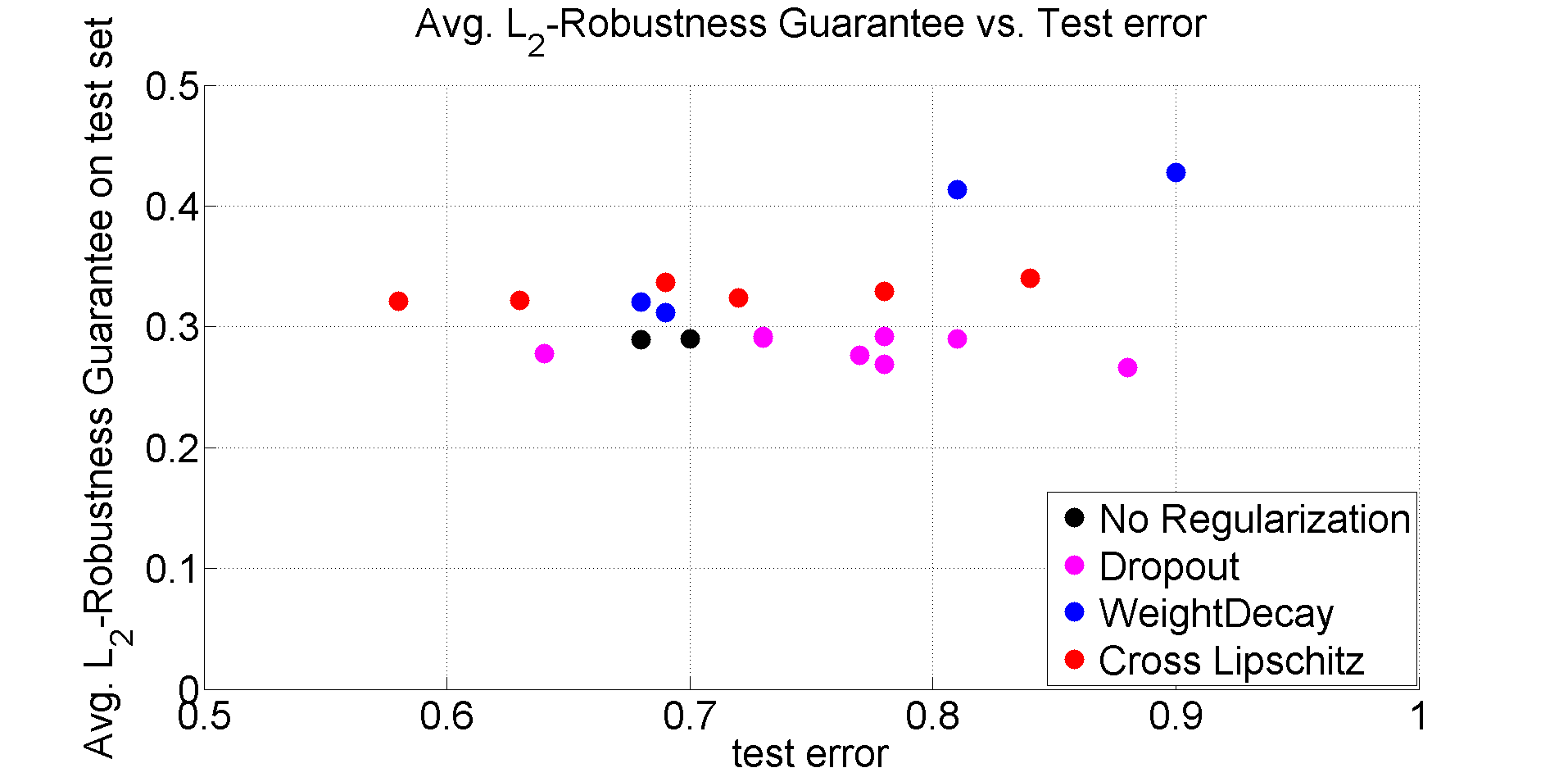}\\
  \includegraphics[width=0.45\textwidth]{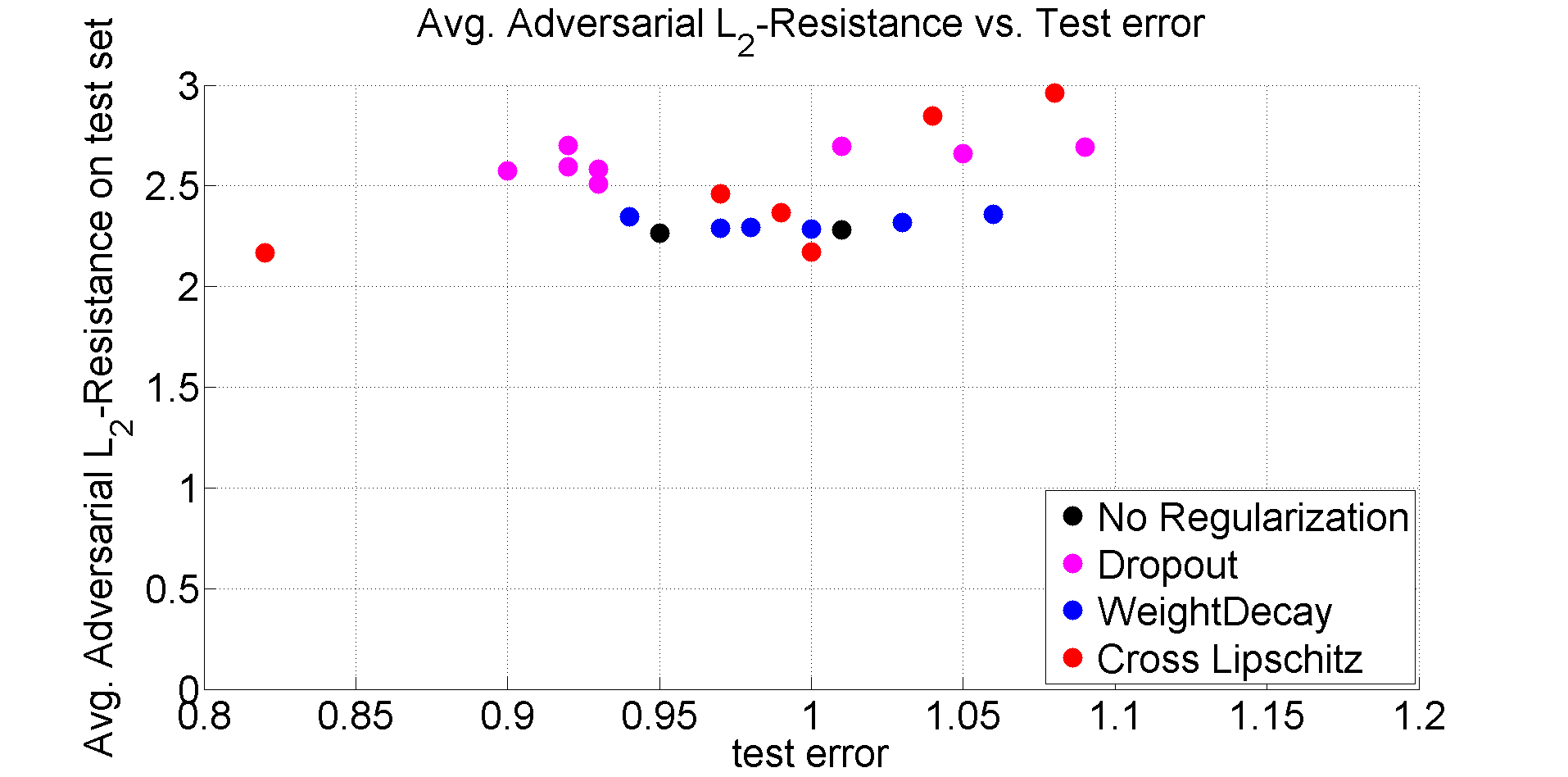}&\includegraphics[width=0.45\textwidth]{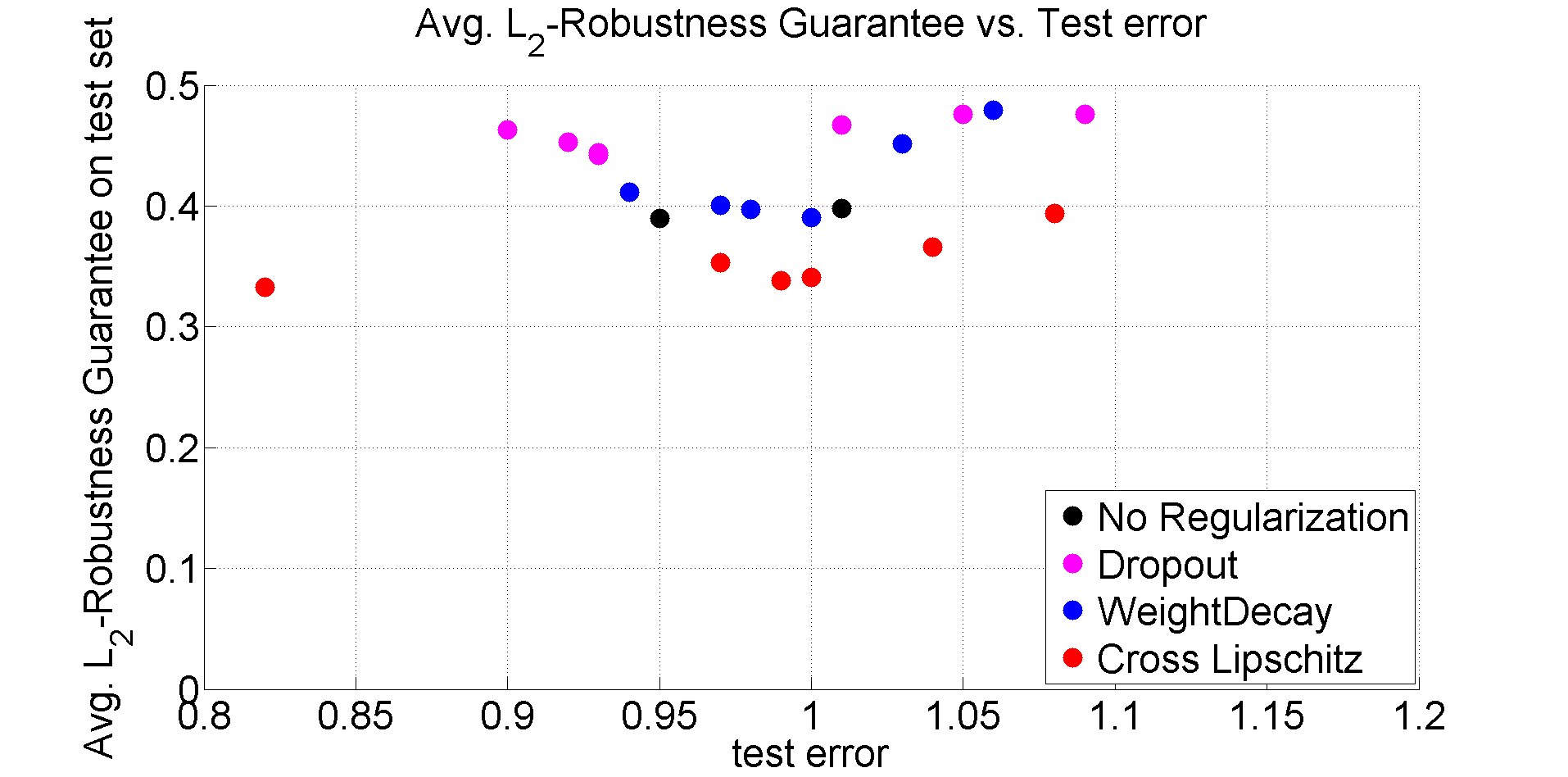}\\
\end{tabular}%}
\caption{\label{exp:NN-MNIST} \textbf{Neural Networks, Left:} Adversarial resistance wrt to $L_2$-norm on MNIST. \textbf{Right:} Average robustness guarantee wrt to $L_2$-norm on MNIST for different neural networks (one hidden layer, 1024 HU) and hyperparameters. The Cross-Lipschitz regularization leads to better robustness with similar or better prediction
performance. \textbf{Top row:} plain MNIST, \textbf{Middle:} Data Augmentation, \textbf{Bottom:} Adv. Training}
\end{figure}
%\fi
%
\ifpaper
\begin{figure}
{\scriptsize 
\begin{tabular}{c|c}
 Adversarial Resistance (Upper Bound)  & Robustness Guarantee (Lower Bound)\\
  wrt to $L_2$-norm                             & wrt to $L_2$-norm\\
  \includegraphics[width=0.45\textwidth]{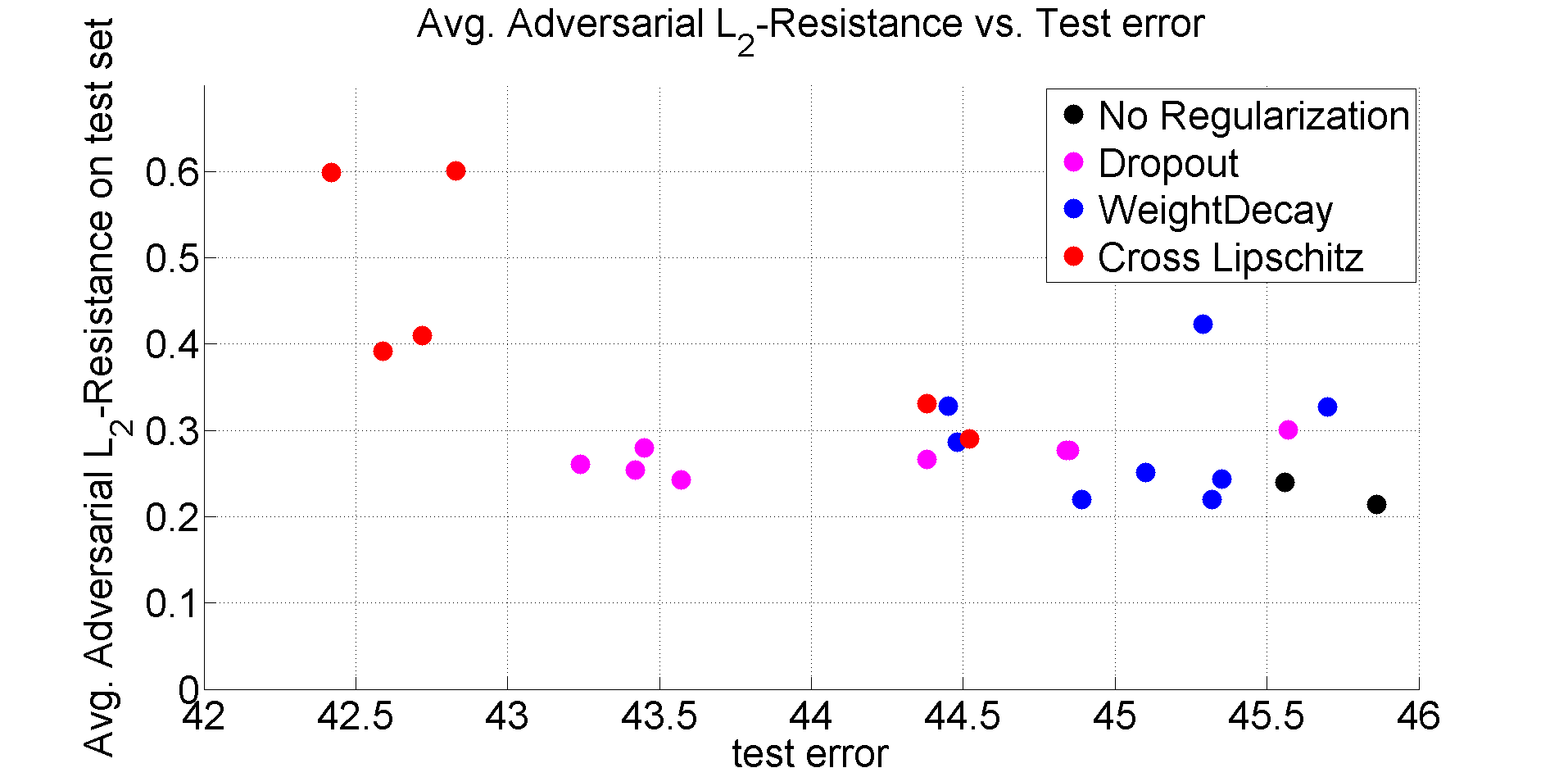}     &\includegraphics[width=0.45\textwidth]{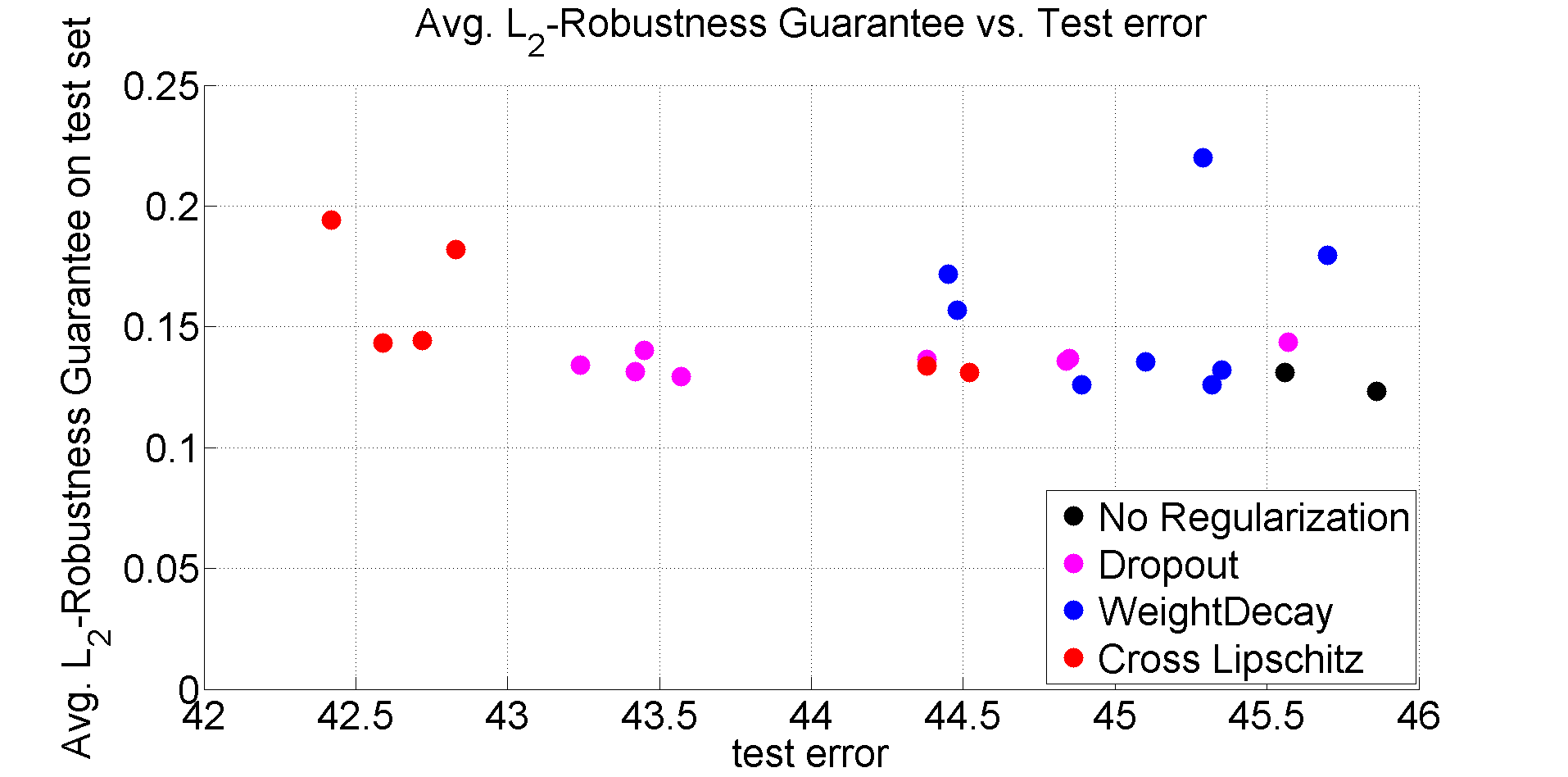}\\
  \includegraphics[width=0.45\textwidth]{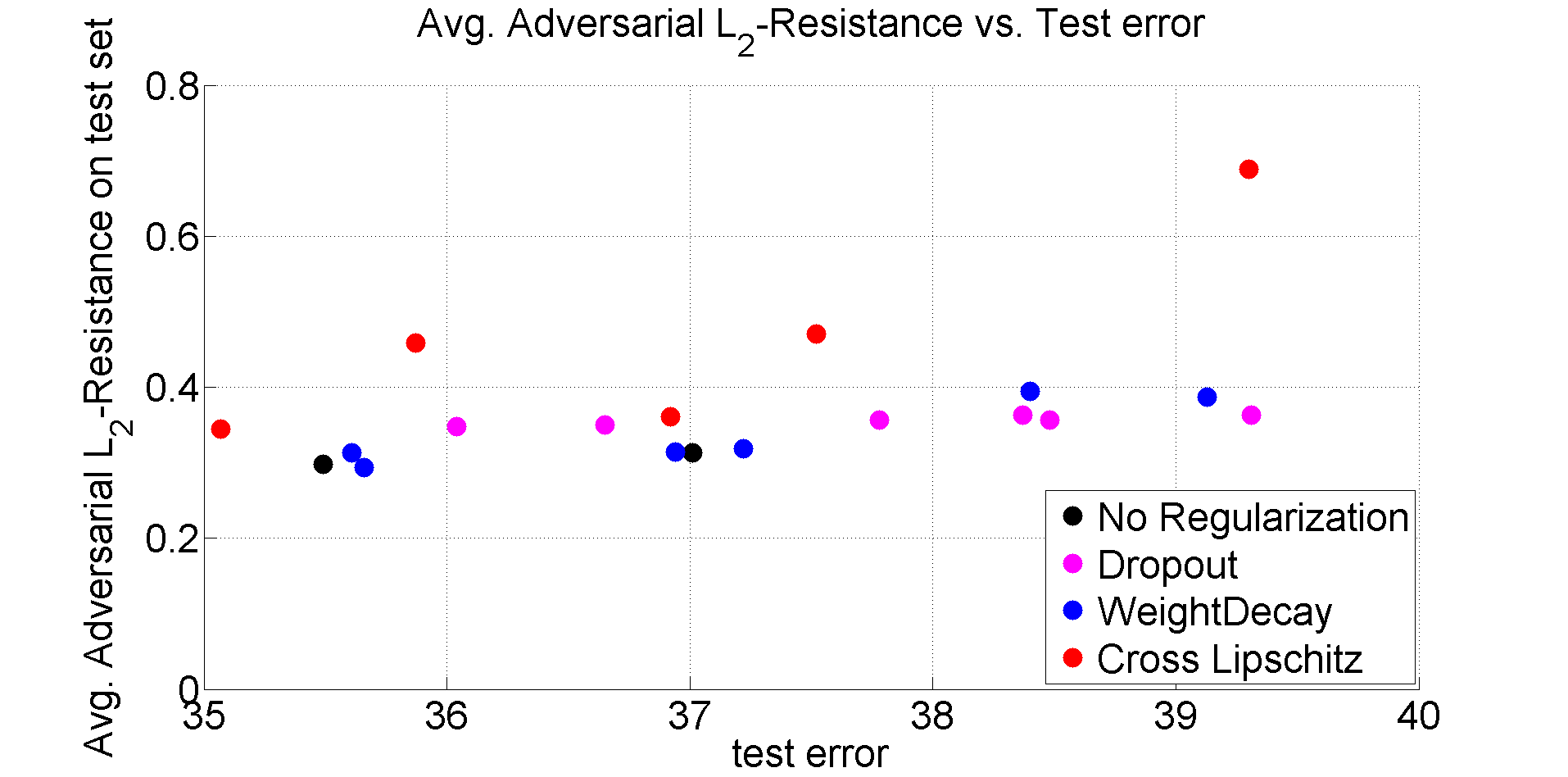}&\includegraphics[width=0.45\textwidth]{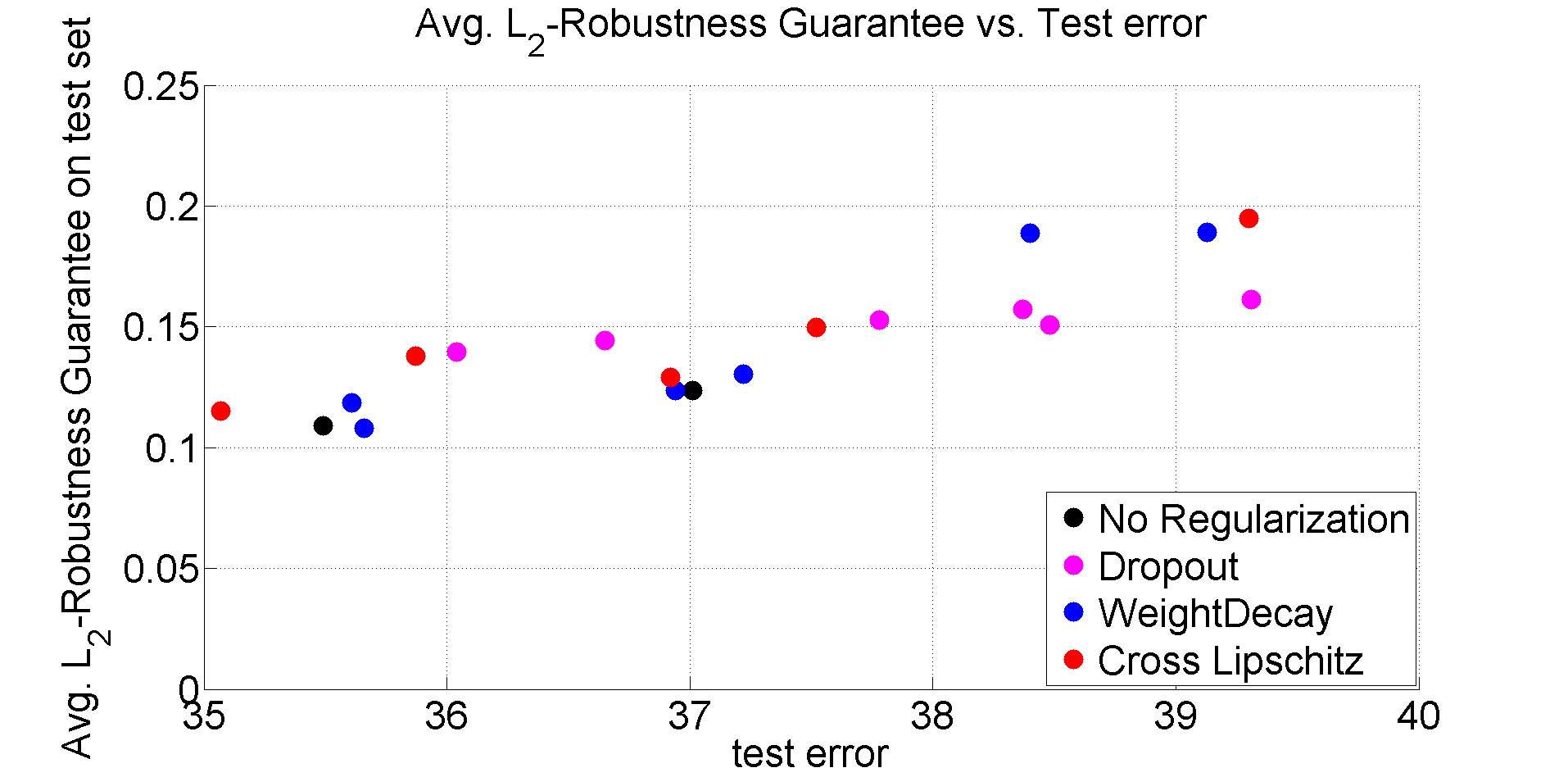}\\
 \includegraphics[width=0.45\textwidth]{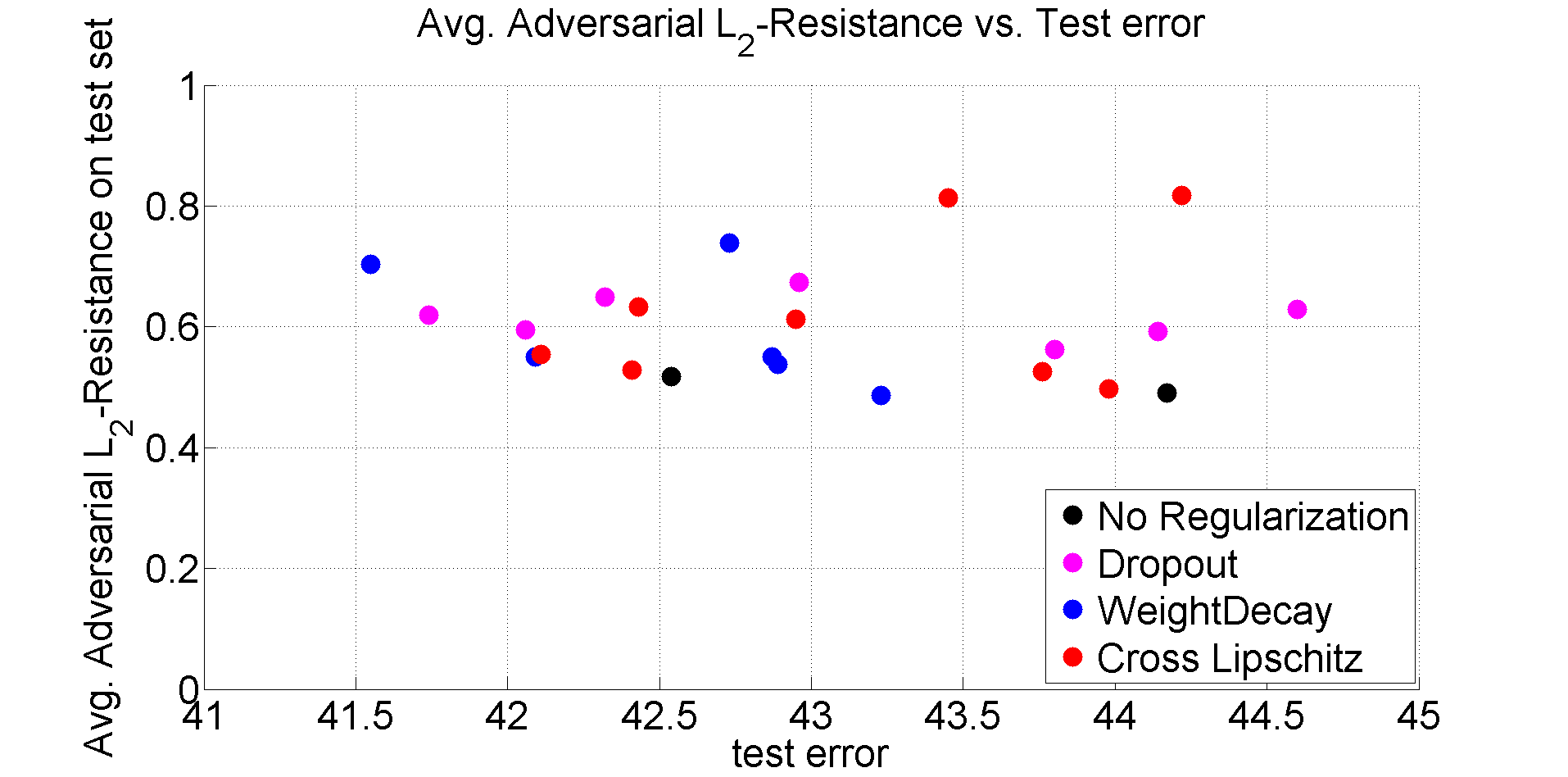}&\includegraphics[width=0.45\textwidth]{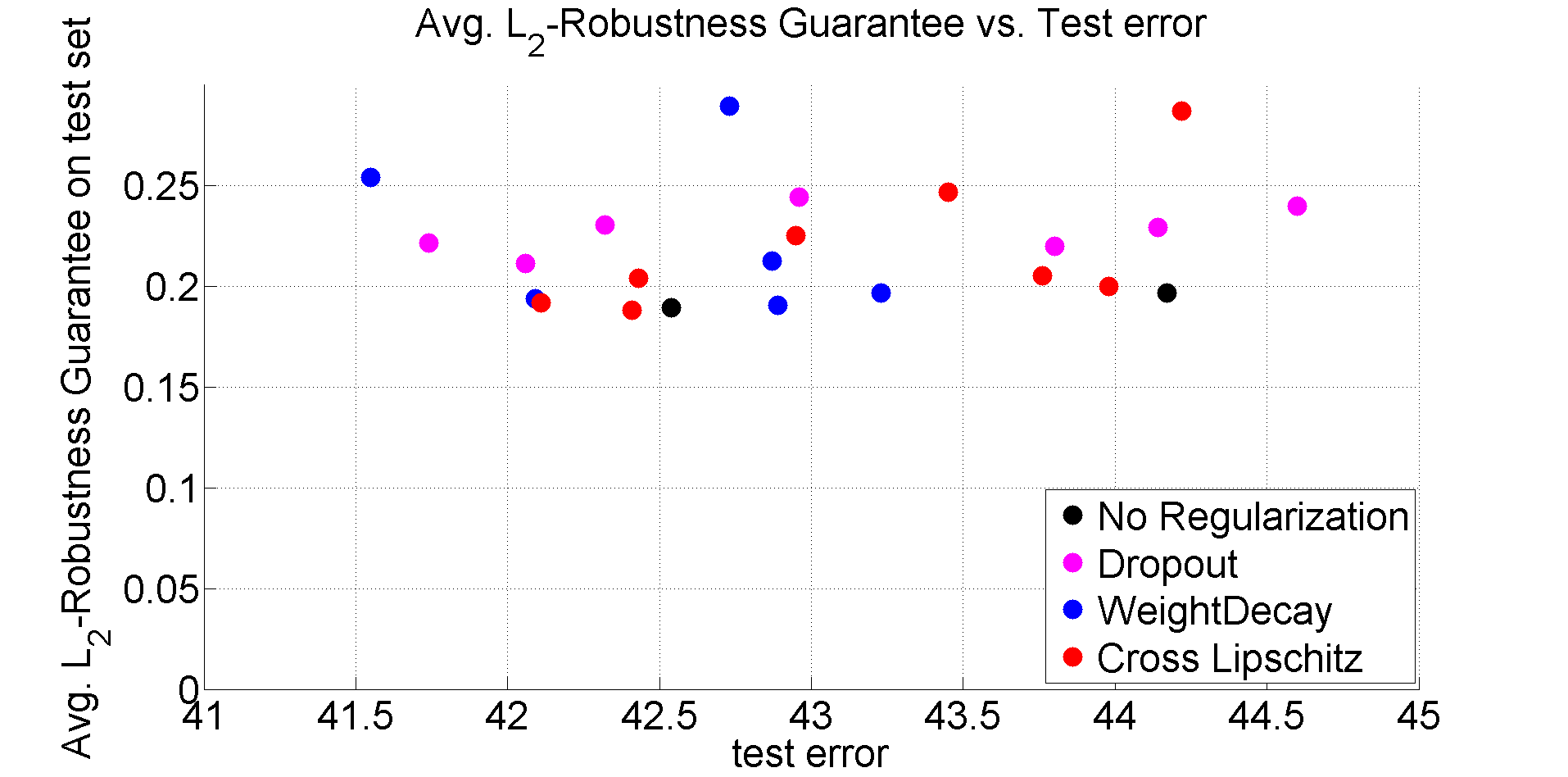}\\
\end{tabular}}
\caption{\label{exp:NN-CIFAR}  \textbf{Left:} Adversarial resistance wrt to $L_2$-norm on test set of CIFAR10. \textbf{Right:} Average robustness guarantee on the test set wrt to $L_2$-norm
for the test set of CIFAR10 for different neural networks (one hidden layer, 1024 HU) and hyperparameters. While Cross-Lipschitz regularization yields good test errors, the guarantees are not necessarily stronger. \textbf{Top row:} CIFAR10 (plain), \textbf{Middle:} CIFAR10 trained with data augmentation, \textbf{Bottom:} Adversarial Training.}
\end{figure}
\fi

\paragraph{Illustration of adversarial samples:} we take one test image from MNIST and apply the adversarial generation
from Section \ref{sec:adv} wrt to the $2$-norm to generate the adversarial samples for the different kernel methods and  neural networks
(plain setting), where we use for each method the parameters leading to best test performance. All classifiers change their originally correct decision to a ``wrong'' one. It is interesting to note that for Cross-Lipschitz regularization (both kernel method and neural network) 
the ``adversarial'' sample is really at the decision boundary between $1$ and $8$ (as predicted) and thus the new decision is actually correct.
This effect is strongest for our Kernel-CL, which also requires the
strongest modification to generate the adversarial sample. The situation is different for neural networks, where the classifiers obtained
from the two standard regularization techniques are still vulnerable, as the adversarial sample is still clearly a 1 for dropout and weight decay.
\begin{table}[b]
\begin{tabular}{ccc} \includegraphics[width=0.23\textwidth]{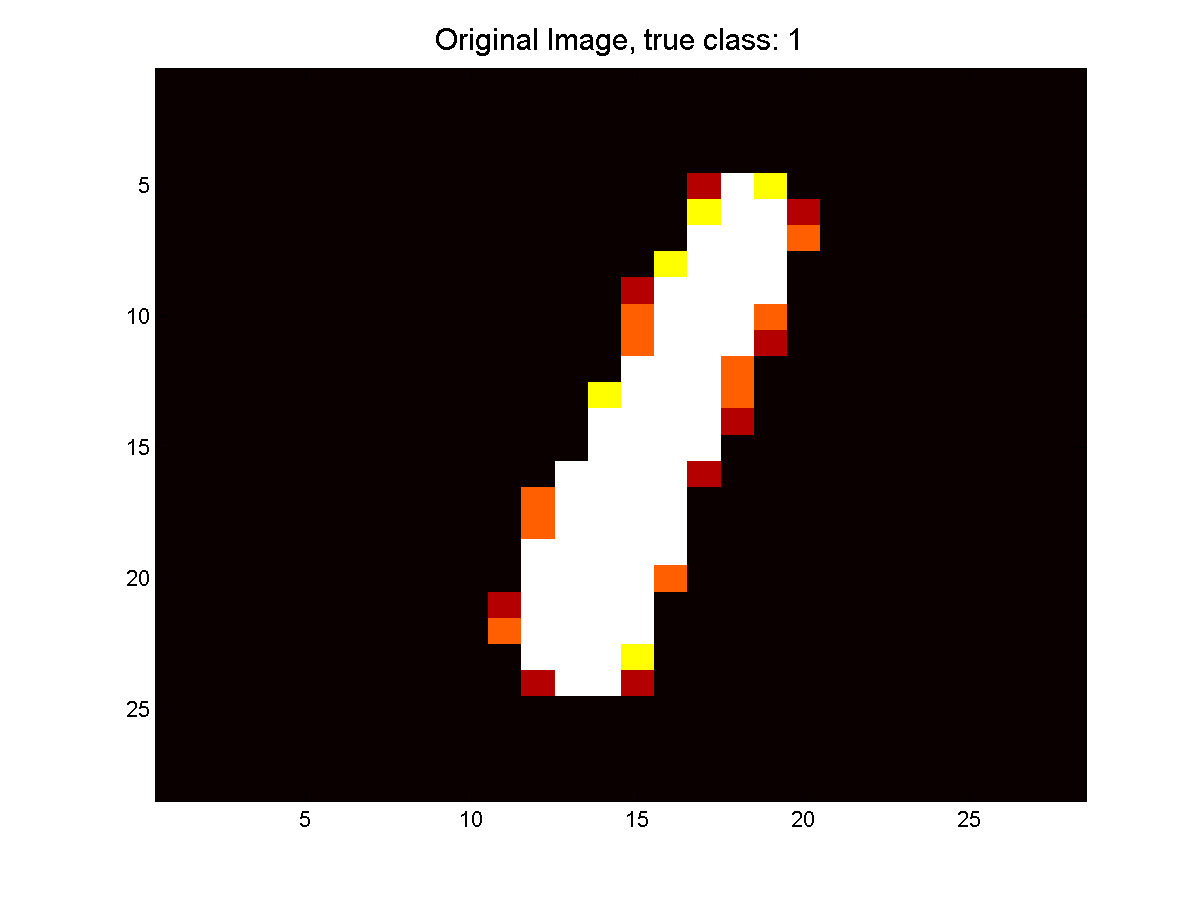}&\includegraphics[width=0.23\textwidth]{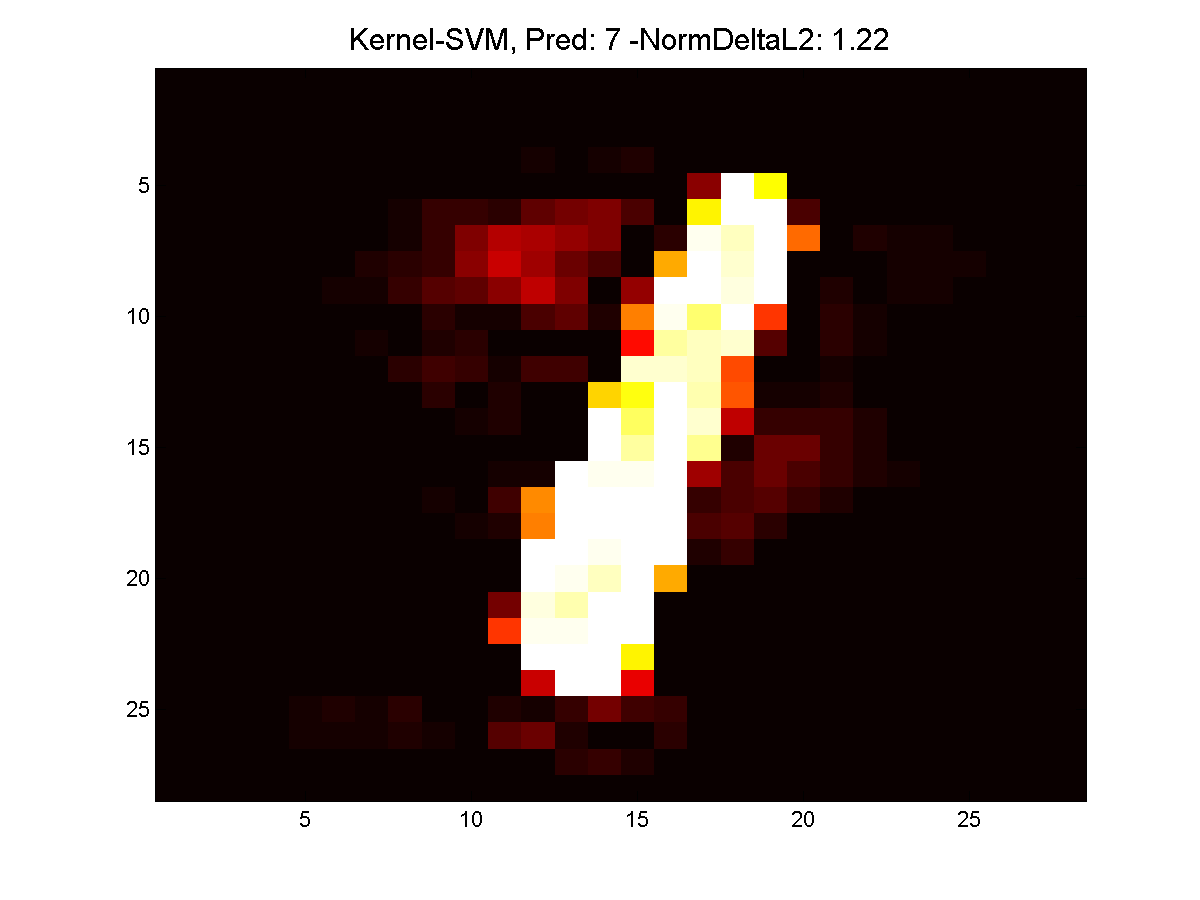}&\includegraphics[width=0.23\textwidth]{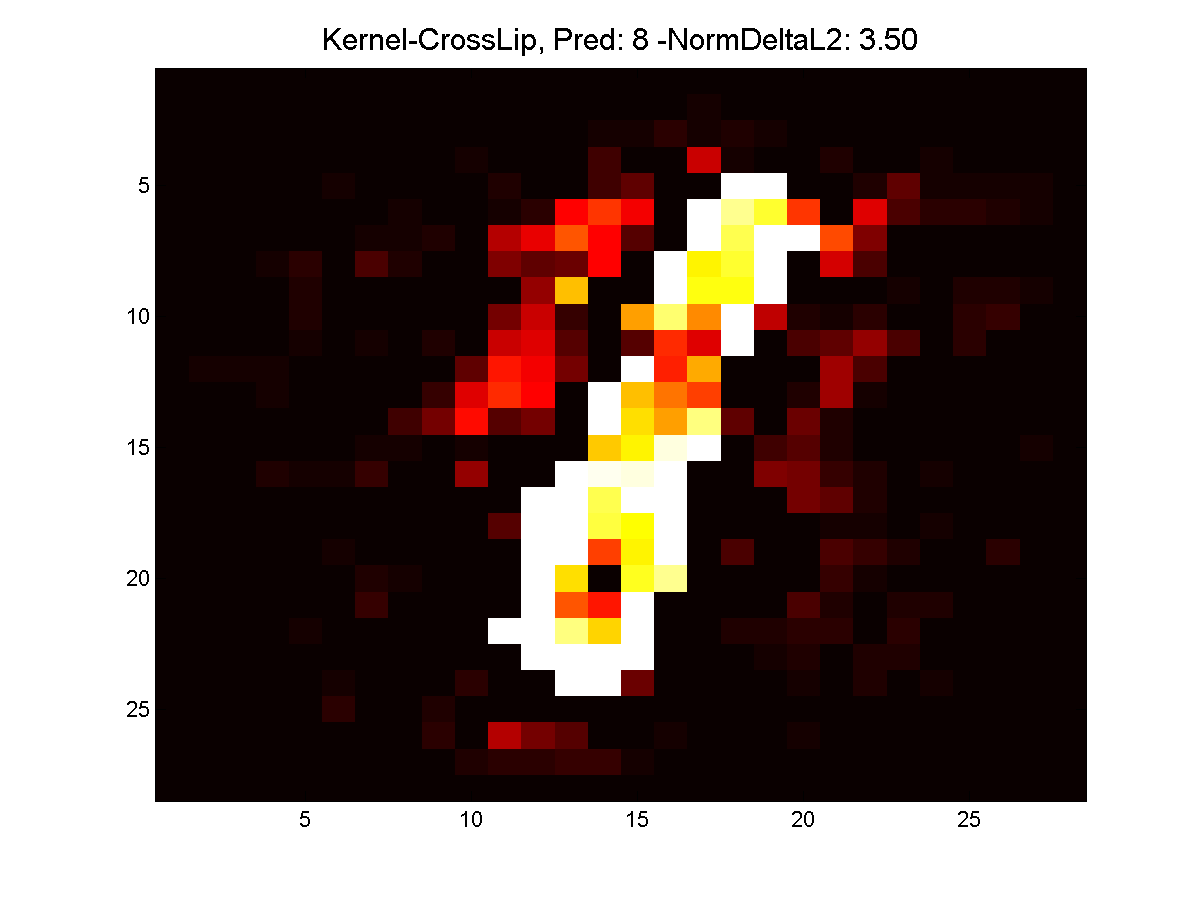}\\
Original, 	Class 1 & K-SVM, Pred:7, $\norm{\delta}_2=1.2$ & K-CL, Pred:8, $\norm{\delta}_2=3.5$\\
\includegraphics[width=0.23\textwidth]{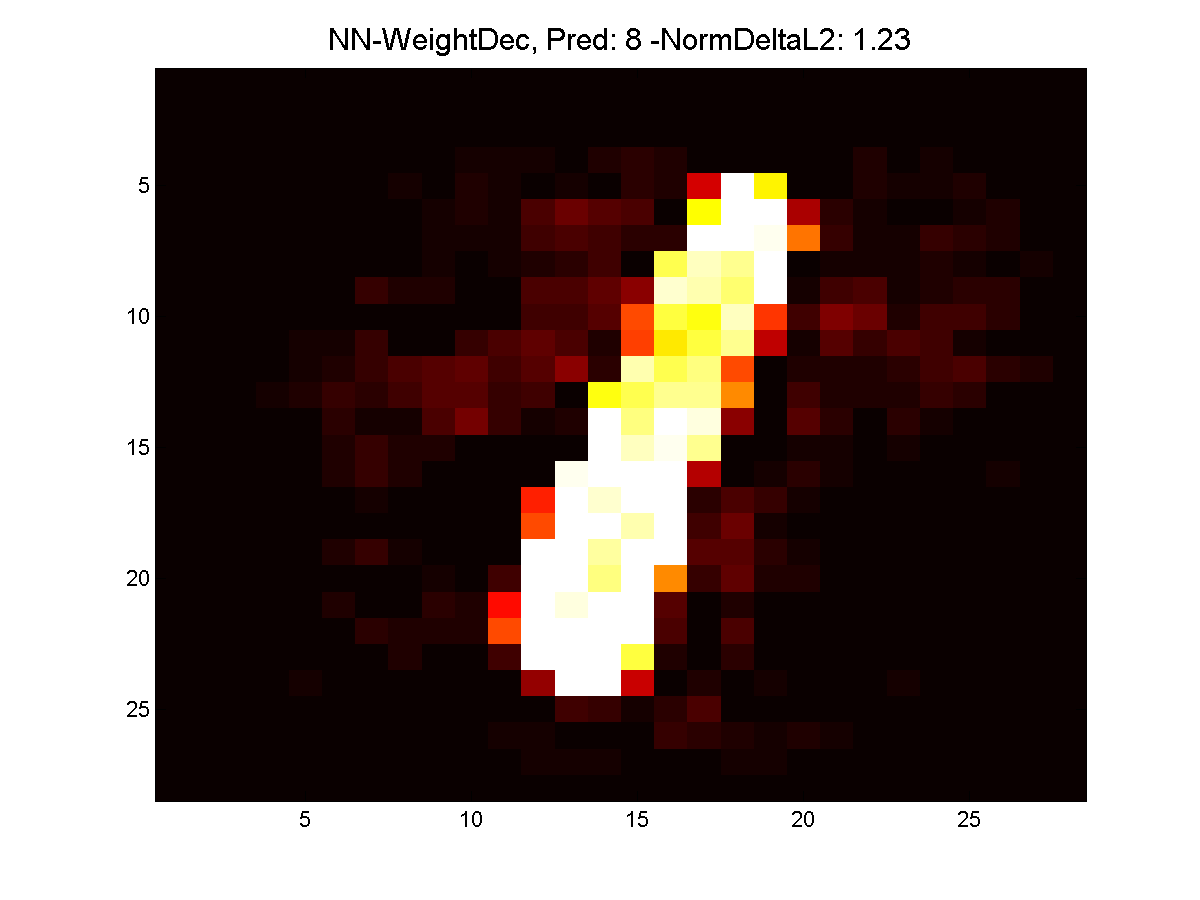}&\includegraphics[width=0.23\textwidth]{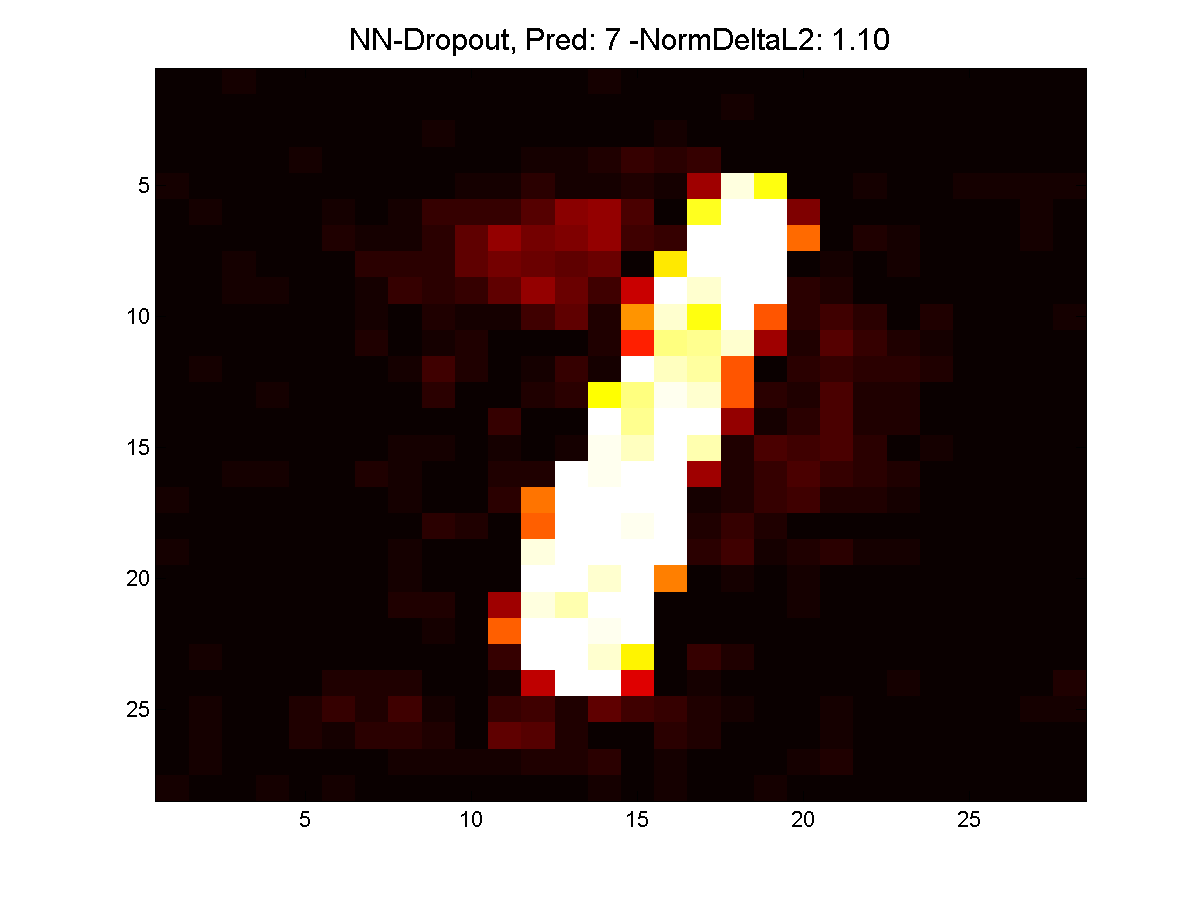}&\includegraphics[width=0.23\textwidth]{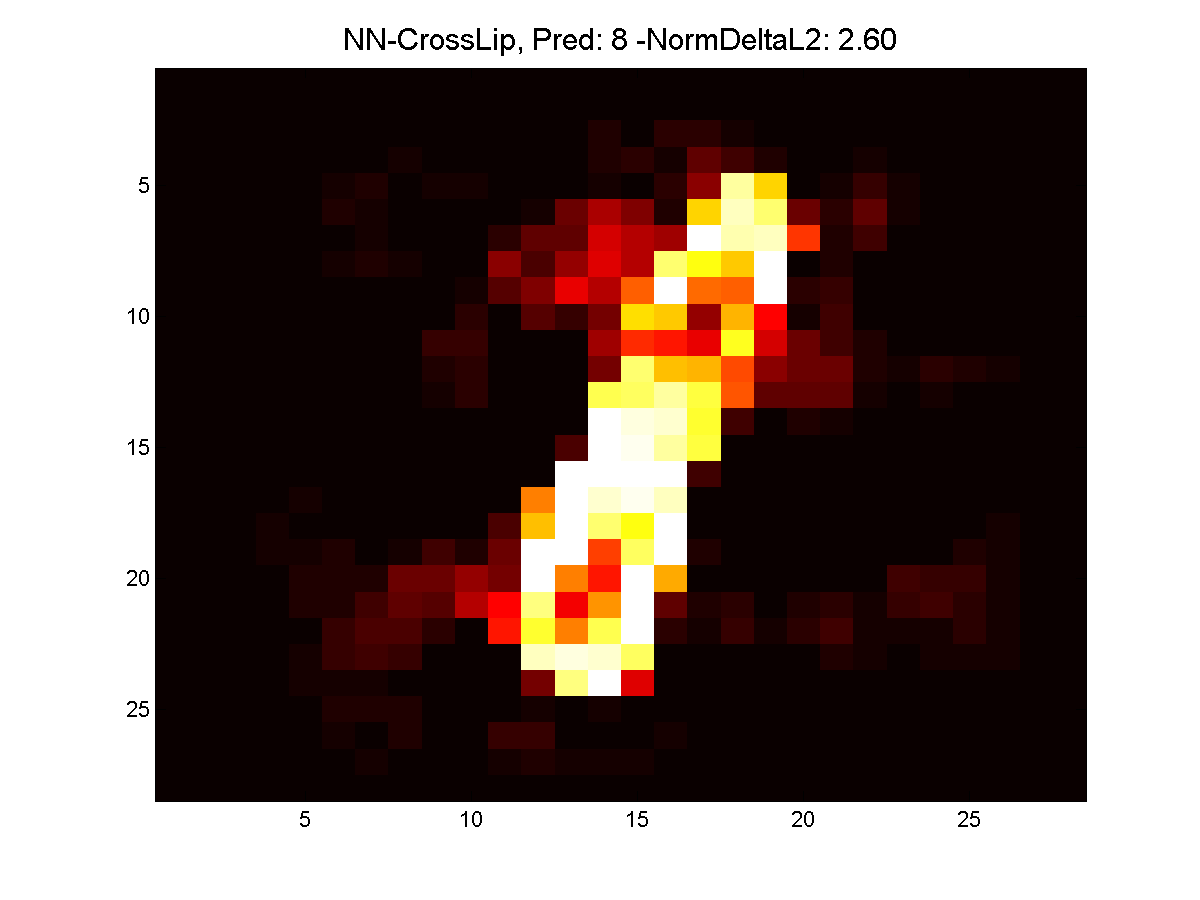}\\
NN-WD, Pred:8, $\norm{\delta}_2=1.2$ & NN-DO, Pred:7, $\norm{\delta}_2=1.1$ & NN-CL, Pred:8, $\norm{\delta}_2=2.6$
\end{tabular}
\captionof{figure}{\label{fig:adv} Top left: original test image, for each classifier we generate the corresponding adversarial sample
which changes the classifier decision (denoted as Pred). Note that for Cross-Lipschitz regularization this new decision makes (often) sense, whereas
for the neural network models (weight decay/dropout) the change is so small that the new decision is clearly wrong.}
\end{table}
\ifpaper
We show further examples below.
\begin{center}
\begin{tabular}{ccc} \includegraphics[width=0.23\textwidth]{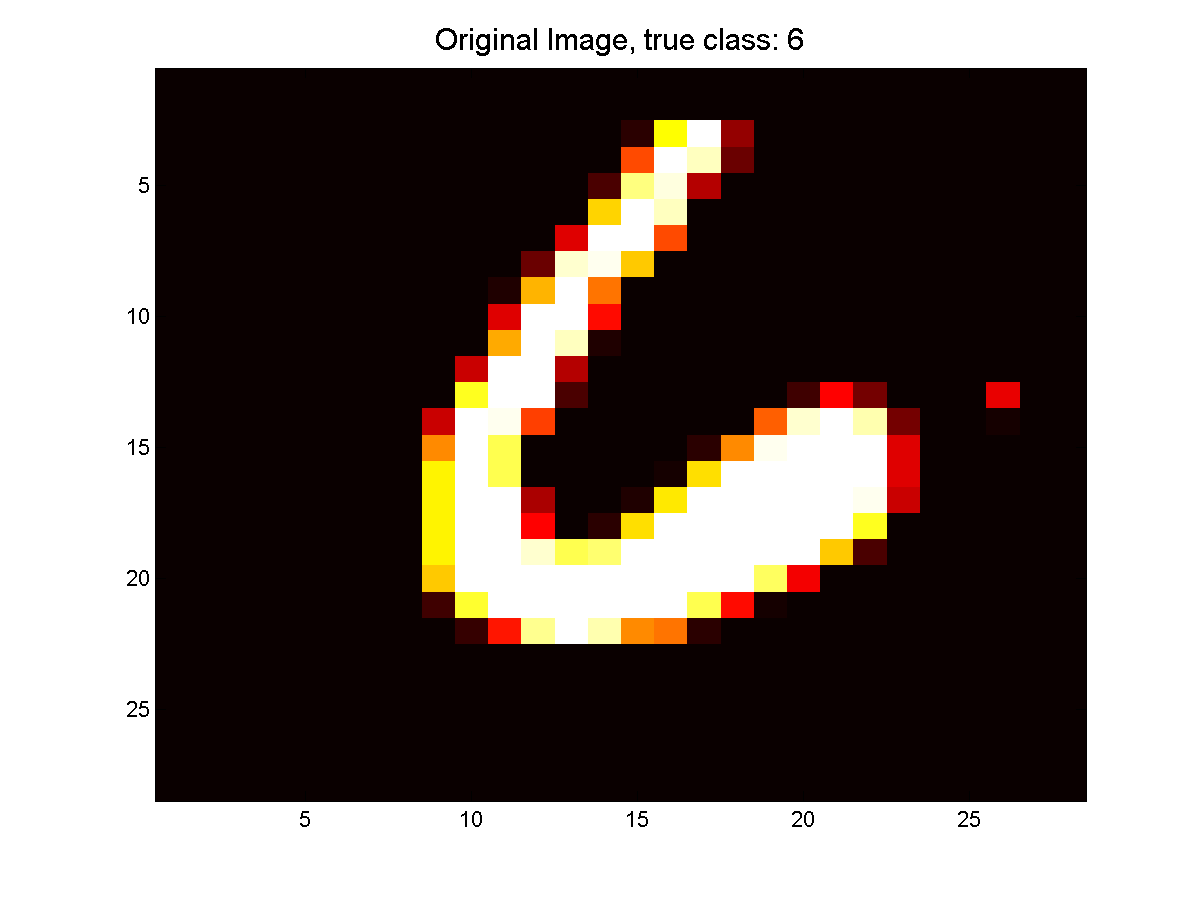}&\includegraphics[width=0.23\textwidth]{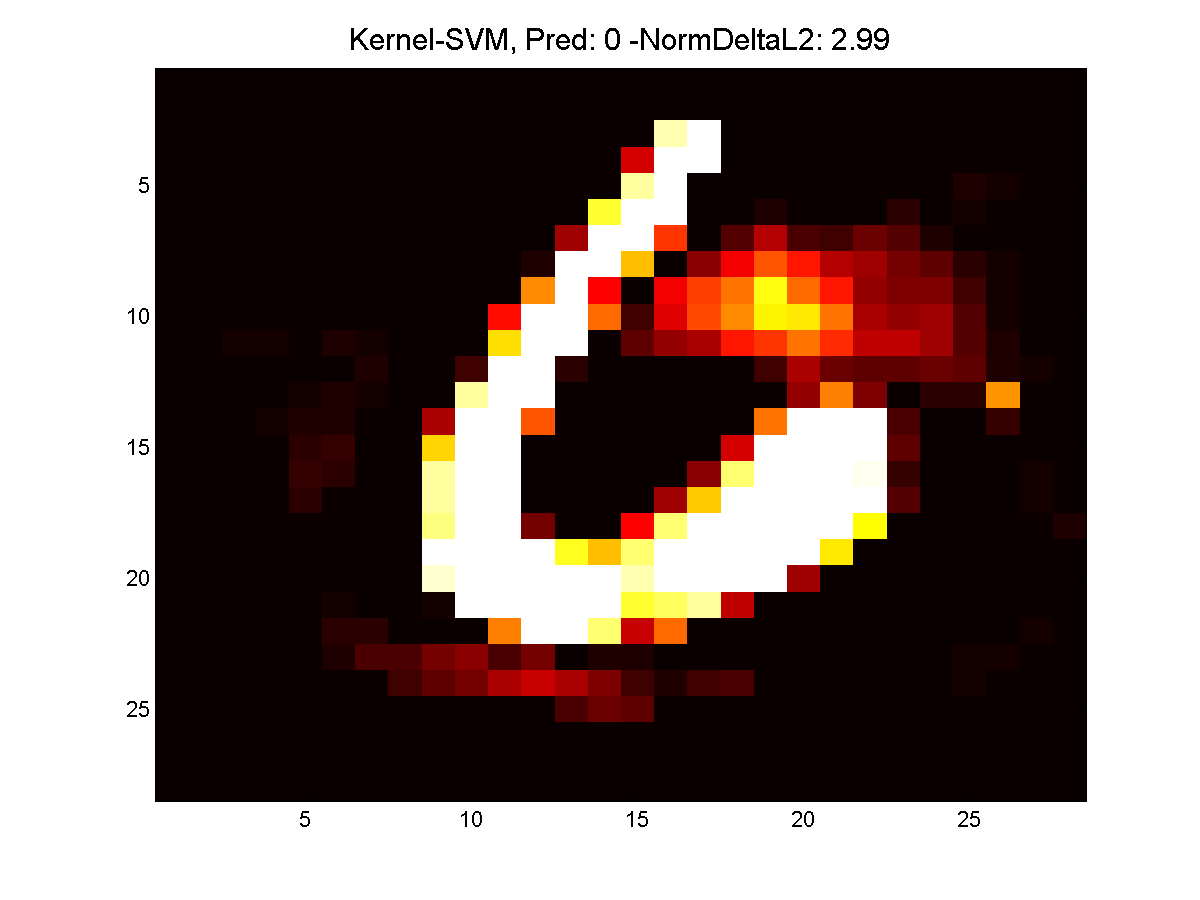}&\includegraphics[width=0.23\textwidth]{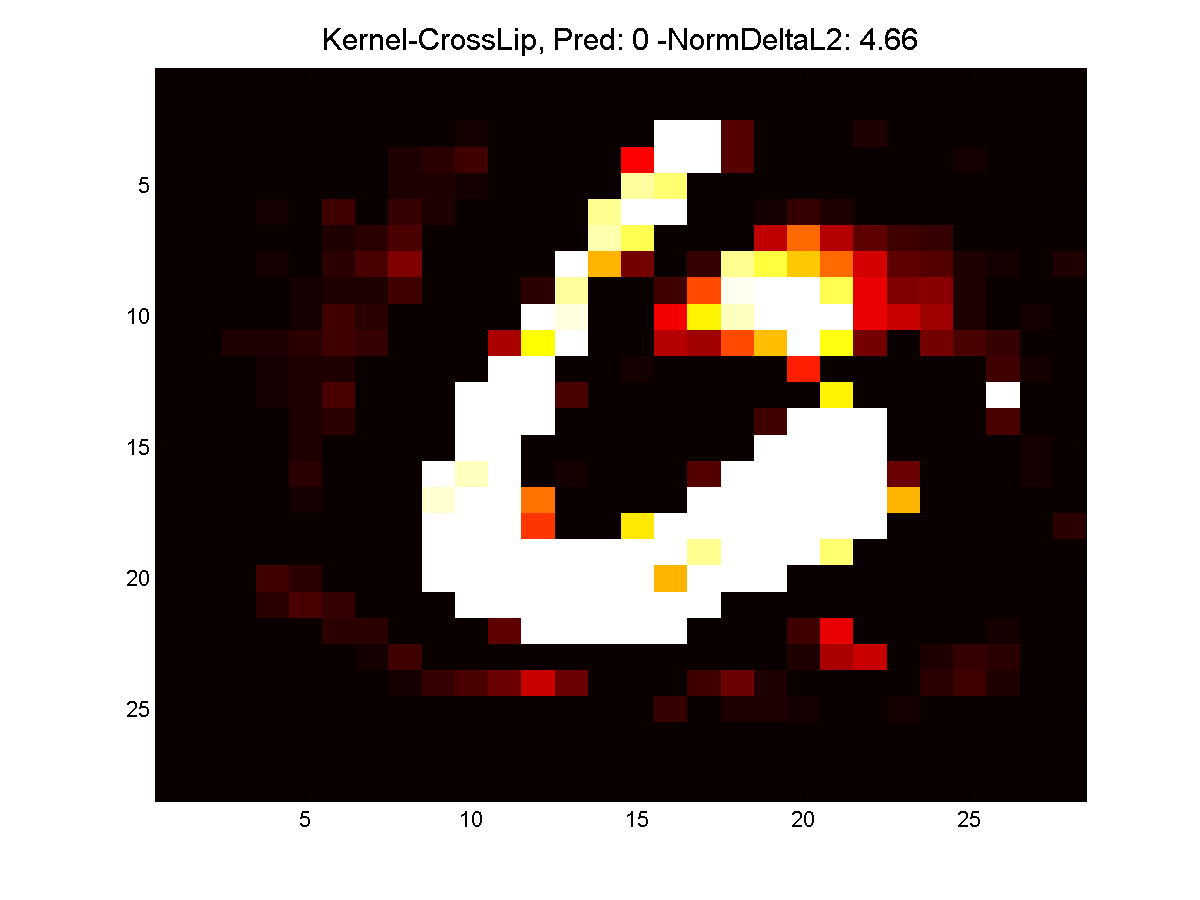}\\
Original, 	Class 6 & K-SVM, Pred:0, $\norm{\delta}_2=3.0$ & K-CL, Pred:0, $\norm{\delta}_2=4.7$\\
\includegraphics[width=0.23\textwidth]{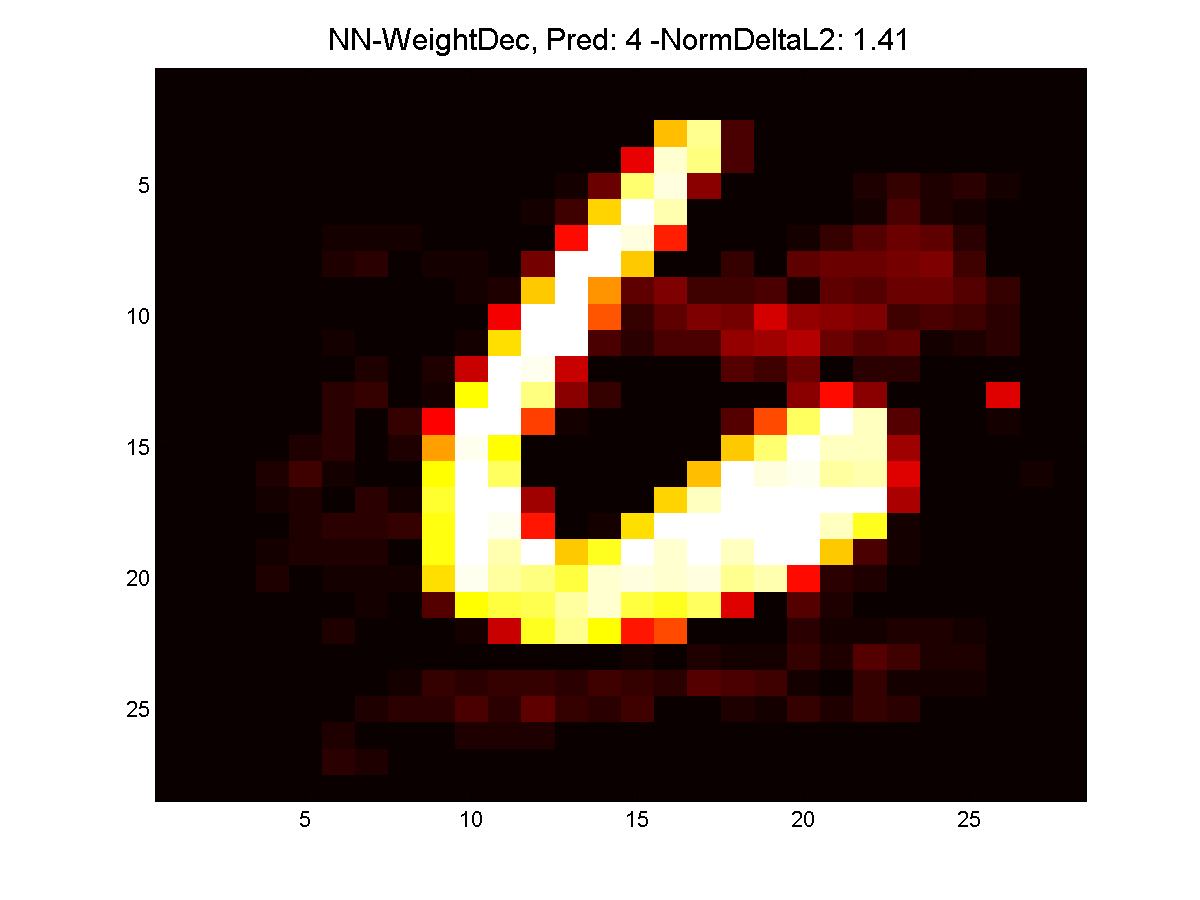}&\includegraphics[width=0.23\textwidth]{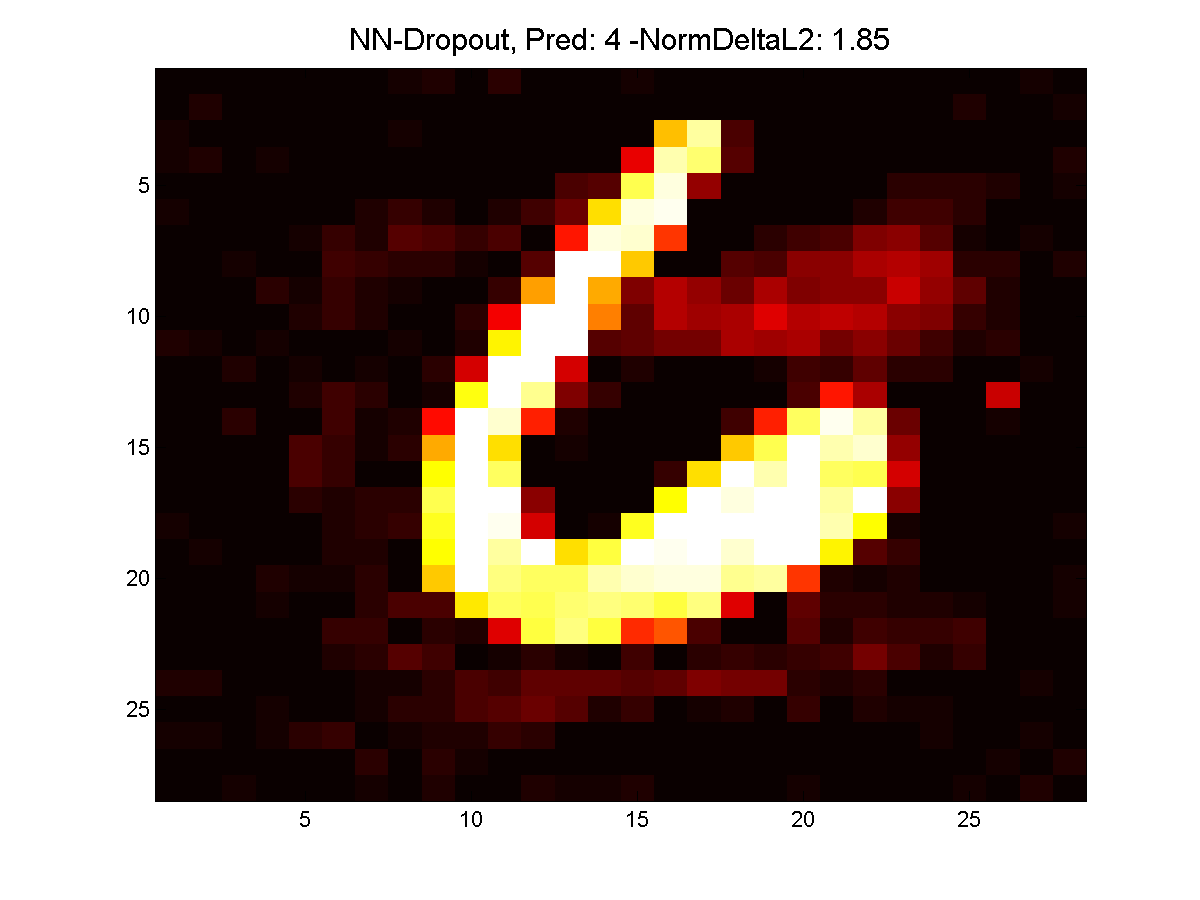}&\includegraphics[width=0.23\textwidth]{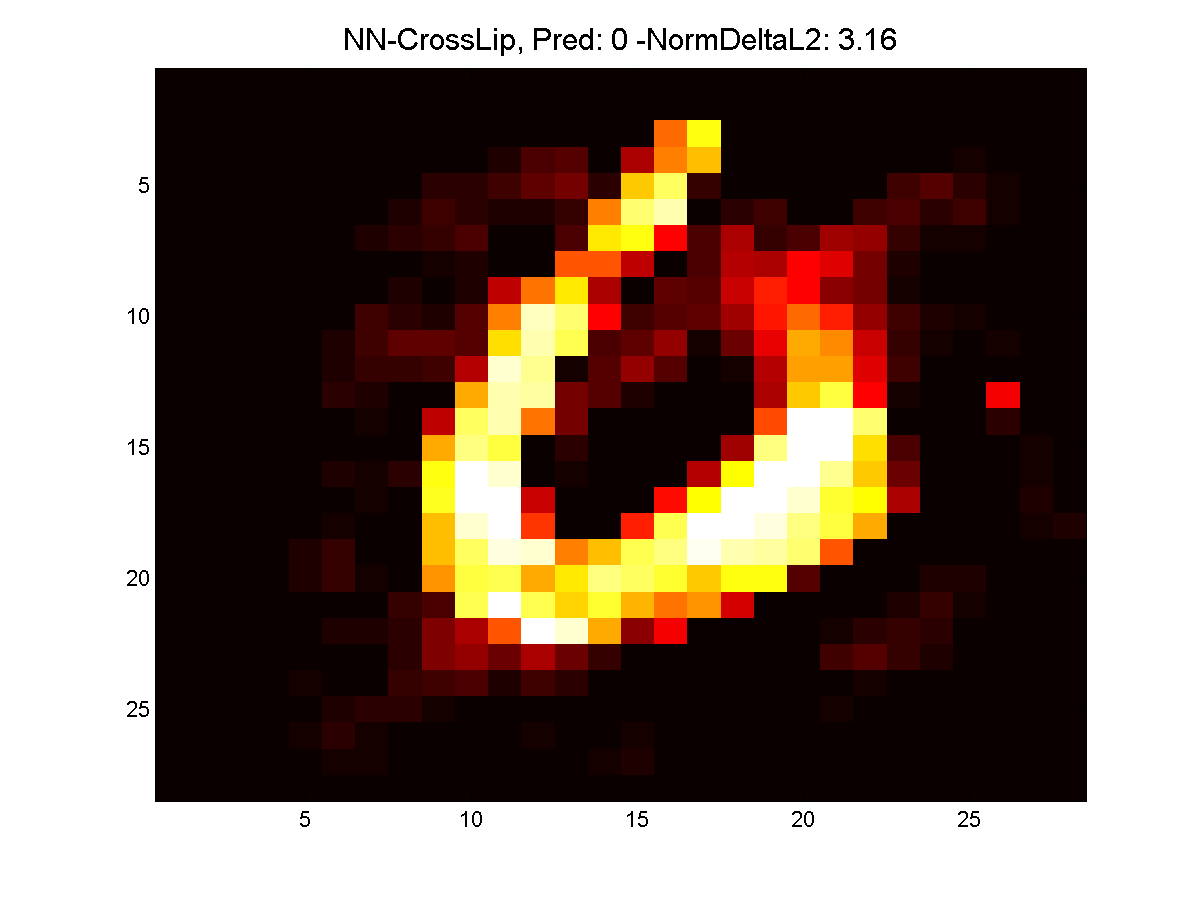}\\
NN-WD, Pred:4, $\norm{\delta}_2=1.4$ & NN-DO, Pred:4, $\norm{\delta}_2=1.9$ & NN-CL, Pred:0, $\norm{\delta}_2=3.2$
\end{tabular}
\captionof{figure}{Top left: original test image, for each classifier we generate the corresponding adversarial sample
which changes the classifier decision (denoted as Pred). Note that for the kernel methods this new decision makes sense, whereas
for all neural network models the change is so small that the new decision is clearly wrong.}
\end{center}
\begin{center}
\begin{tabular}{ccc} \includegraphics[width=0.23\textwidth]{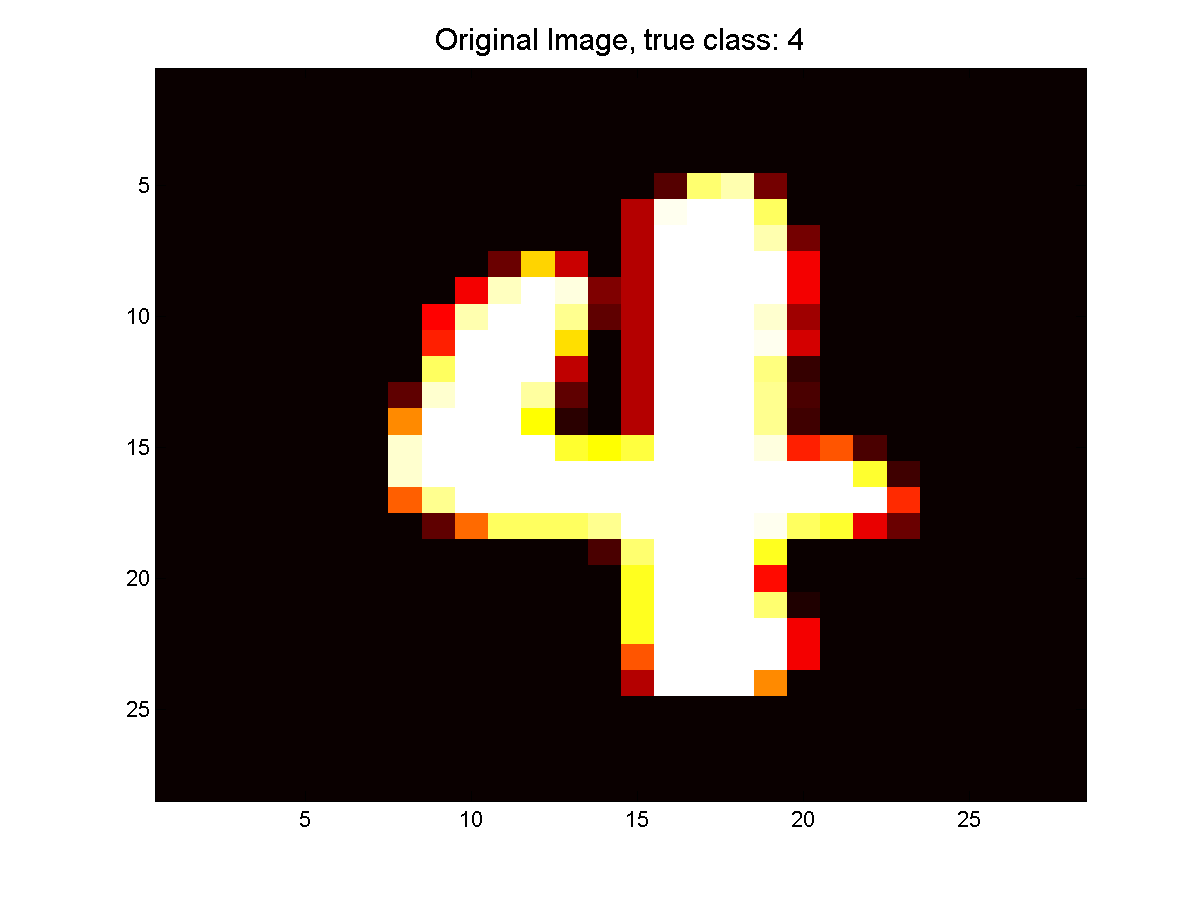}&\includegraphics[width=0.23\textwidth]{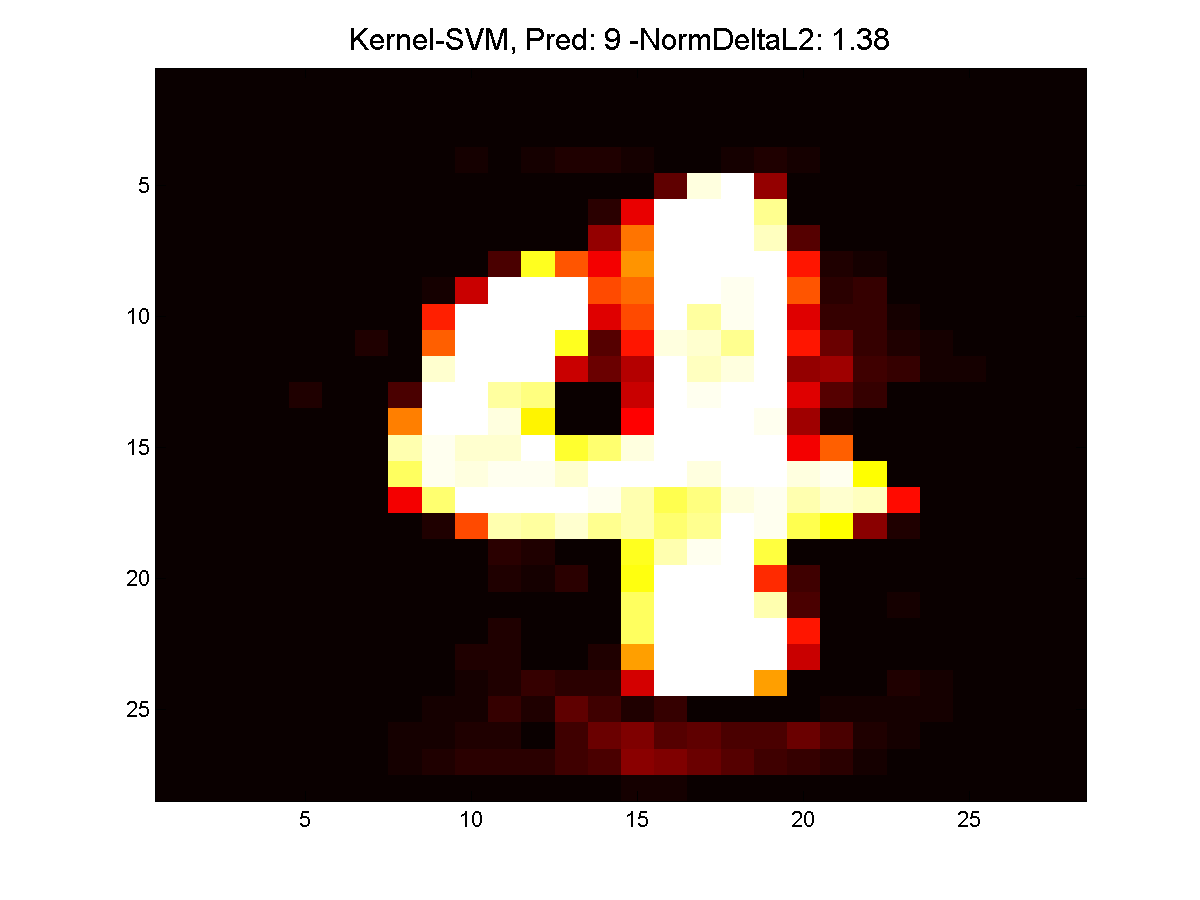}&\includegraphics[width=0.23\textwidth]{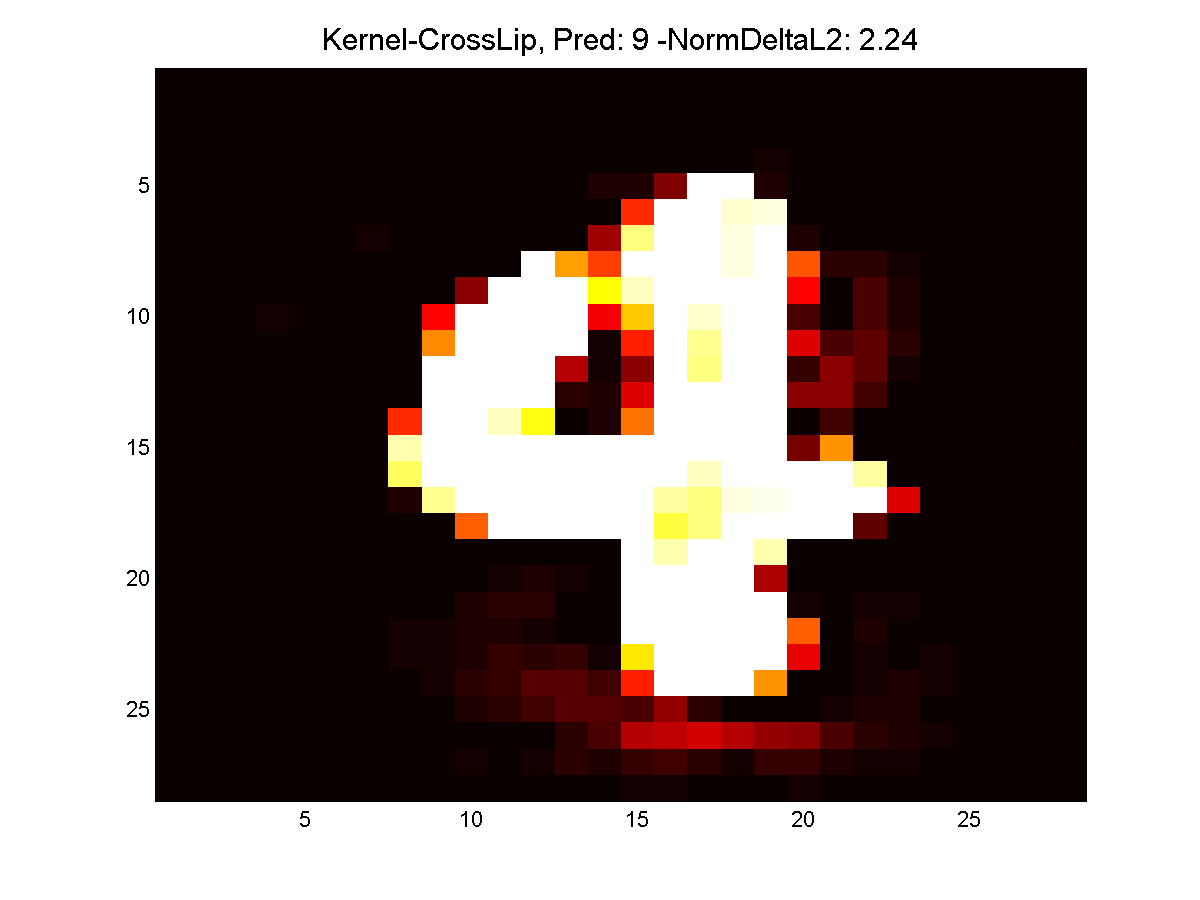}\\
Original, 	Class 4 & K-SVM, Pred:9, $\norm{\delta}_2=1.4$ & K-CL, Pred:9, $\norm{\delta}_2=2.2$\\
\includegraphics[width=0.23\textwidth]{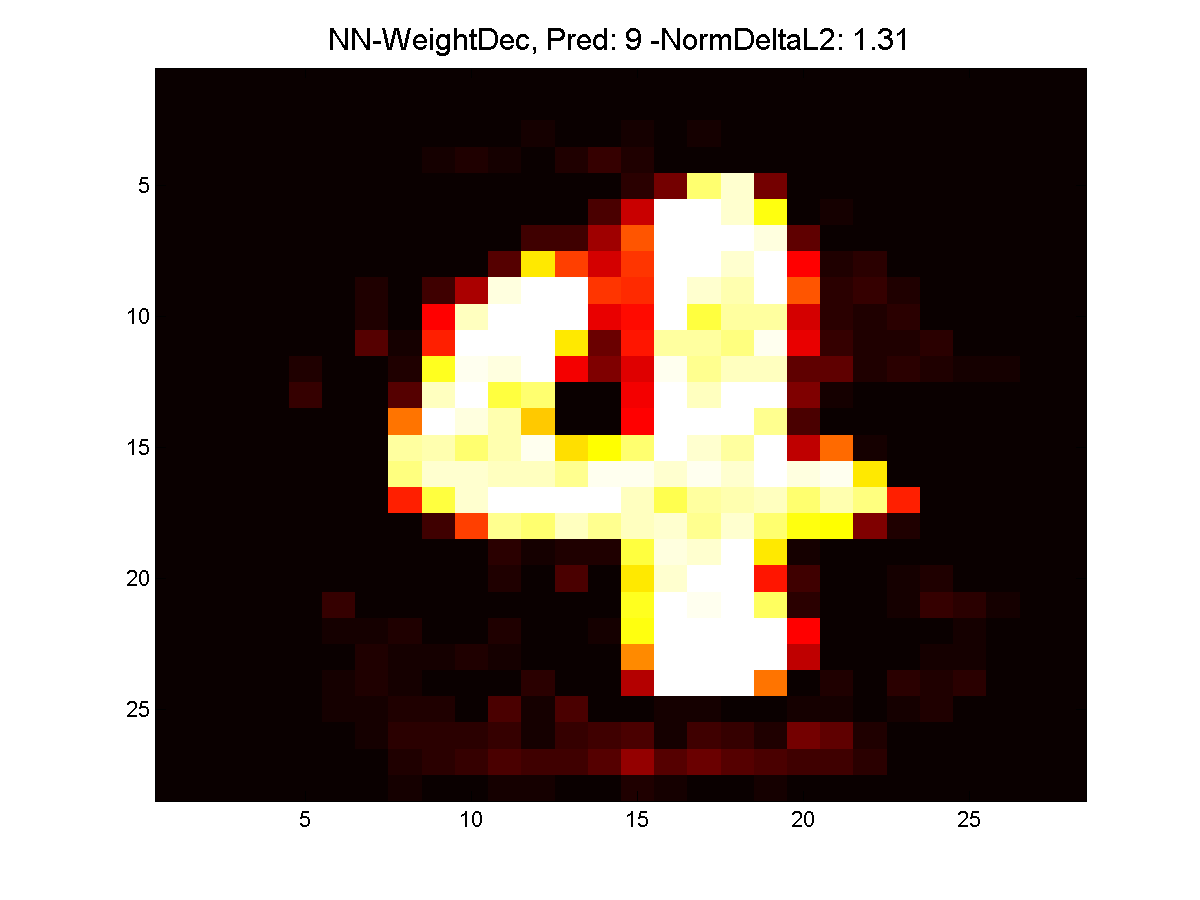}&\includegraphics[width=0.23\textwidth]{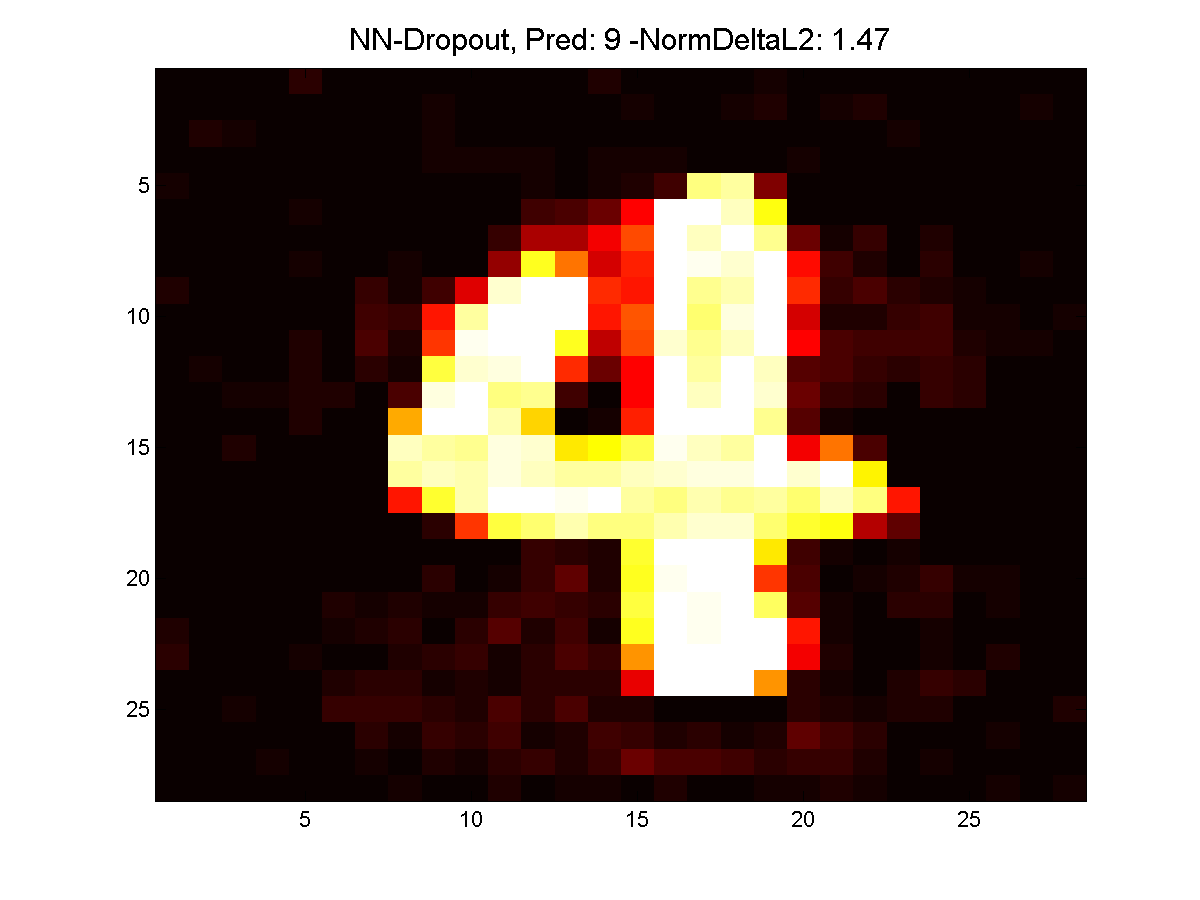}&\includegraphics[width=0.23\textwidth]{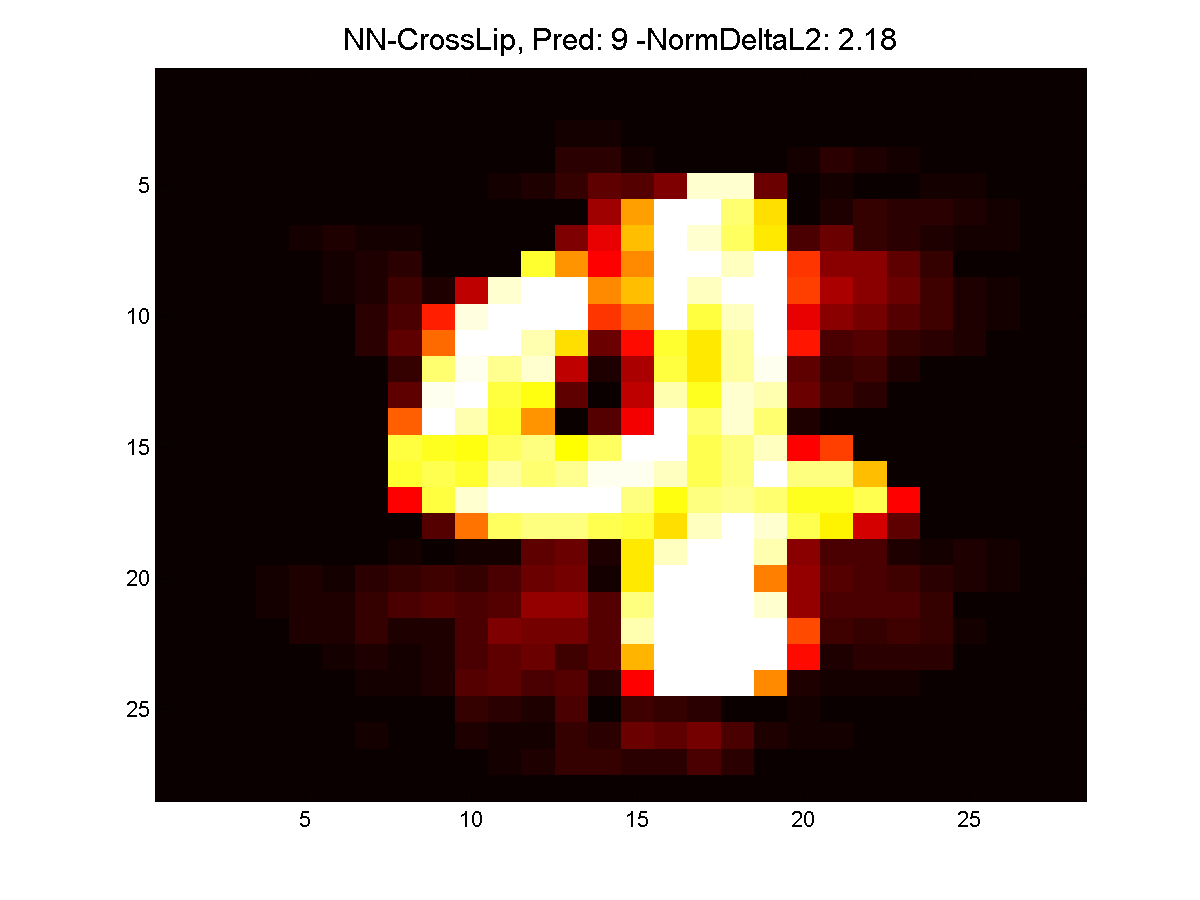}\\
NN-WD, Pred:9, $\norm{\delta}_2=1.3$ & NN-DO, Pred:9, $\norm{\delta}_2=1.5$ & NN-CL, Pred:9, $\norm{\delta}_2=2.2$
\end{tabular}
\captionof{figure}{Top left: original test image, for each classifier we generate the corresponding adversarial sample
which changes the classifier decision (denoted as Pred). Note that for the kernel methods this new decision makes sense, whereas
for all neural network models the change is so small that the new decision is clearly wrong.}
\end{center}
\begin{center}
\begin{tabular}{ccc} \includegraphics[width=0.23\textwidth]{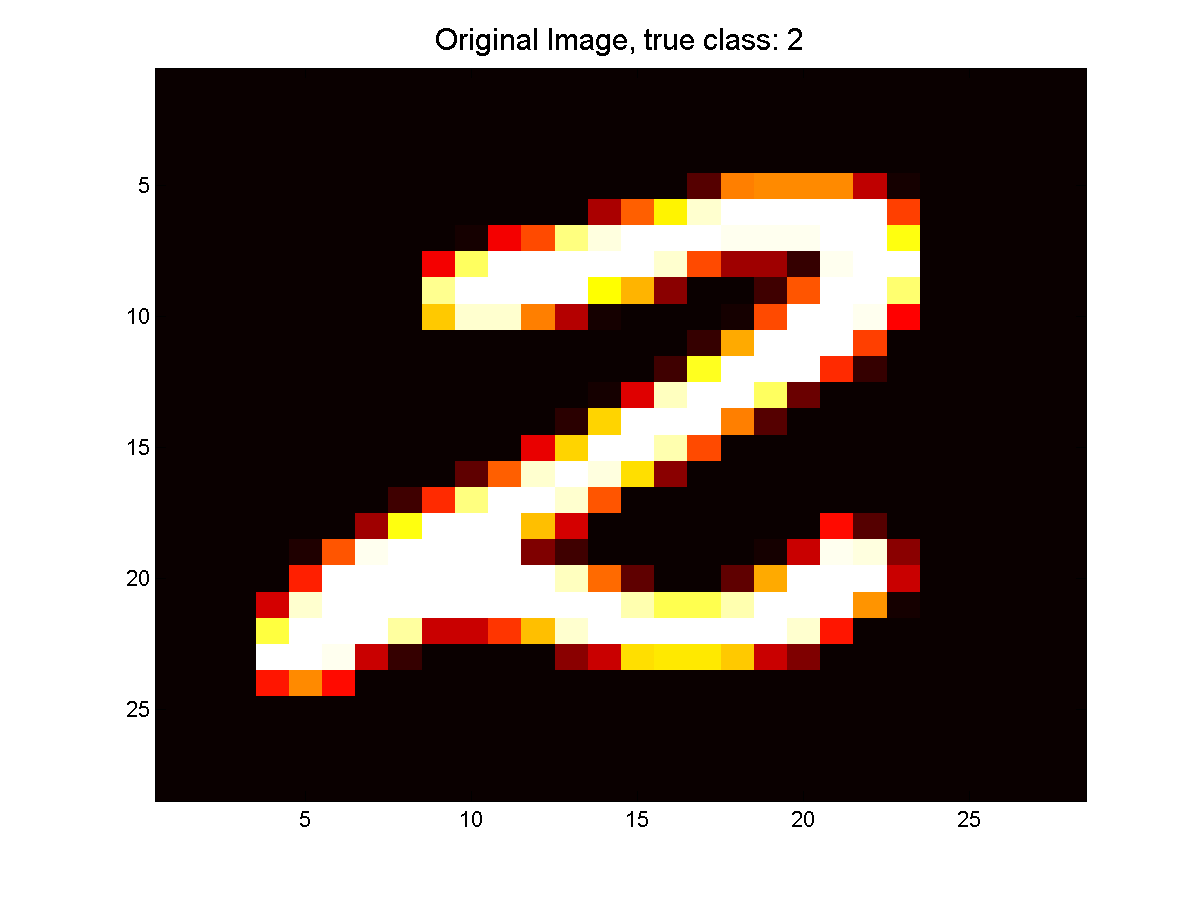}&\includegraphics[width=0.23\textwidth]{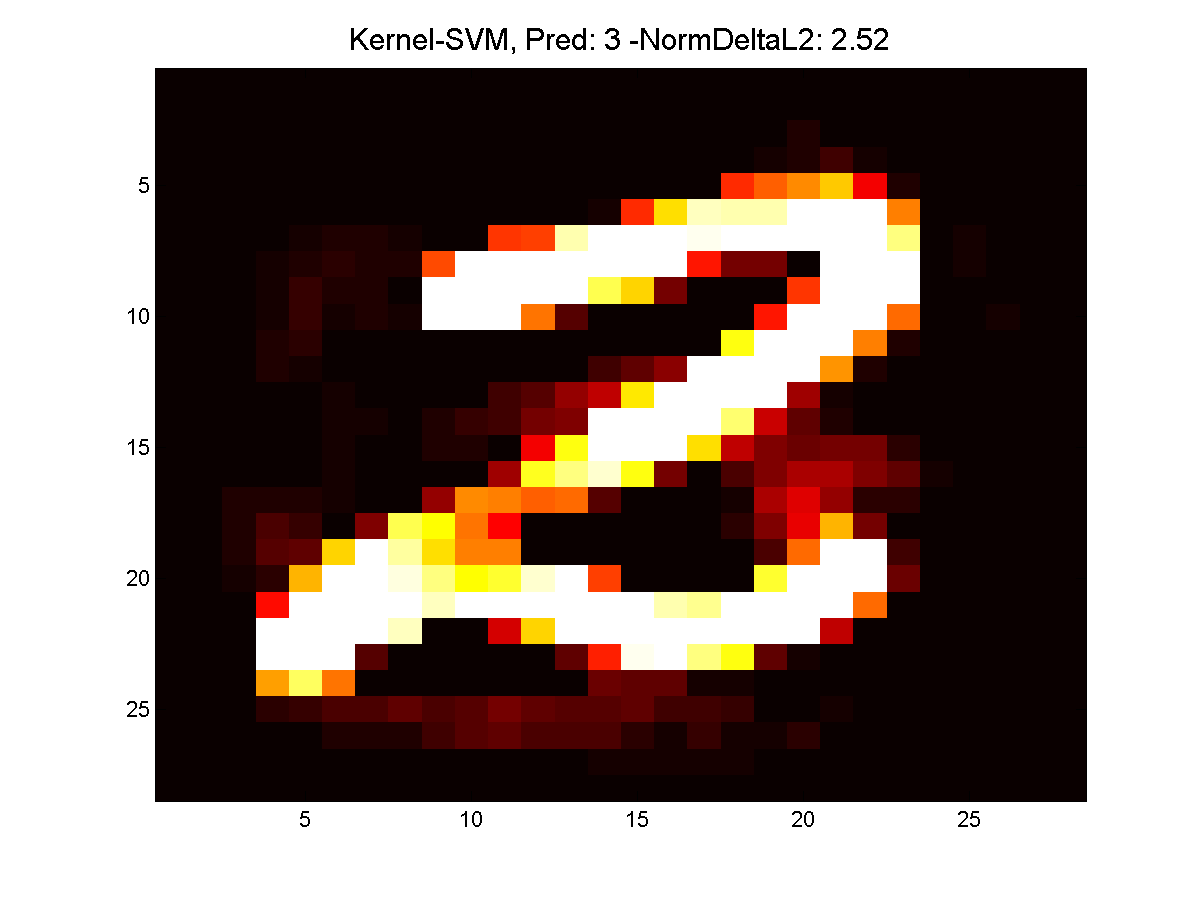}&\includegraphics[width=0.23\textwidth]{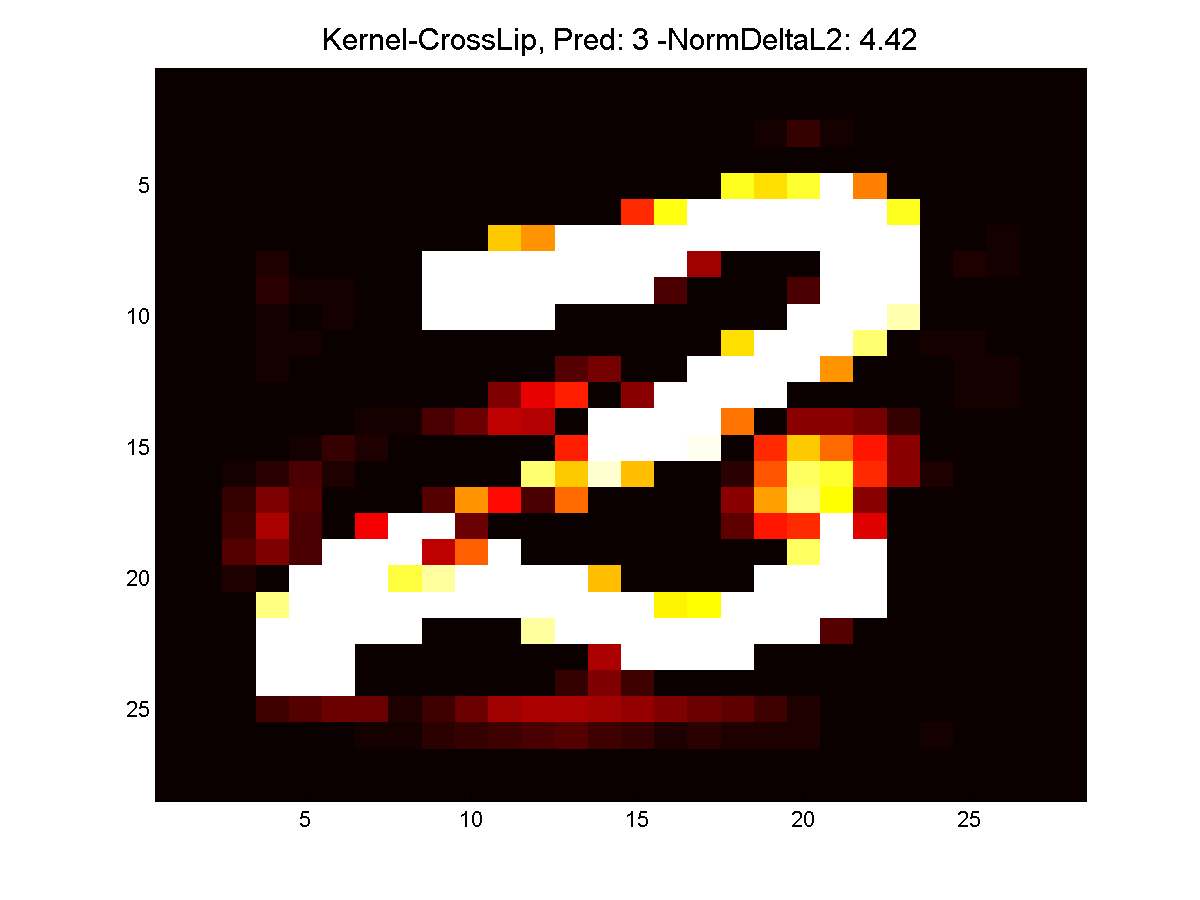}\\
Original, 	Class 2 & K-SVM, Pred:3, $\norm{\delta}_2=2.5$ & K-CL, Pred:3, $\norm{\delta}_2=4.4$\\
\includegraphics[width=0.23\textwidth]{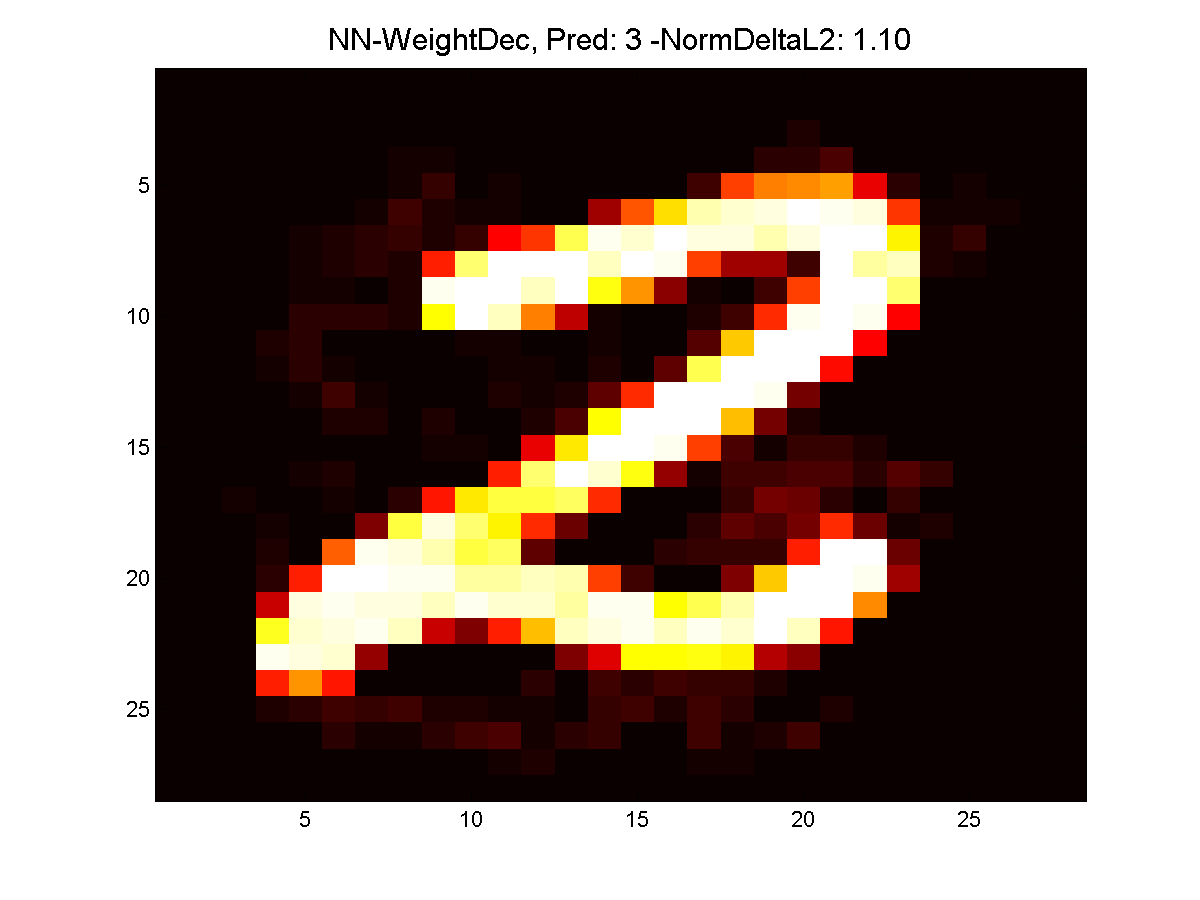}&\includegraphics[width=0.23\textwidth]{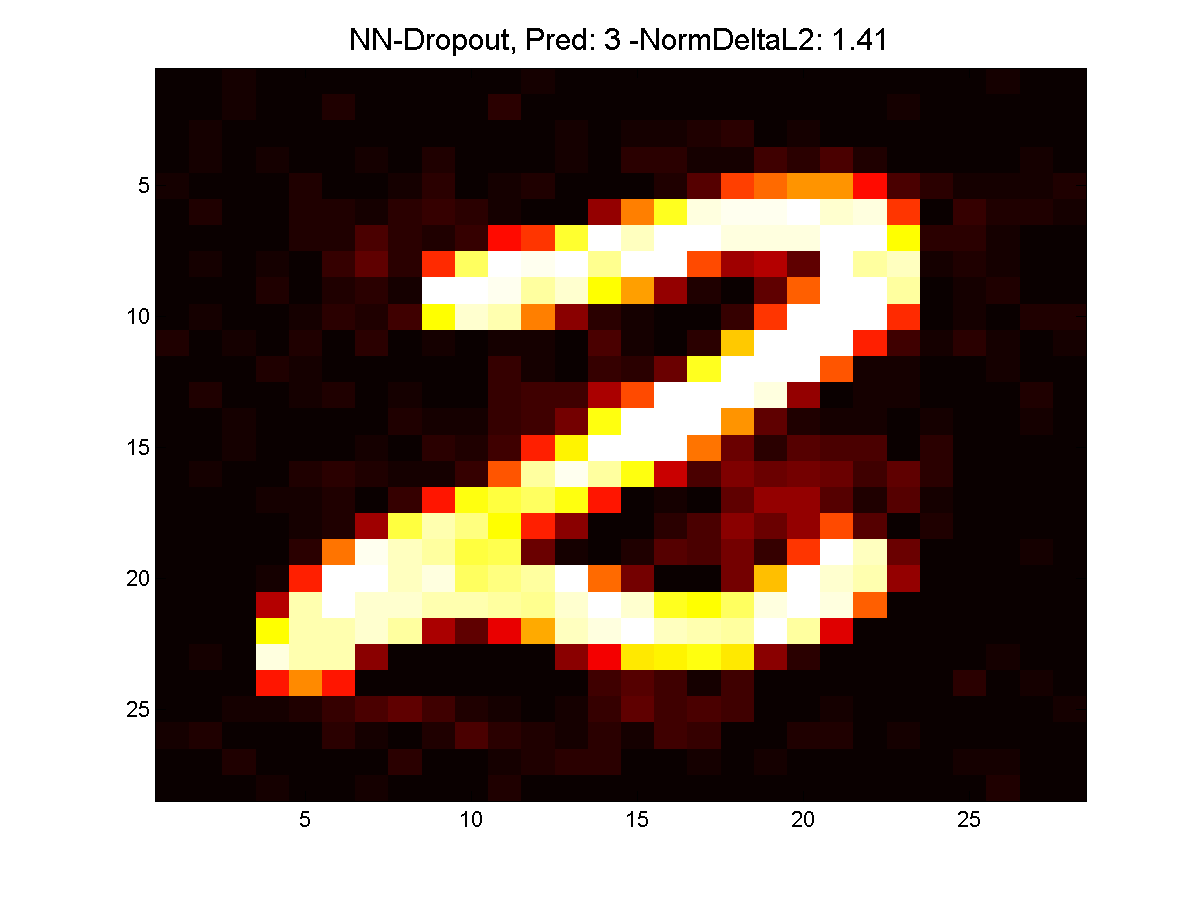}&\includegraphics[width=0.23\textwidth]{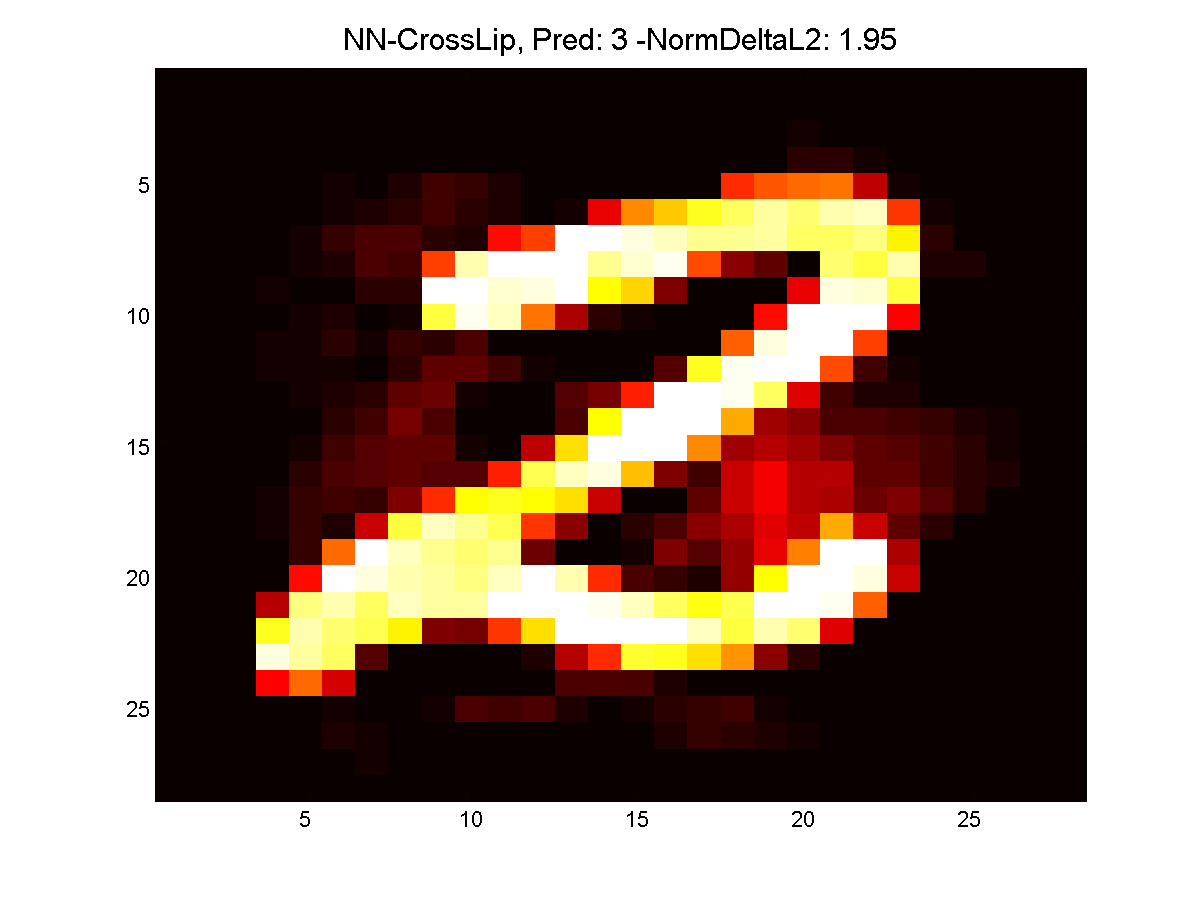}\\
NN-WD, Pred:3, $\norm{\delta}_2=1.1$ & NN-DO, Pred:3, $\norm{\delta}_2=1.4$ & NN-CL, Pred:3, $\norm{\delta}_2=2.0$
\end{tabular}
\captionof{figure}{Top left: original test image, for each classifier we generate the corresponding adversarial sample
which changes the classifier decision (denoted as Pred). Note that for the kernel methods this new decision makes sense, whereas
for all neural network models the change is so small that the new decision is clearly wrong.}
\end{center}
\begin{center}
\begin{tabular}{ccc} \includegraphics[width=0.23\textwidth]{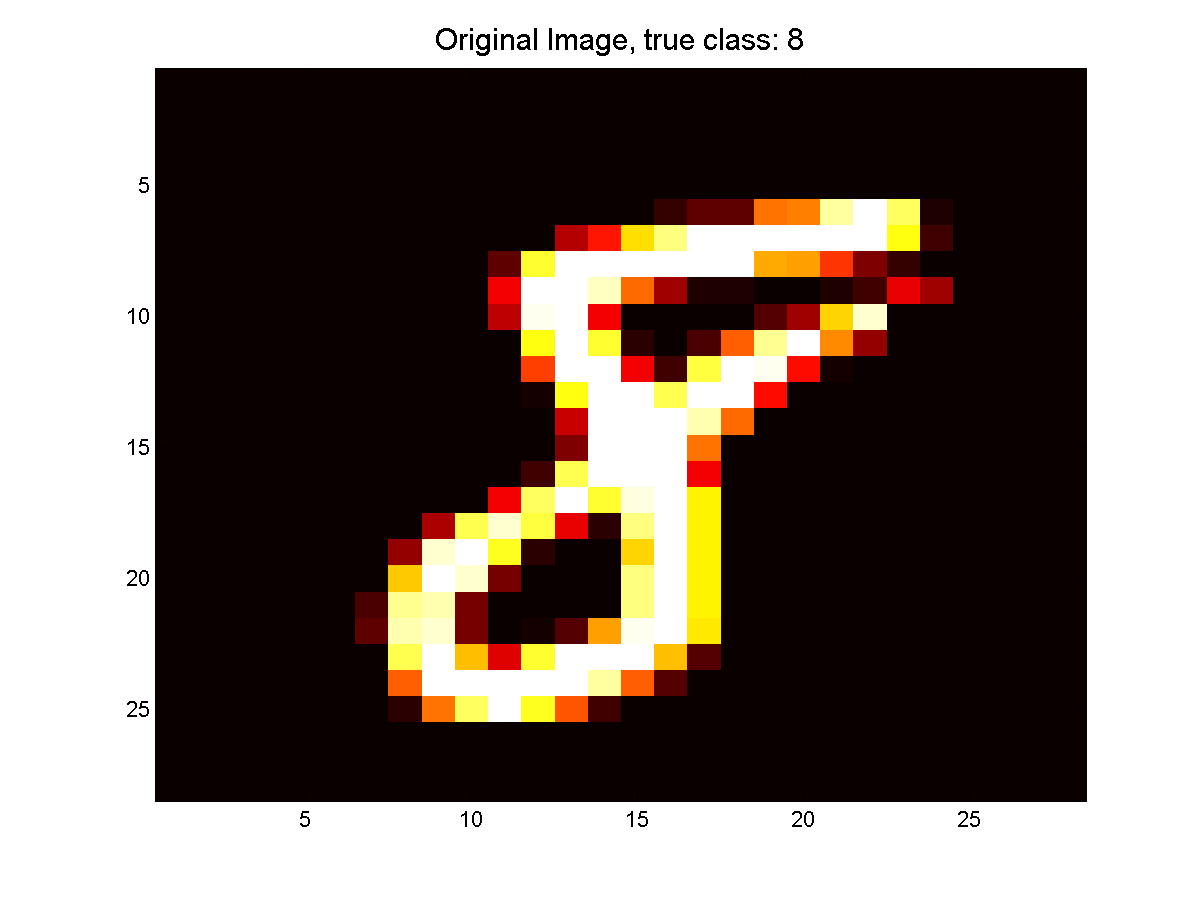}&\includegraphics[width=0.23\textwidth]{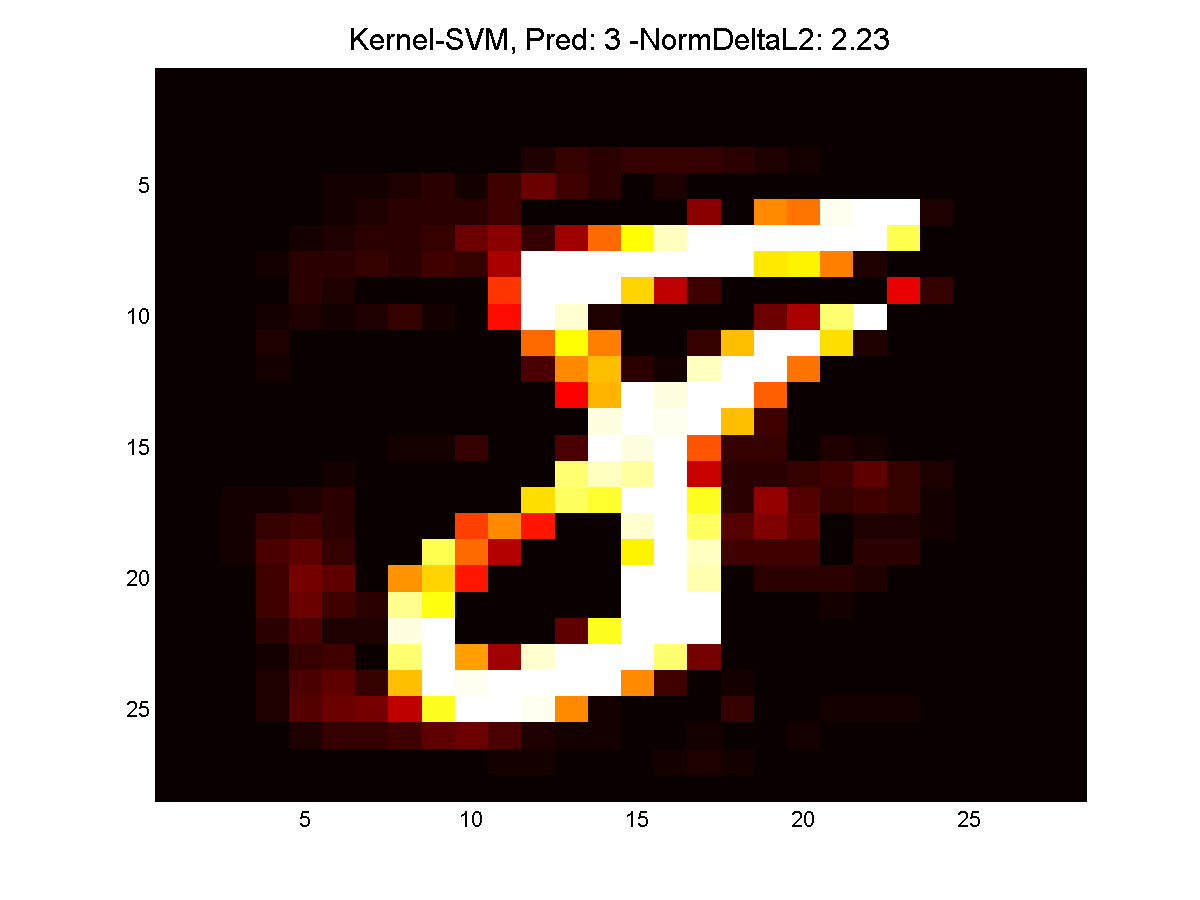}&\includegraphics[width=0.23\textwidth]{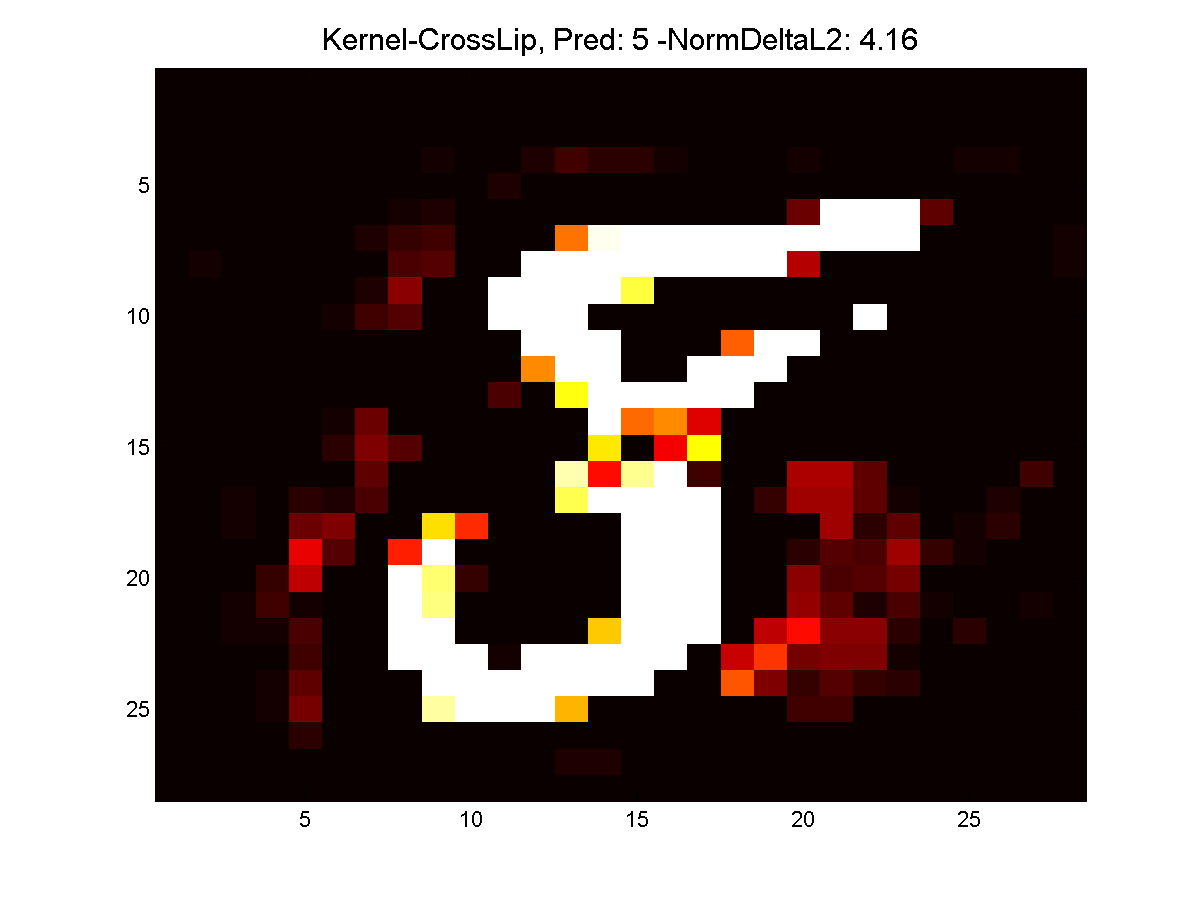}\\
Original, 	Class 8 & K-SVM, Pred:3, $\norm{\delta}_2=2.2$ & K-CL, Pred:5, $\norm{\delta}_2=4.2$\\
\includegraphics[width=0.23\textwidth]{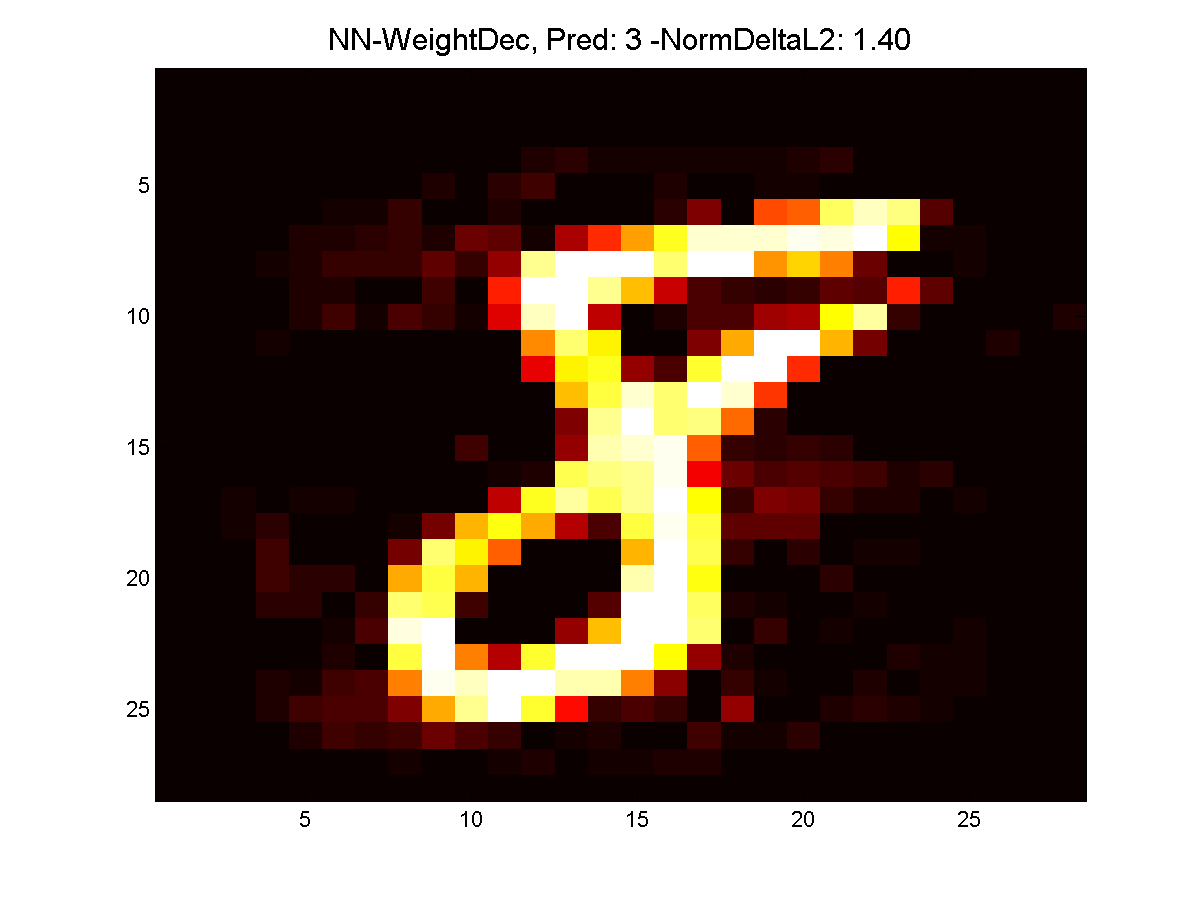}&\includegraphics[width=0.23\textwidth]{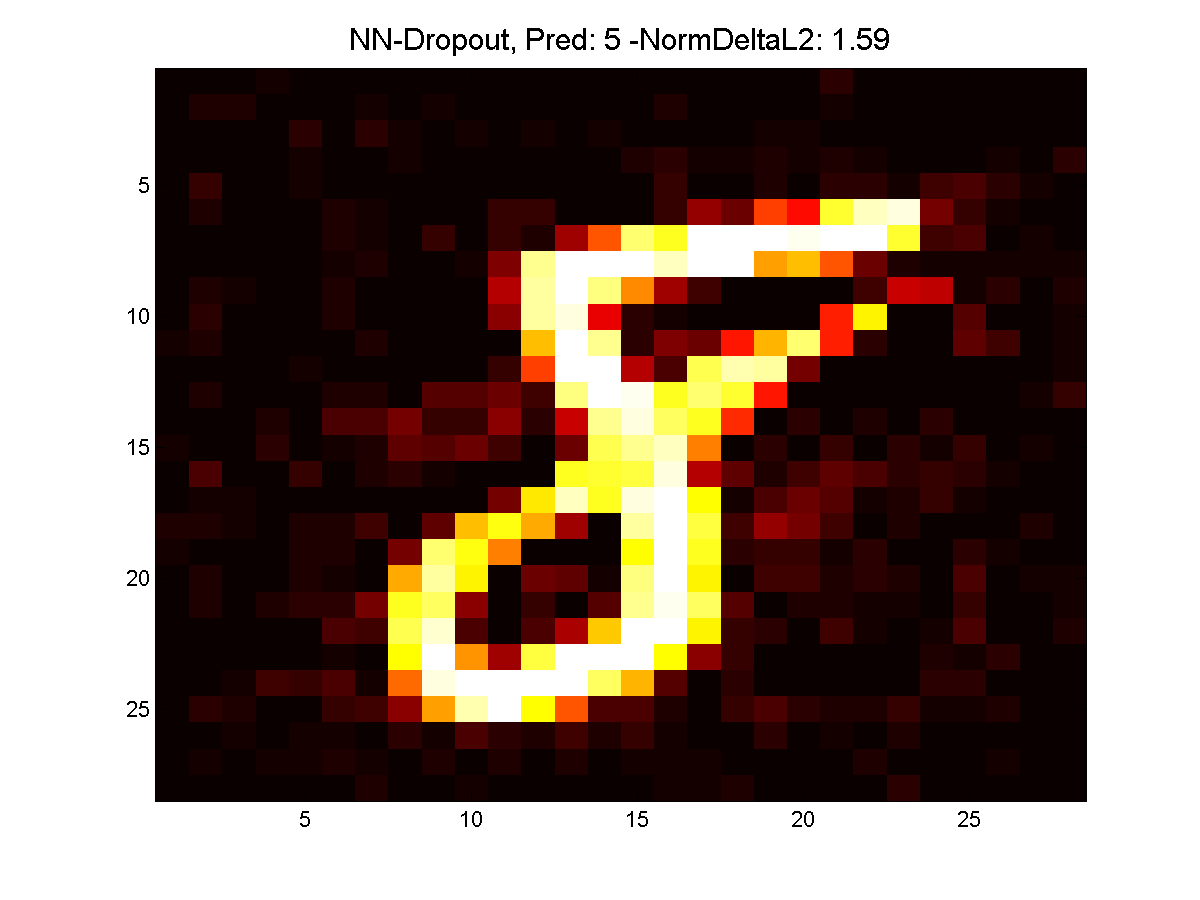}&\includegraphics[width=0.23\textwidth]{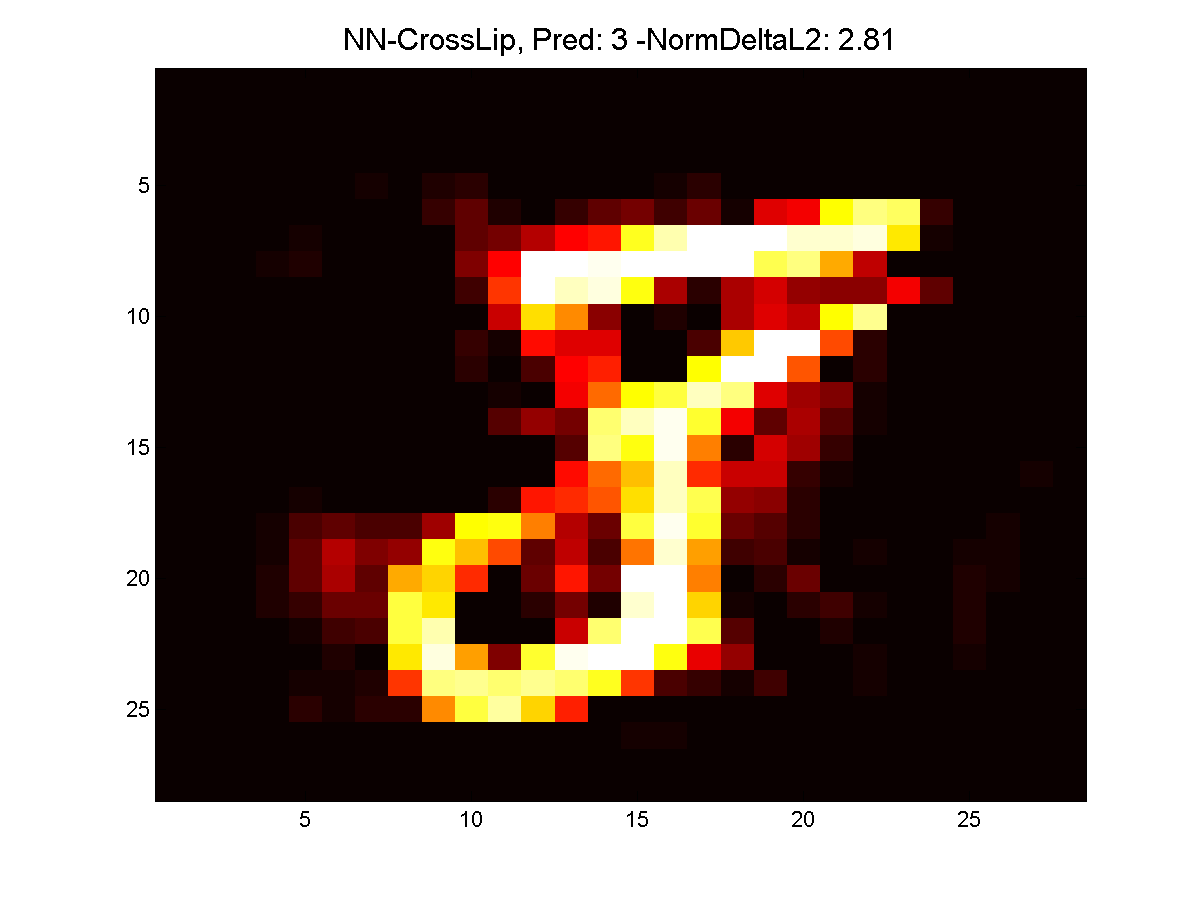}\\
NN-WD, Pred:3, $\norm{\delta}_2=1.4$ & NN-DO, Pred:5, $\norm{\delta}_2=1.6$ & NN-CL, Pred:3, $\norm{\delta}_2=2.8$
\end{tabular}
\captionof{figure}{Top left: original test image, for each classifier we generate the corresponding adversarial sample
which changes the classifier decision (denoted as Pred). Note that for the kernel methods this new decision makes sense, whereas
for all neural network models the change is so small that the new decision is clearly wrong.}
\end{center}
\begin{center}
\begin{tabular}{ccc} \includegraphics[width=0.23\textwidth]{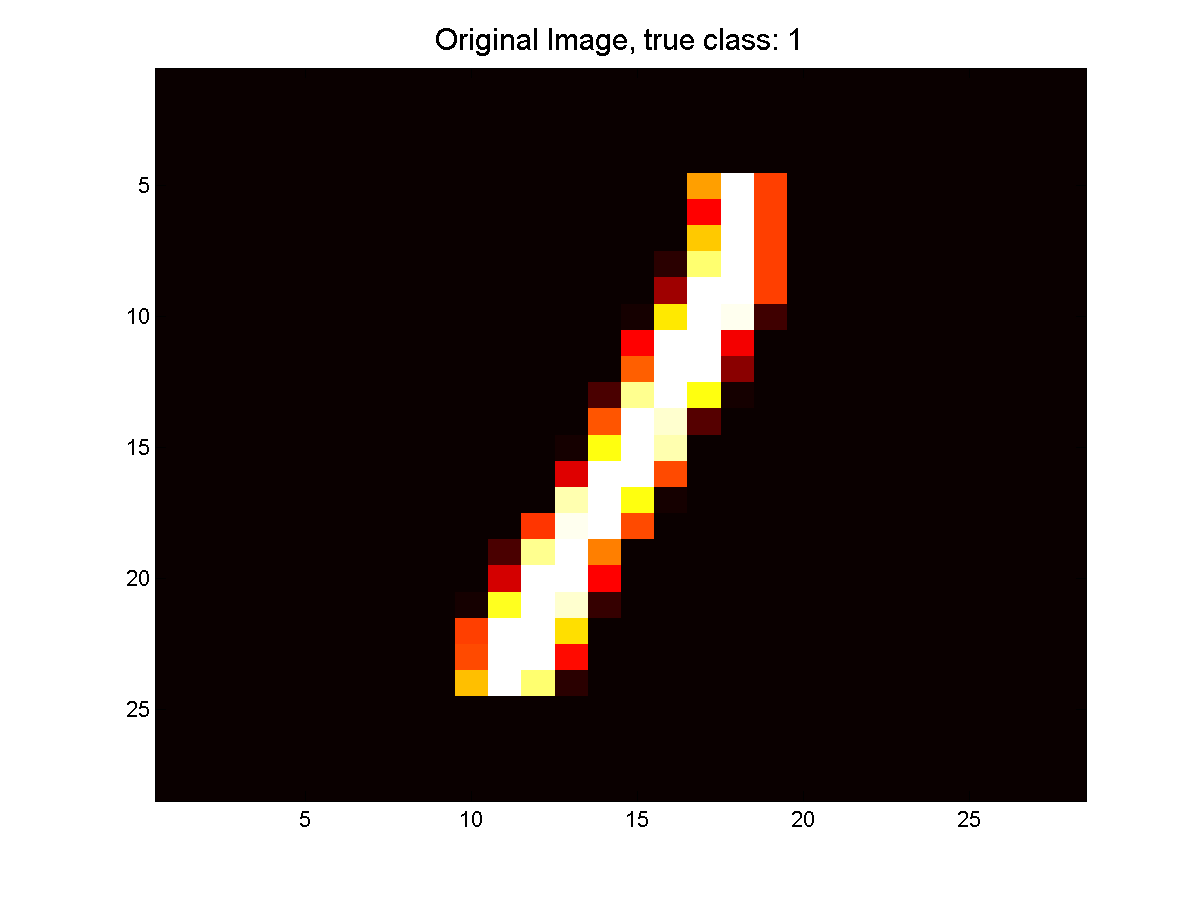}&\includegraphics[width=0.23\textwidth]{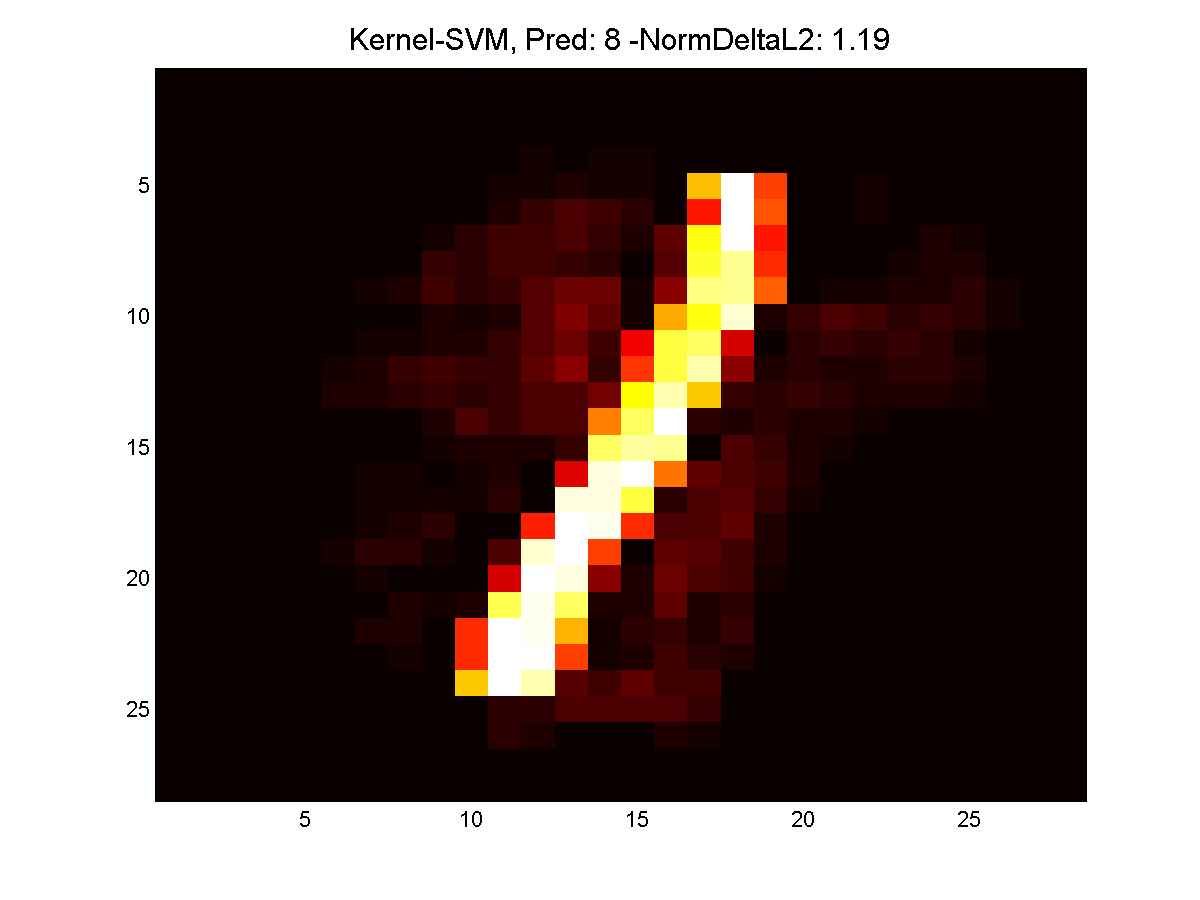}&\includegraphics[width=0.23\textwidth]{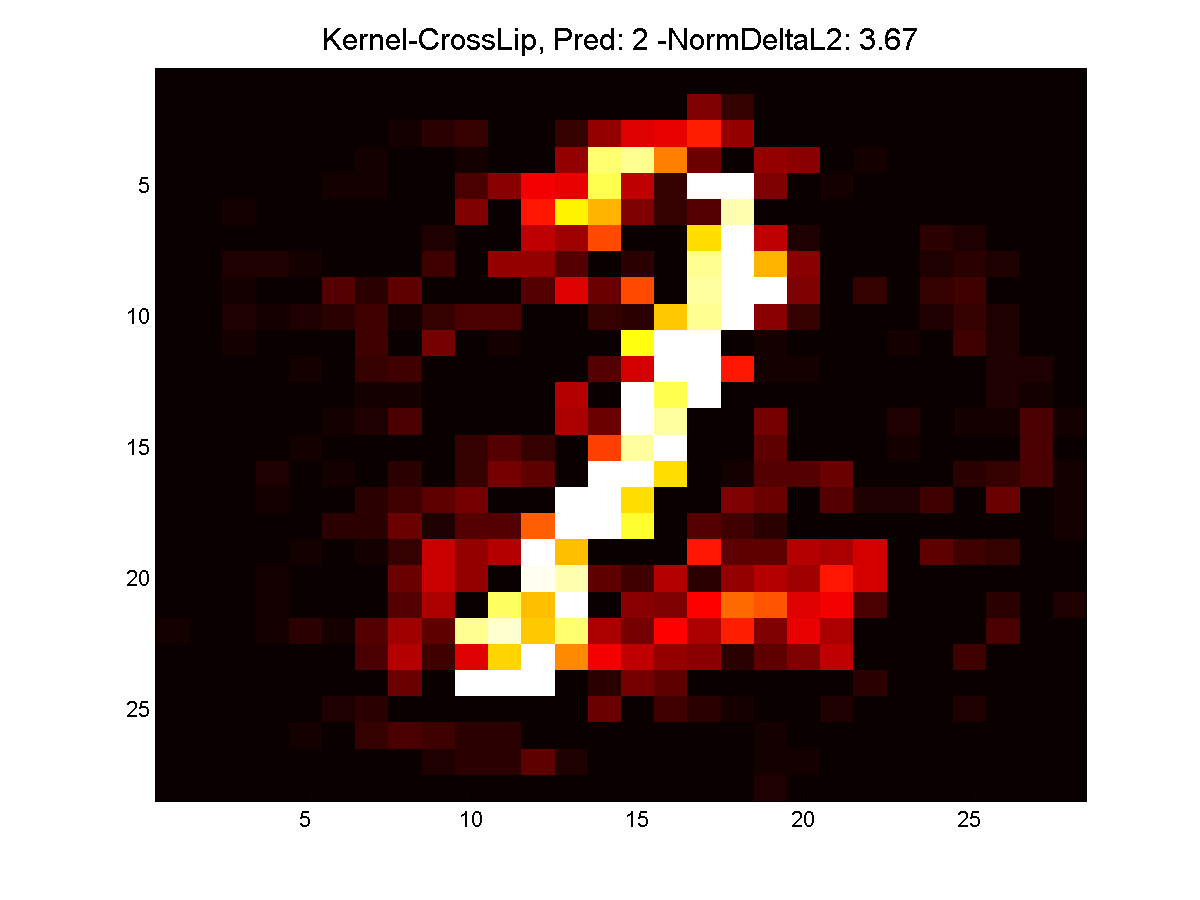}\\
Original, 	Class 1 & K-SVM, Pred:8, $\norm{\delta}_2=1.2$ & K-CL, Pred:2, $\norm{\delta}_2=3.7$\\
\includegraphics[width=0.23\textwidth]{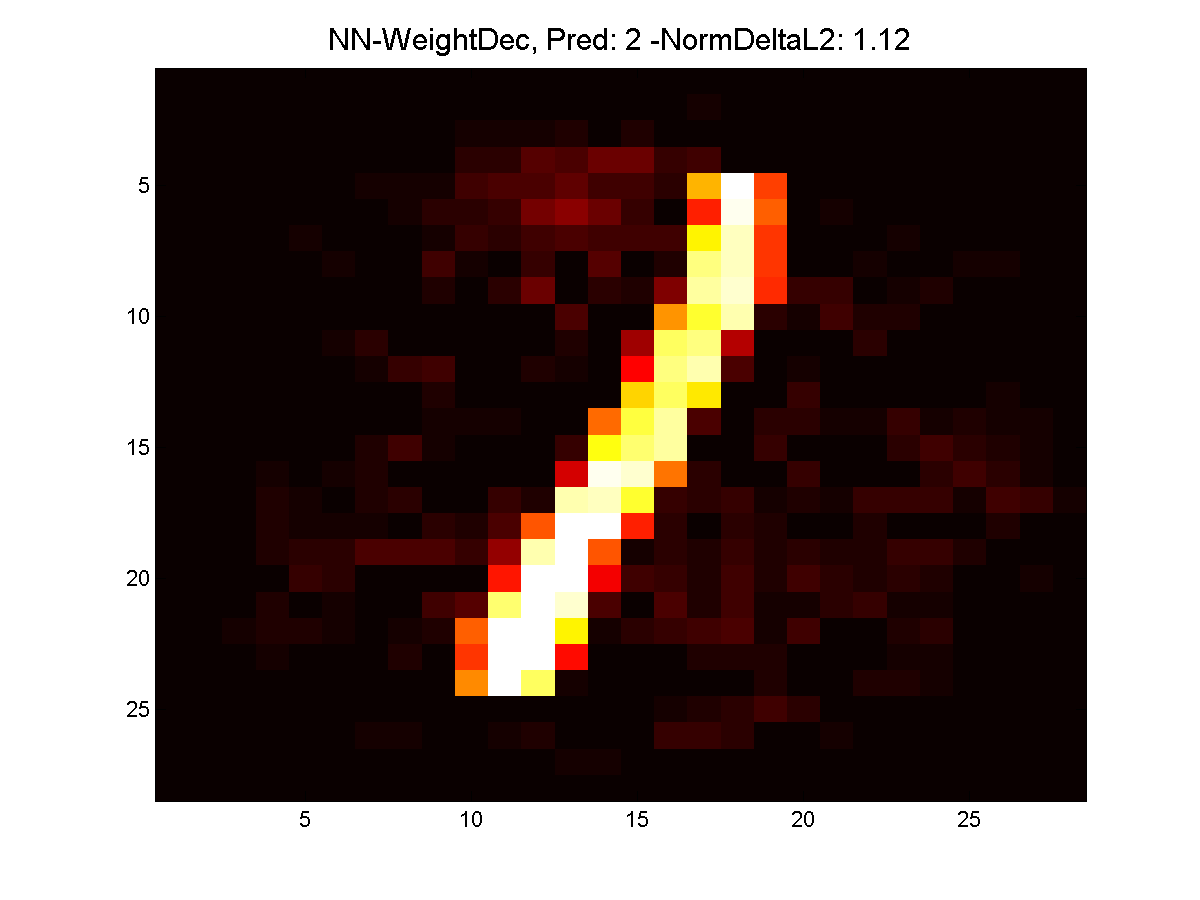}&\includegraphics[width=0.23\textwidth]{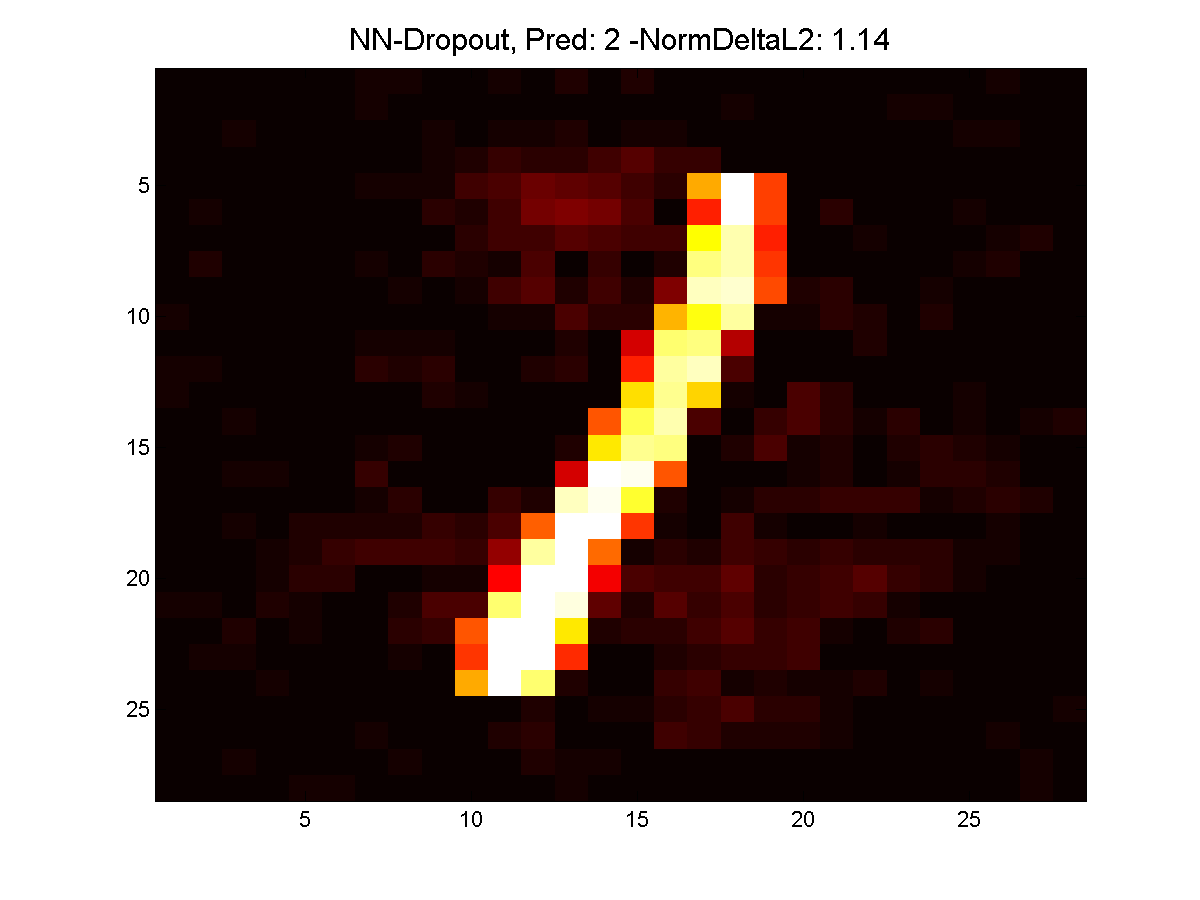}&\includegraphics[width=0.23\textwidth]{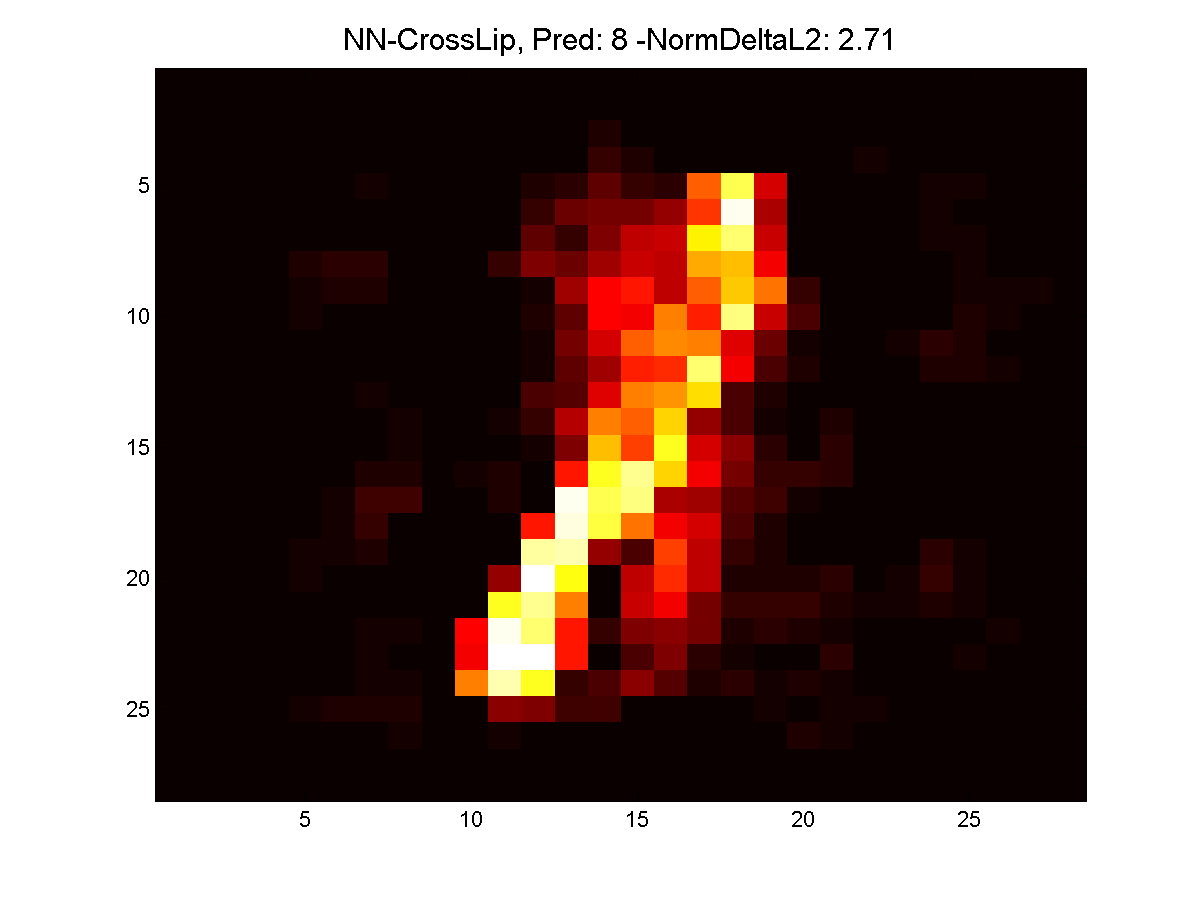}\\
NN-WD, Pred:2, $\norm{\delta}_2=1.1$ & NN-DO, Pred:2, $\norm{\delta}_2=1.1$ & NN-CL, Pred:8, $\norm{\delta}_2=2.7$
\end{tabular}
\captionof{figure}{Top left: original test image, for each classifier we generate the corresponding adversarial sample
which changes the classifier decision (denoted as Pred). Note that for the kernel methods this new decision makes sense, whereas
for all neural network models the change is so small that the new decision is clearly wrong.}
\end{center}
\begin{center}
\begin{tabular}{ccc} \includegraphics[width=0.23\textwidth]{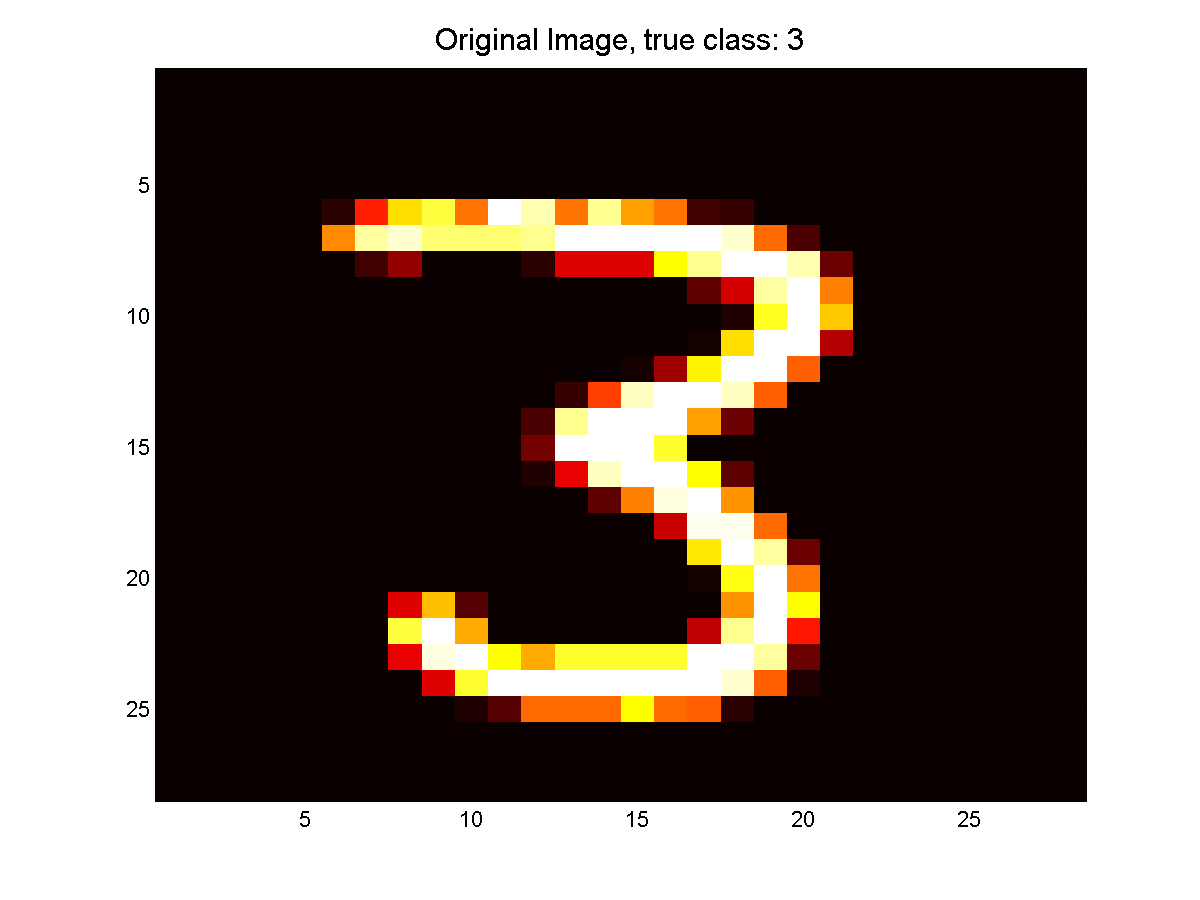}&\includegraphics[width=0.23\textwidth]{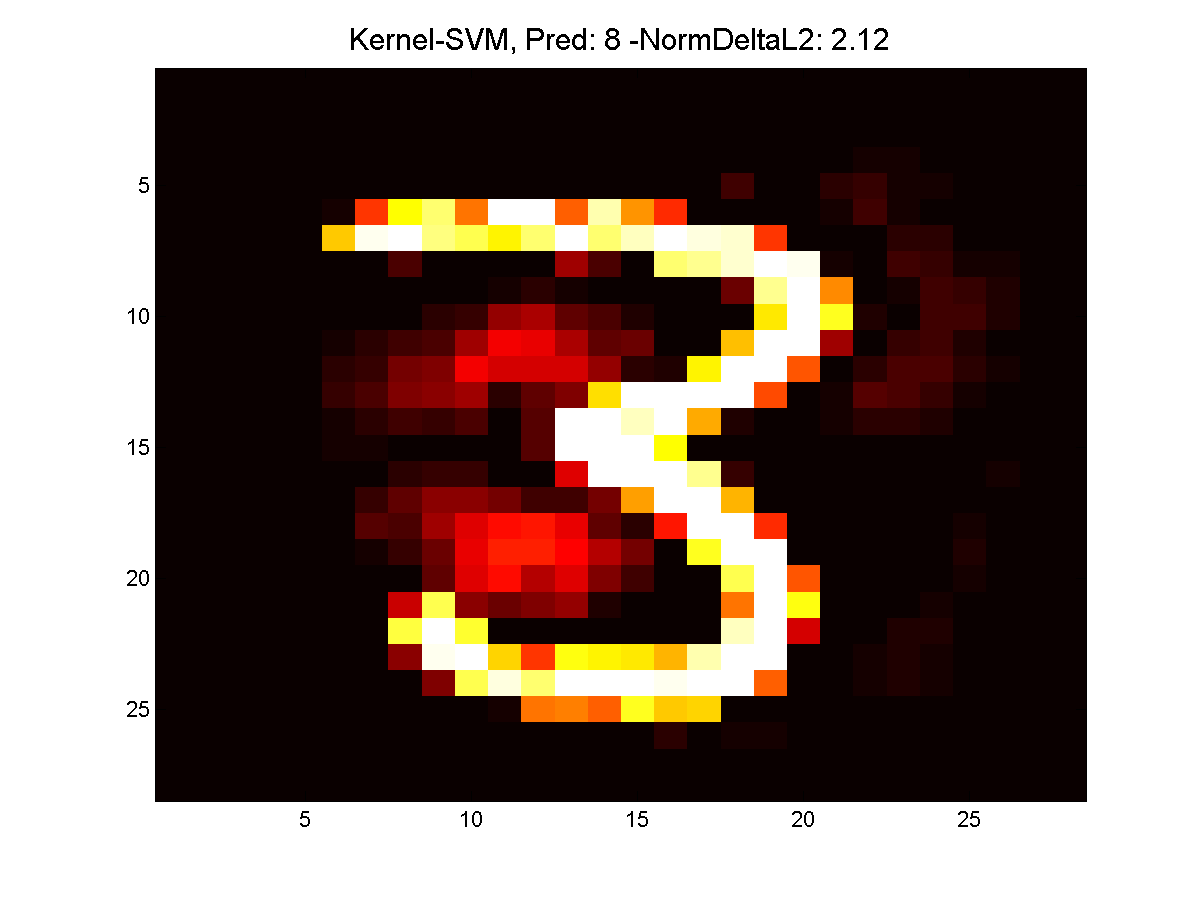}&\includegraphics[width=0.23\textwidth]{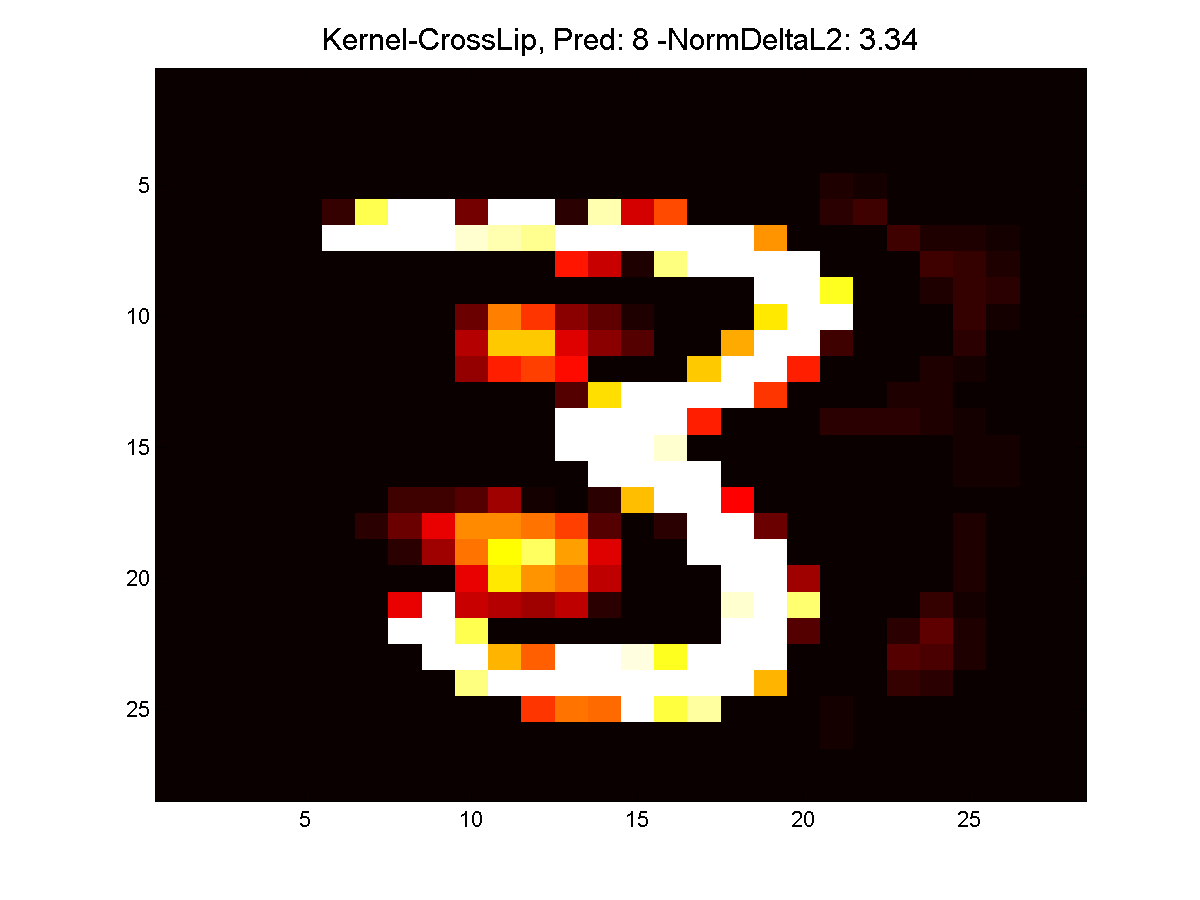}\\
Original, 	Class 3 & K-SVM, Pred:8, $\norm{\delta}_2=2.1$ & K-CL, Pred:8, $\norm{\delta}_2=3.3$\\
\includegraphics[width=0.23\textwidth]{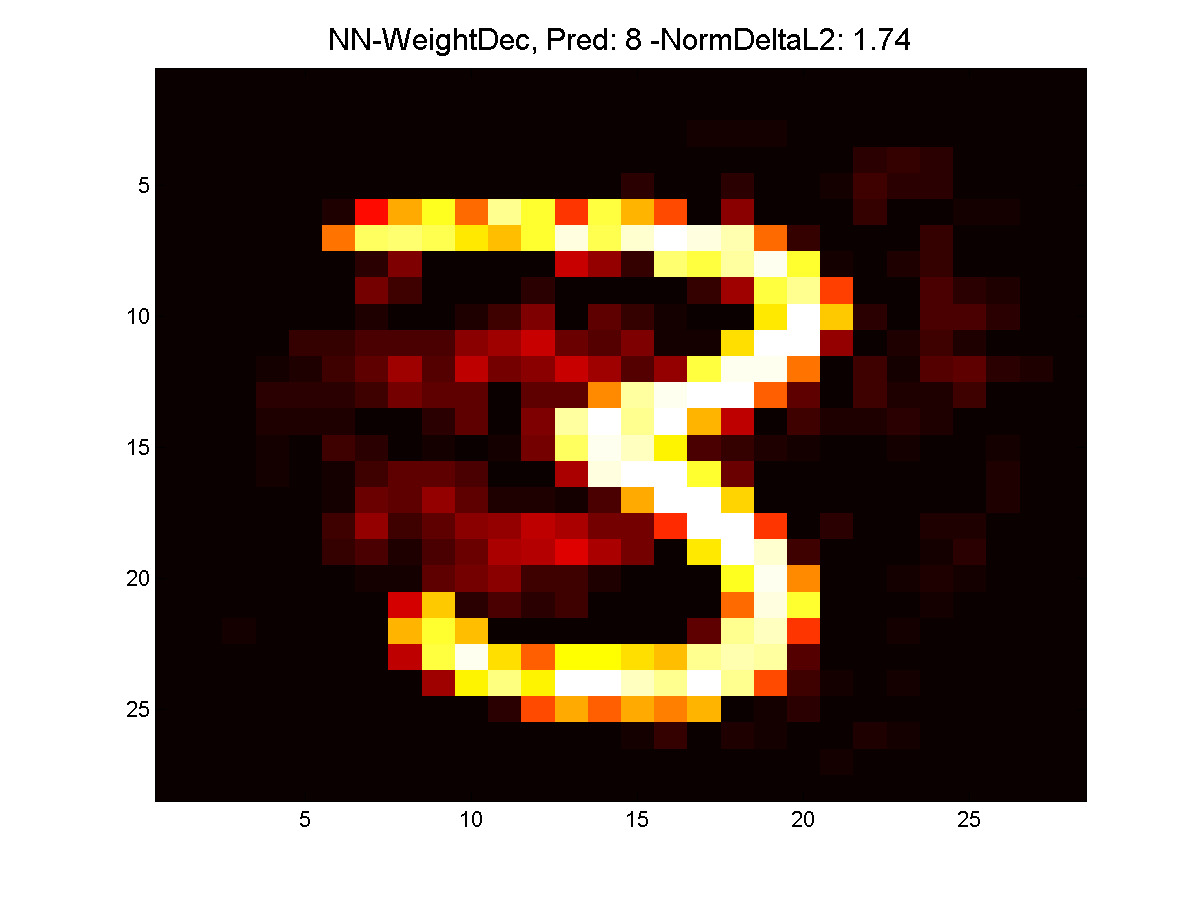}&\includegraphics[width=0.23\textwidth]{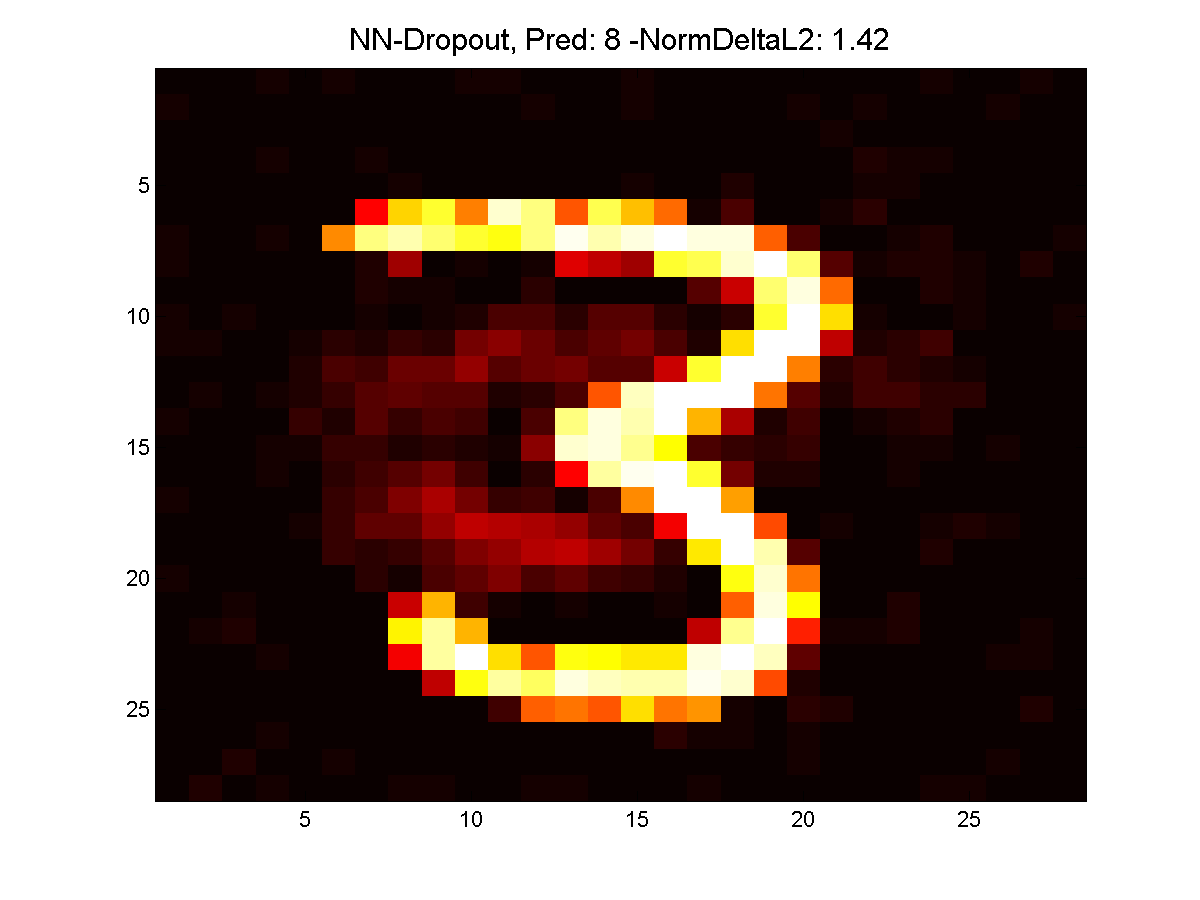}&\includegraphics[width=0.23\textwidth]{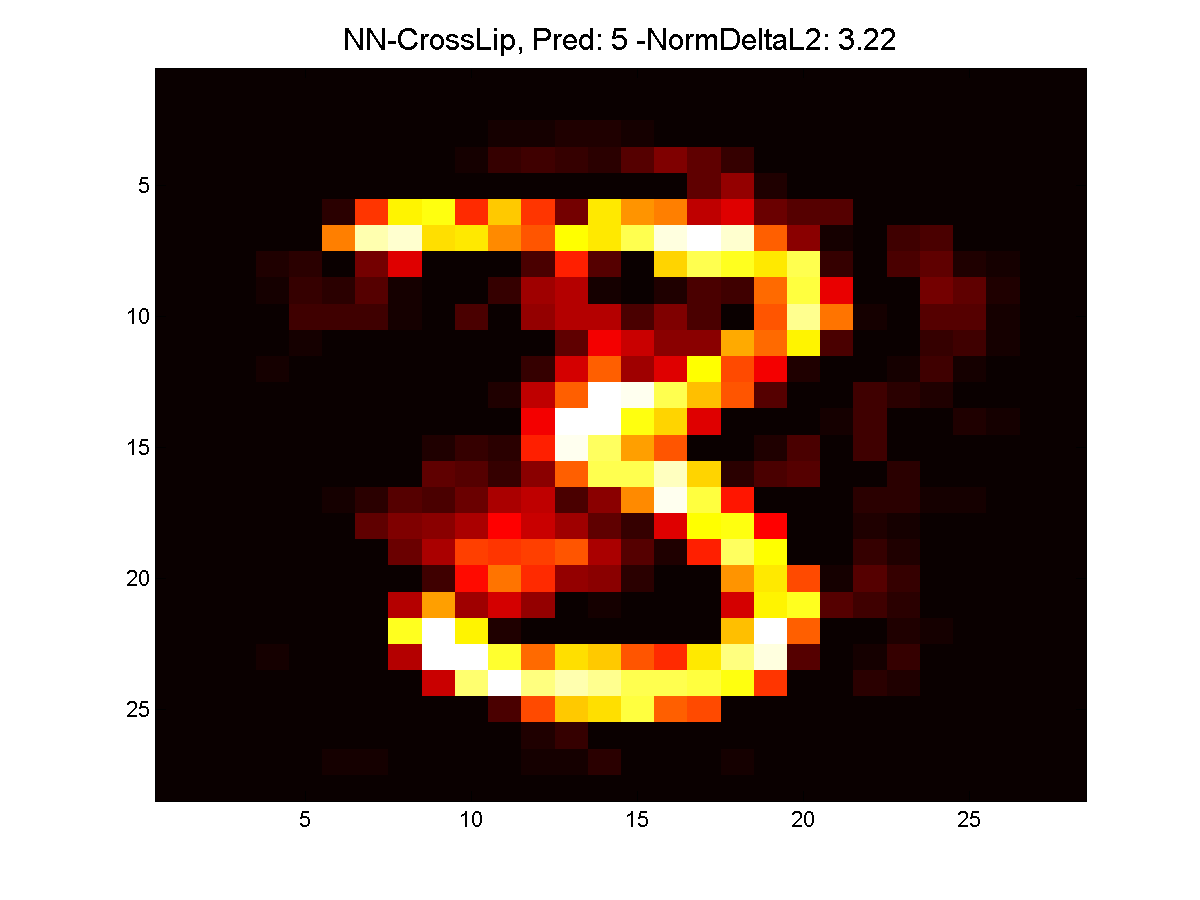}\\
NN-WD, Pred:8, $\norm{\delta}_2=1.7$ & NN-DO, Pred:8, $\norm{\delta}_2=1.4$ & NN-CL, Pred:5, $\norm{\delta}_2=3.2$
\end{tabular}
\captionof{figure}{Top left: original test image, for each classifier we generate the corresponding adversarial sample
which changes the classifier decision (denoted as Pred). Note that for the kernel methods this new decision makes sense, whereas
for all neural network models the change is so small that the new decision is clearly wrong.}
\end{center}
\begin{center}
\begin{tabular}{ccc} \includegraphics[width=0.23\textwidth]{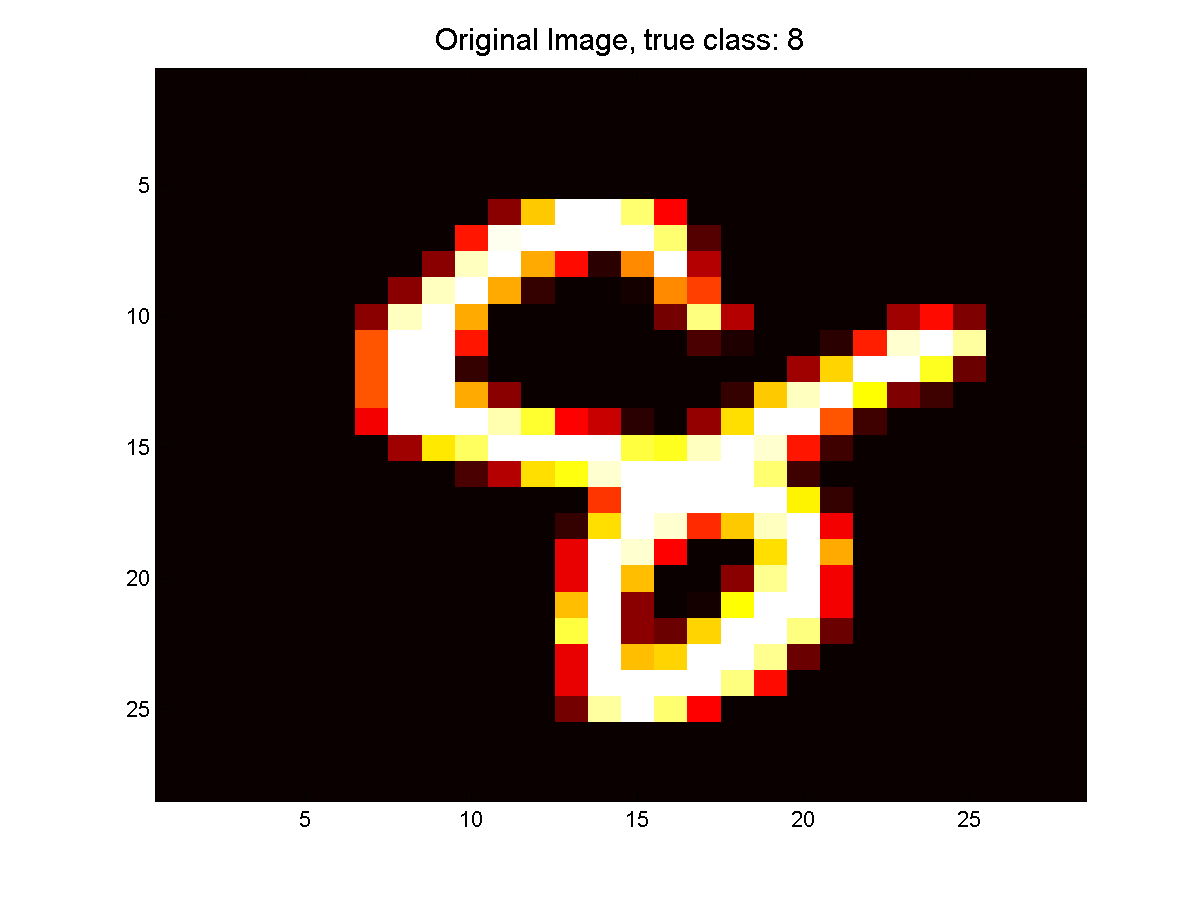}&\includegraphics[width=0.23\textwidth]{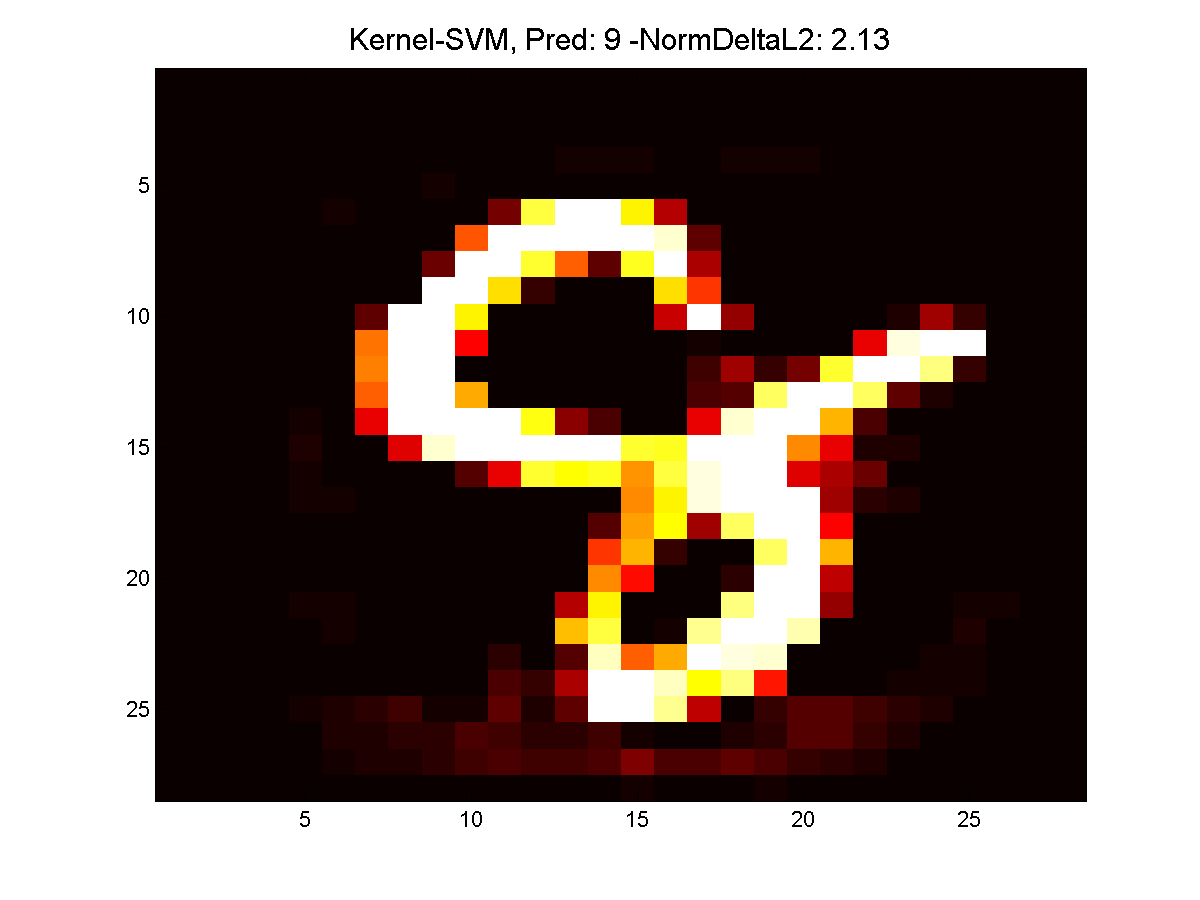}&\includegraphics[width=0.23\textwidth]{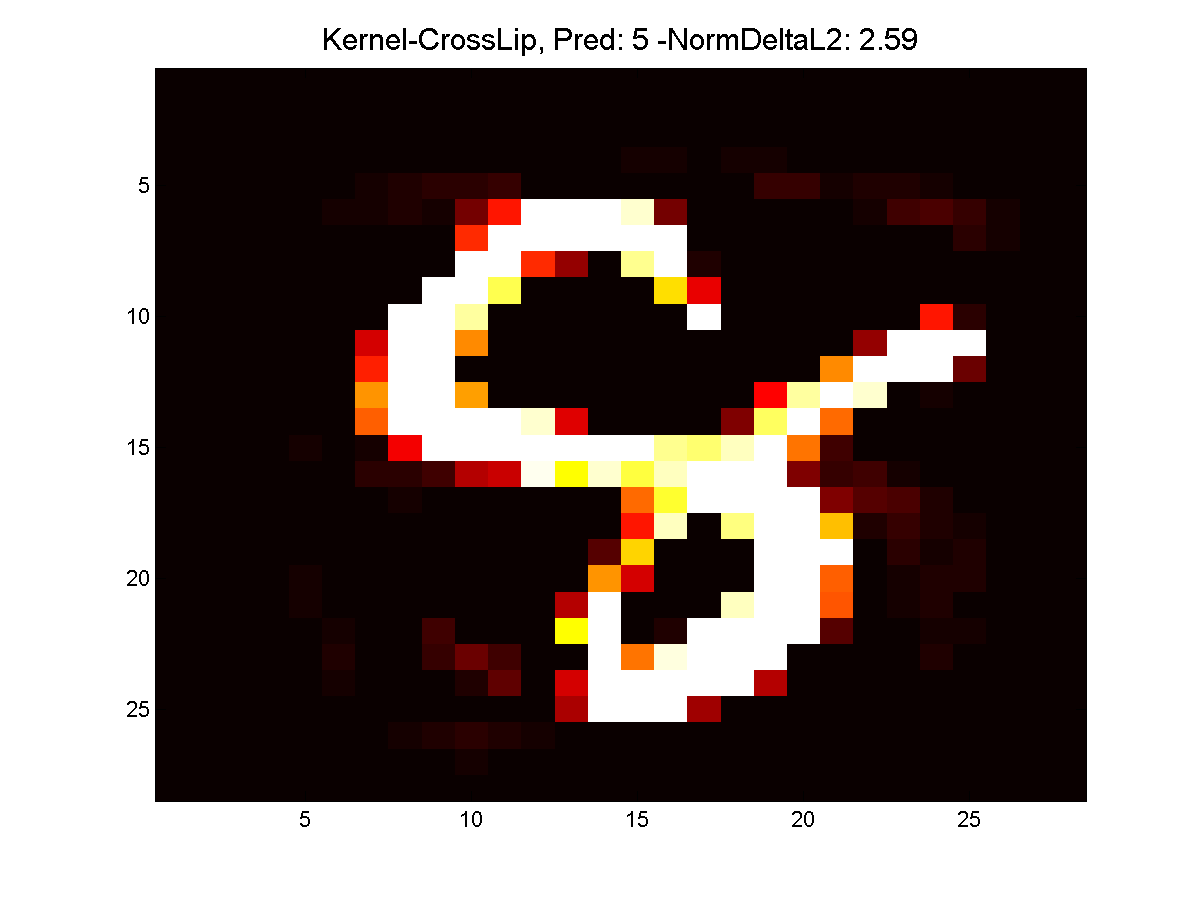}\\
Original, 	Class 8 & K-SVM, Pred:9, $\norm{\delta}_2=2.1$ & K-CL, Pred:5, $\norm{\delta}_2=2.6$\\
\includegraphics[width=0.23\textwidth]{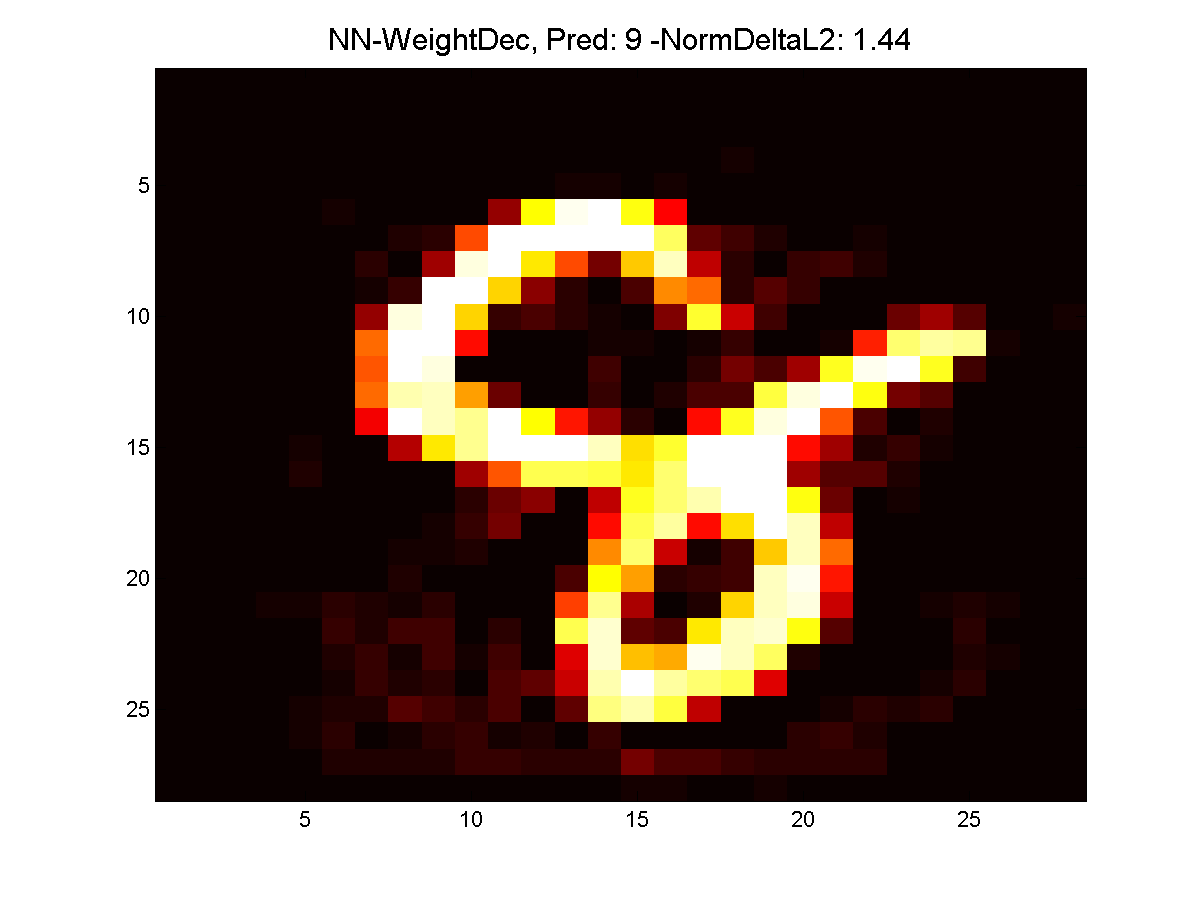}&\includegraphics[width=0.23\textwidth]{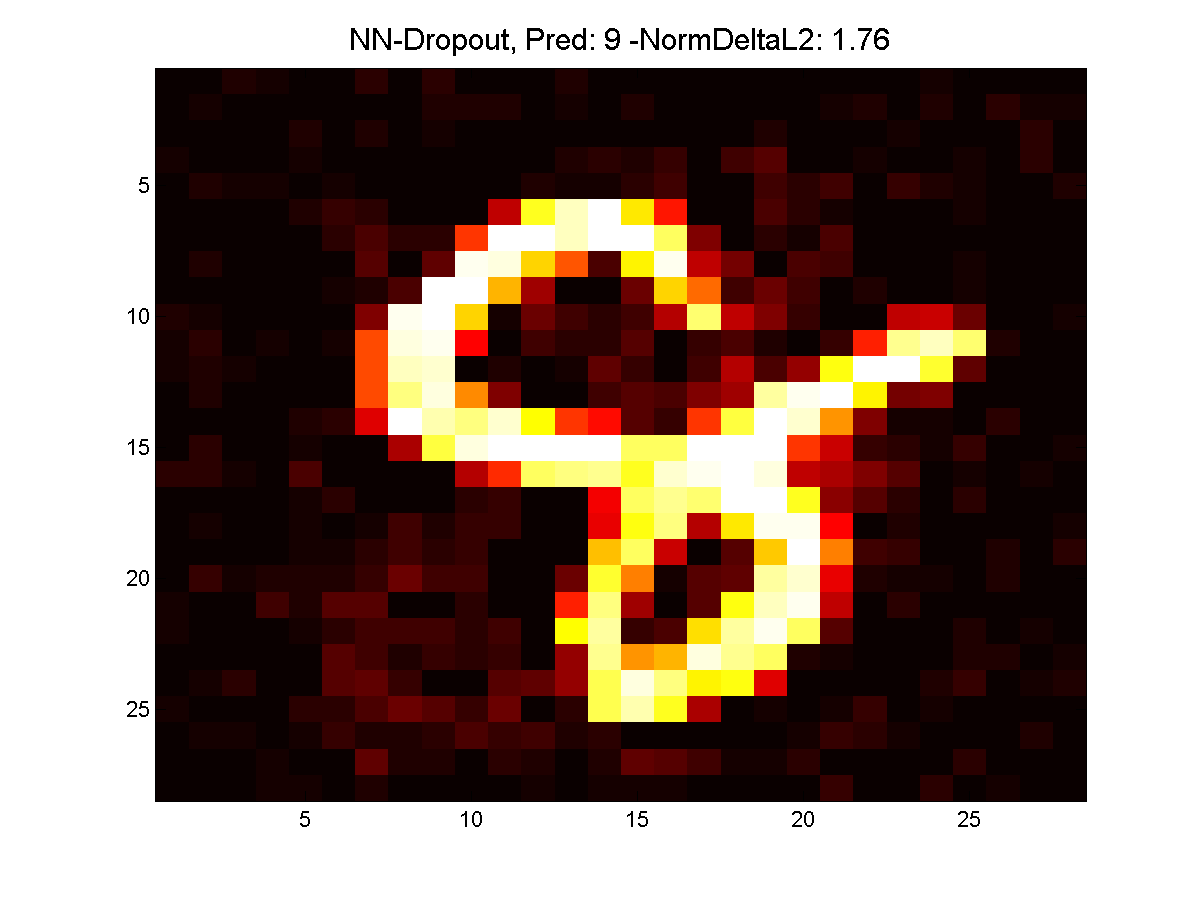}&\includegraphics[width=0.23\textwidth]{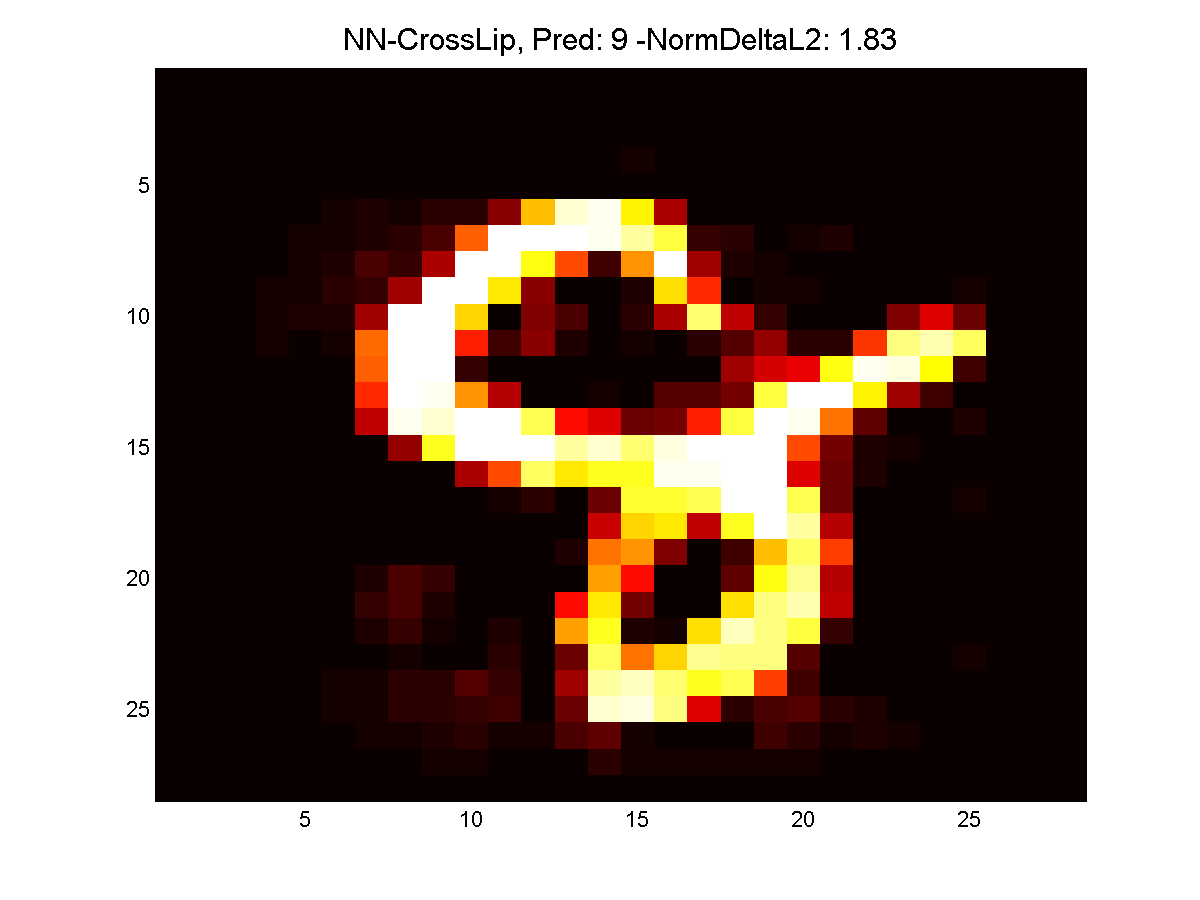}\\
NN-WD, Pred:9, $\norm{\delta}_2=1.4$ & NN-DO, Pred:9, $\norm{\delta}_2=1.8$ & NN-CL, Pred:9, $\norm{\delta}_2=1.8$
\end{tabular}
\captionof{figure}{Top left: original test image, for each classifier we generate the corresponding adversarial sample
which changes the classifier decision (denoted as Pred). Note that for the kernel methods this new decision makes sense, whereas
for all neural network models the change is so small that the new decision is clearly wrong.}
\end{center}
\begin{center}
\begin{tabular}{ccc} \includegraphics[width=0.23\textwidth]{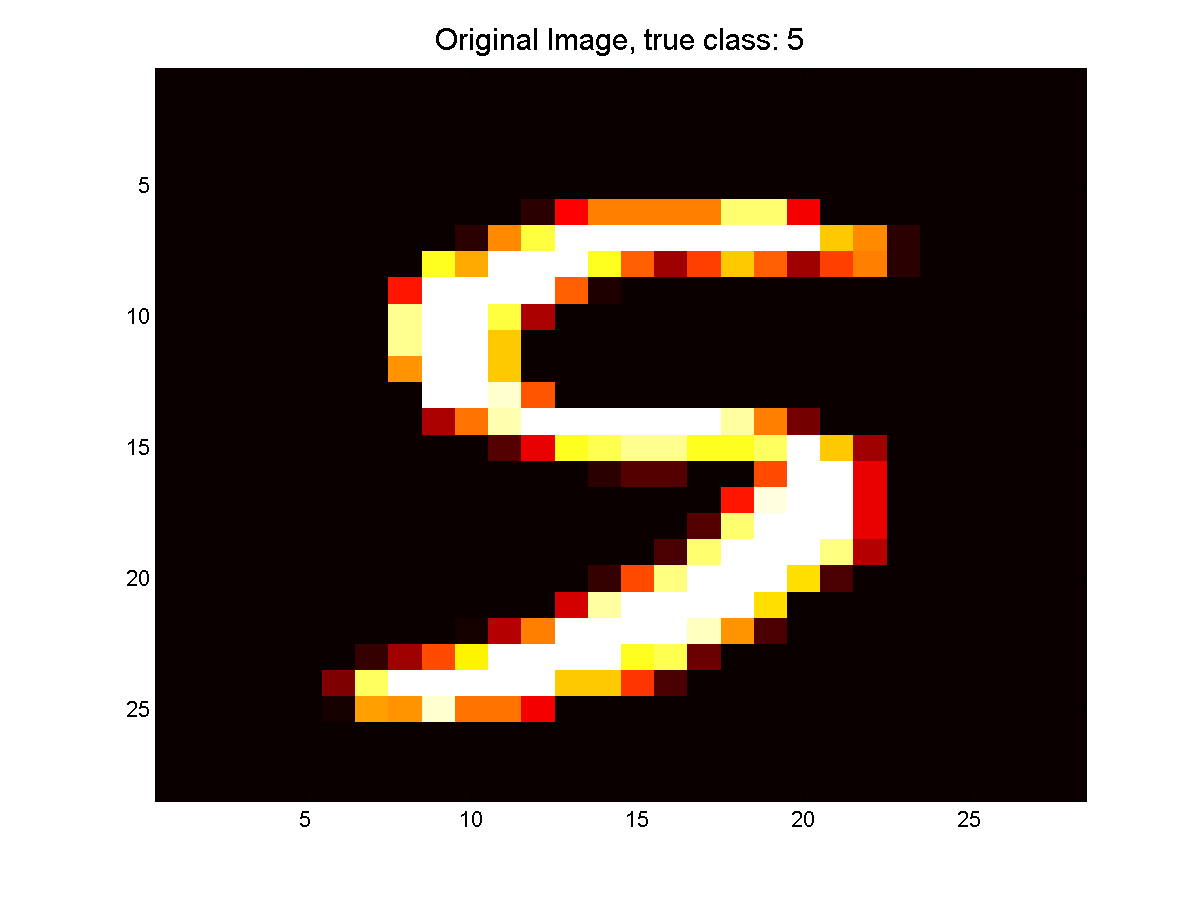}&\includegraphics[width=0.23\textwidth]{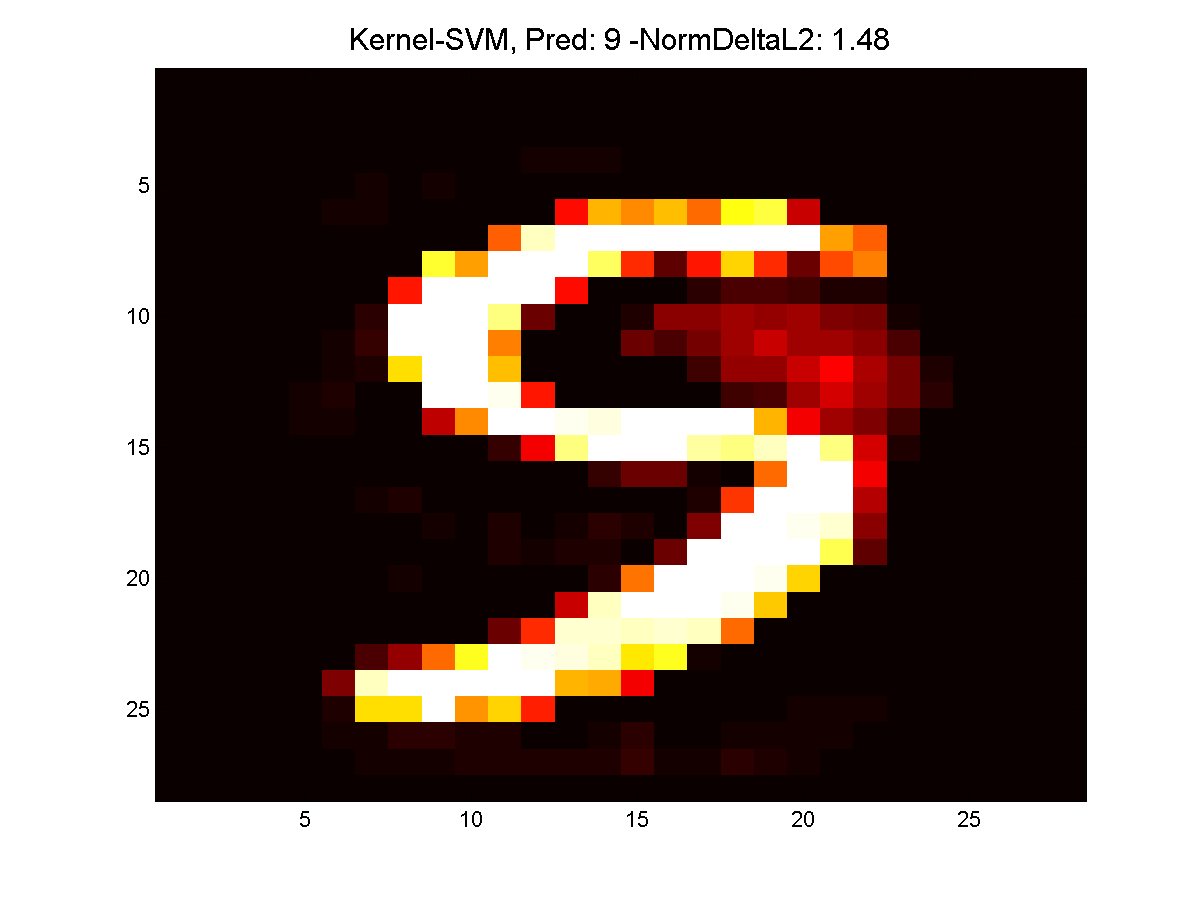}&\includegraphics[width=0.23\textwidth]{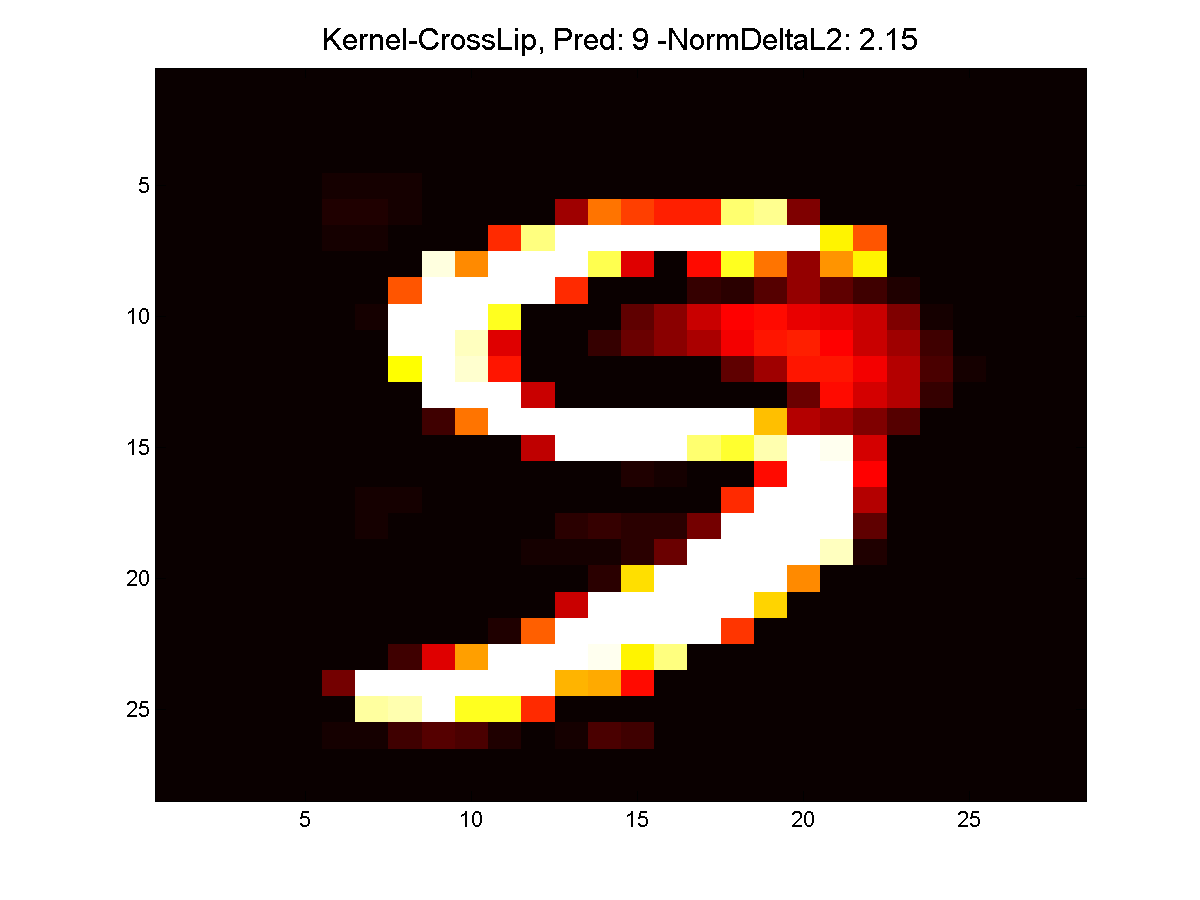}\\
Original, 	Class 5 & K-SVM, Pred:9, $\norm{\delta}_2=1.5$ & K-CL, Pred:9, $\norm{\delta}_2=2.2$\\
\includegraphics[width=0.23\textwidth]{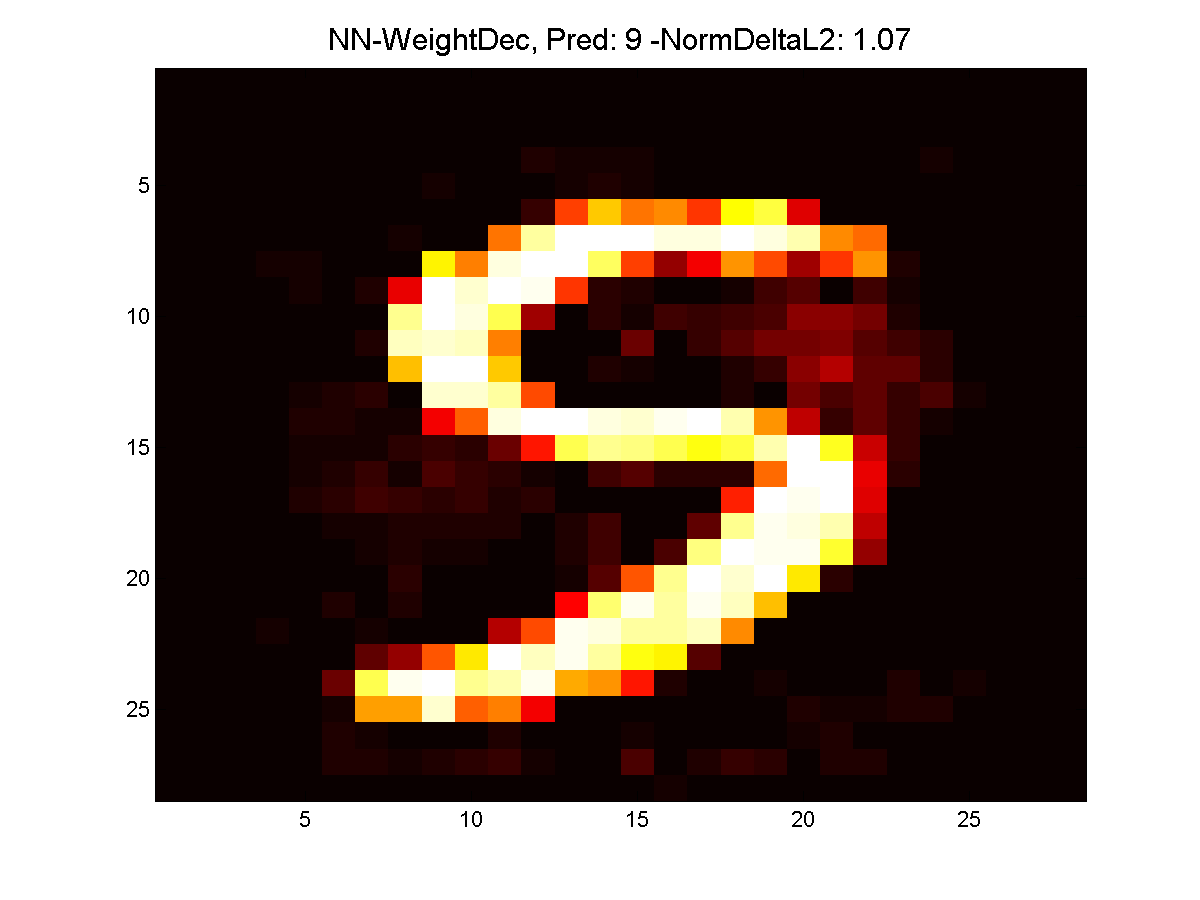}&\includegraphics[width=0.23\textwidth]{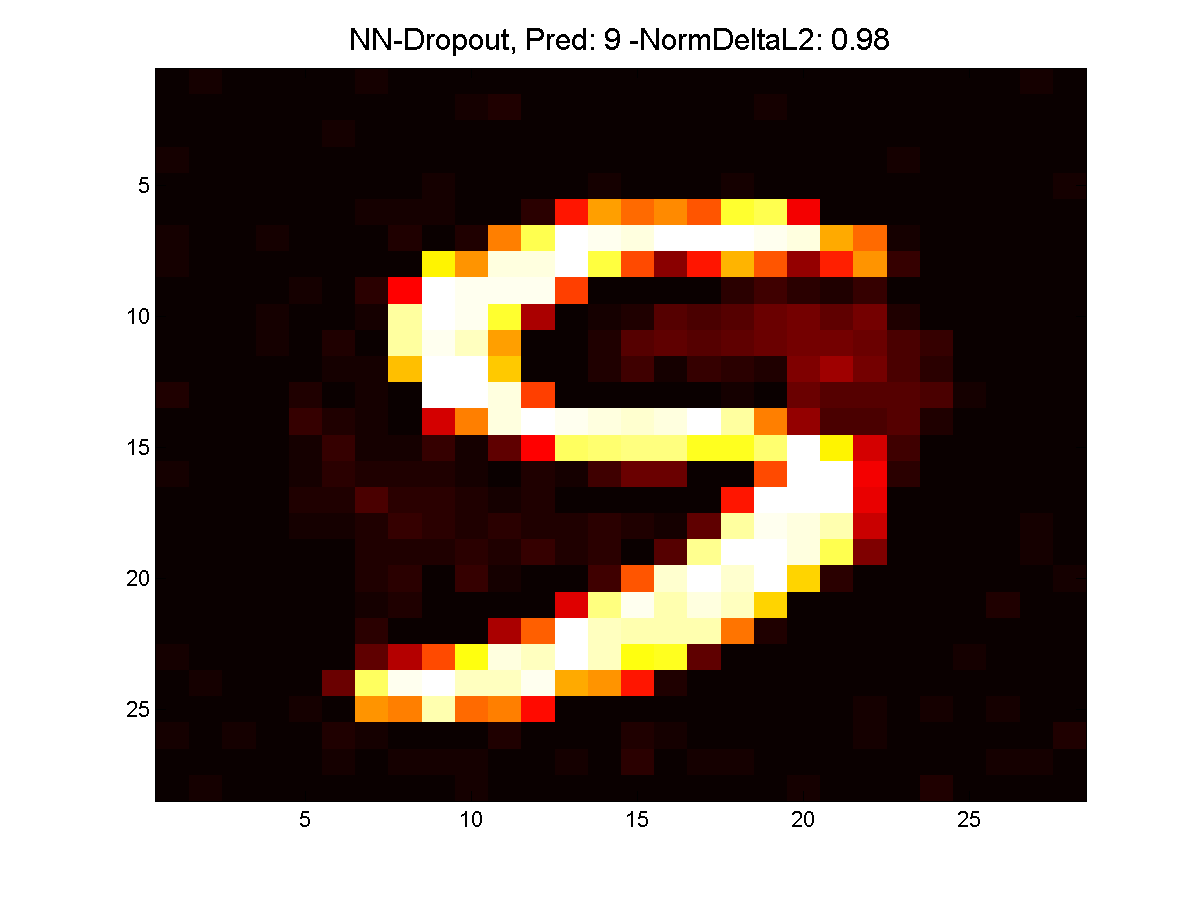}&\includegraphics[width=0.23\textwidth]{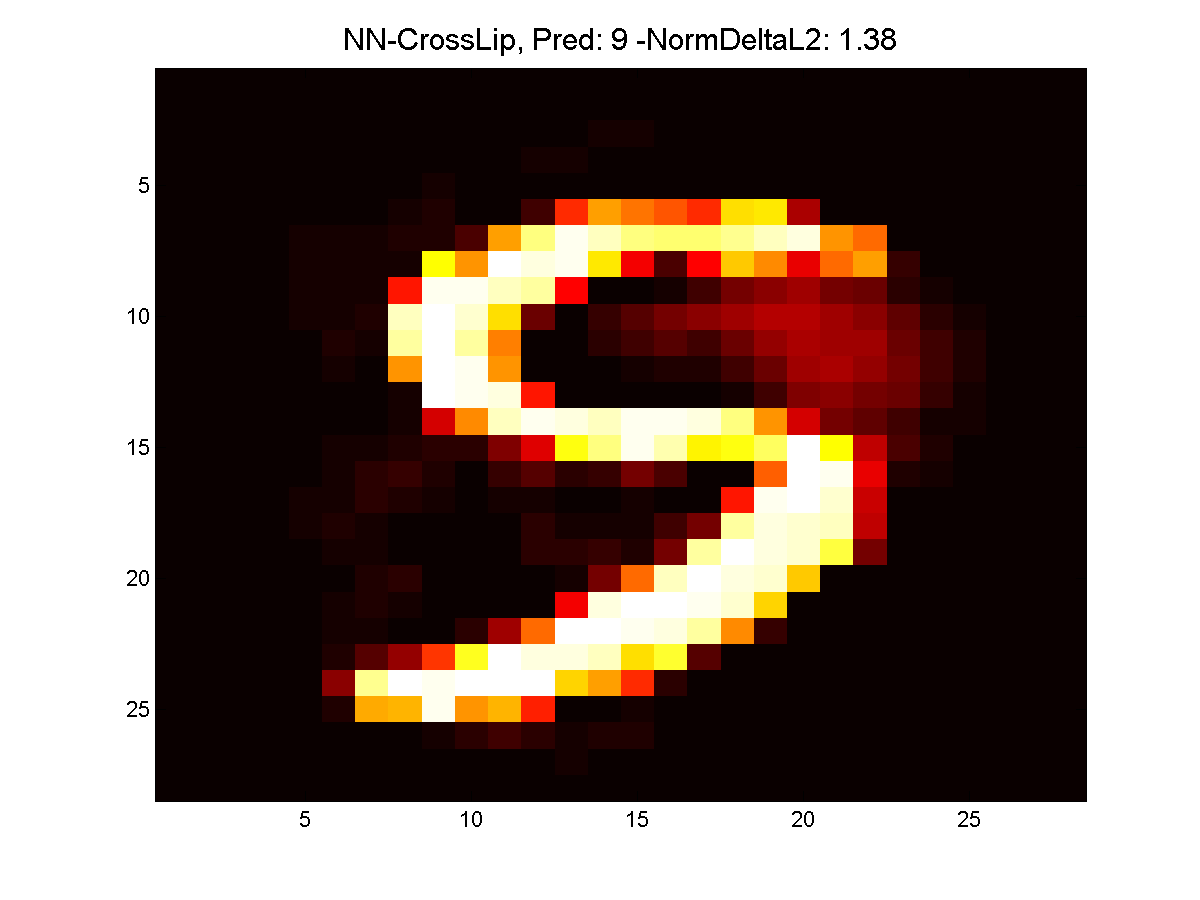}\\
NN-WD, Pred:9, $\norm{\delta}_2=1.1$ & NN-DO, Pred:9, $\norm{\delta}_2=1.0$ & NN-CL, Pred:9, $\norm{\delta}_2=1.4$
\end{tabular}
\captionof{figure}{Top left: original test image, for each classifier we generate the corresponding adversarial sample
which changes the classifier decision (denoted as Pred). Note that for the kernel methods this new decision makes sense, whereas
for all neural network models the change is so small that the new decision is clearly wrong.}
\end{center}
\fi

\ifpaper
\paragraph{German Traffic Sign Benchmark:}  As a third dataset we used the German Traffic Sign Benchmark (GTSB) \cite{GTSB2012}, which consists of images of german traffic signs, which has 43 classes with 34209 training and 12630 test samples.
The results are shown in Figure \ref{exp:NN-GTSB}. For this dataset Cross-Lipschitz regularization improves the upper bounds compared to weight decay but dropout achieves significantly better prediction performance and has similar upper bounds.
The robustness guarantees for weight decay and Cross-Lipschitz are slightly better than for dropout. 
\begin{figure}
{\scriptsize 
\begin{tabular}{c|c}
 Adversarial Resistance (Upper Bound)  & Robustness Guarantee (Lower Bound)\\
  wrt to $L_2$-norm                             & wrt to $L_2$-norm\\
  \includegraphics[width=0.45\textwidth]{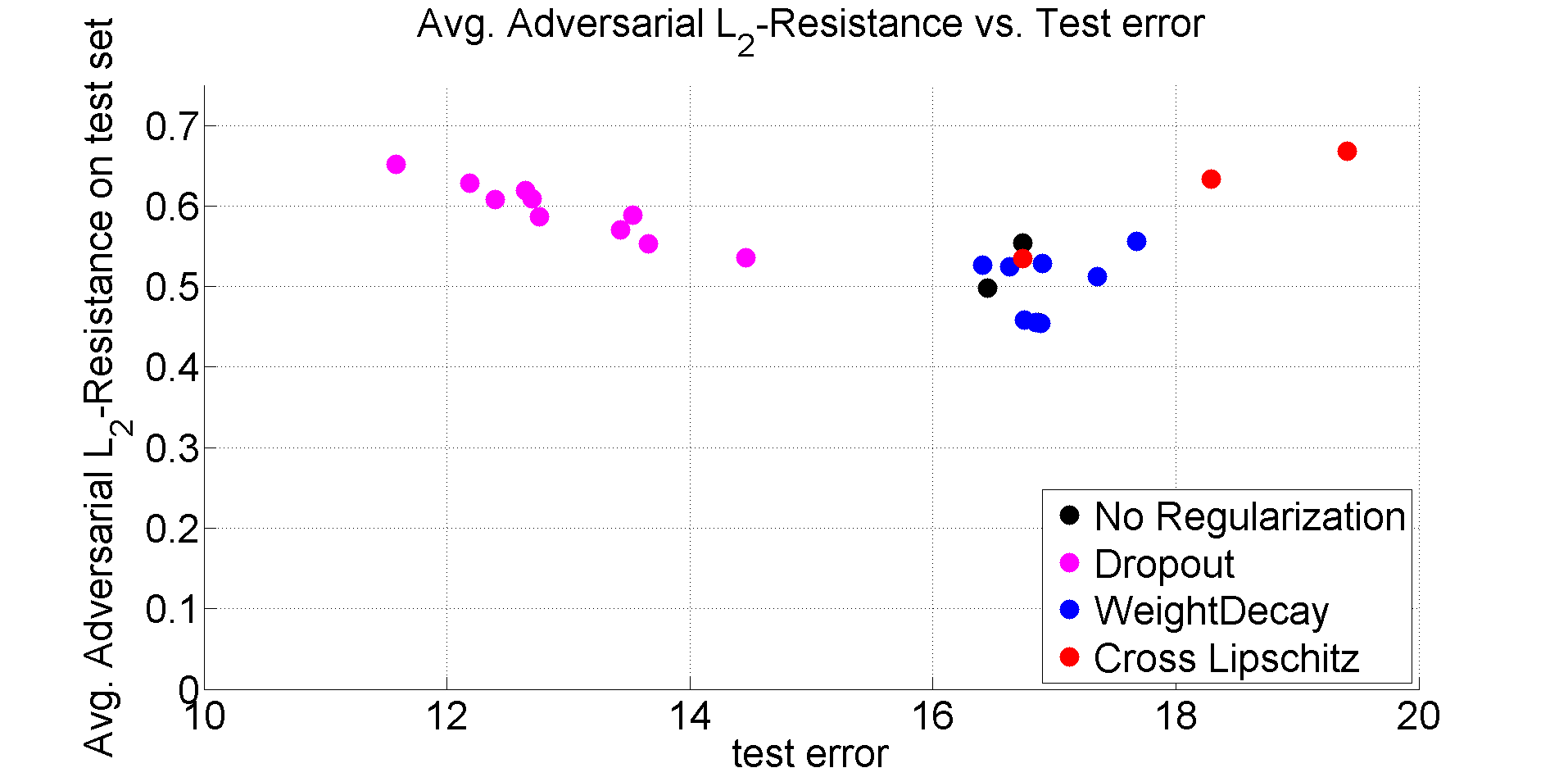}     &\includegraphics[width=0.45\textwidth]{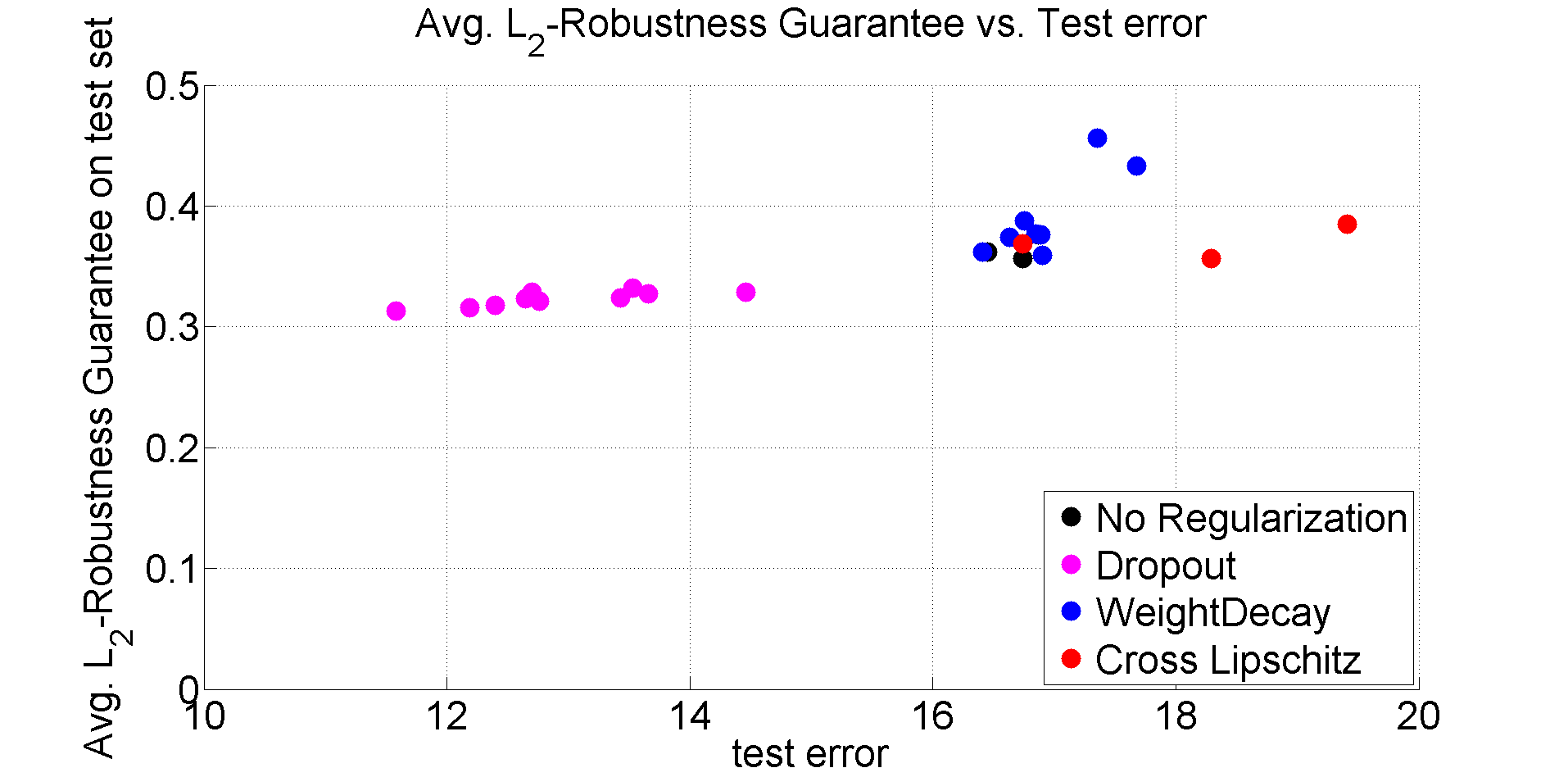}
\end{tabular}}
\caption{\label{exp:NN-GTSB}  \textbf{Left:} Adversarial resistance wrt to $L_2$-norm on test set of the german traffic sign benchmark (GTSB) in the plain setting. \textbf{Right:} Average robustness guarantee on the test set wrt to $L_2$-norm
for the test set of GTSB for different neural networks (one hidden layer, 1024 HU) and hyperparameters. Here dropout performs very well both in terms of performance and robustness.}
\end{figure}
\fi

\ifpaper
\paragraph{Residual Networks:} All experiments so far were done with one hidden layer neural networks so that we can evaluate lower and upper bounds. Now we want to demonstrate that Cross-Lipschitz regularization can also successfully be used for deep networks.
We use residual networks proposed in \cite{HeZhaRen2015} with 32 parameter layers and non-bottleneck residual blocks. We follow basically their setting, apart from that we did not subtract the per-pixel mean so that all images are in $[0,1]^d$ and use random crop  but without any padding as in \cite{HeZhaRen2015}. Similar to \cite{HeZhaRen2015}, we train for 160 epochs, and the learning rate is divided by 10 on the 115-th and 140-th epochs.
For the experiments with dropout we followed the recommendation of \cite{ZagKom2016}, inserting a dropout layer between convolutional layers inside each residual block. For Cross-Lipschitz regularization we use automatic differentiation in TensorFlow \cite{AbaAgaBar2016} 
 to calculate the derivative with respect to the input, which slows done the training by a factor of  $10$.

For the plain setting the learning rate for all methods is chosen from $\{0.2,0.5\}$, except for the runs without regularization, for which it is from $\{0.08, 0.1, 0.2, 0.4, 0.6, 0.8\}$. For weight decay the regularization parameter is chosen from $\{10^{-5},10^{-4},10^{-3},10^{-2}\}$, for Cross-Lipschitz from $\{10^{-4},10^{-3},10^{-2},10^{-1}\}$, and for dropout the probabilities are from $\{0.5, 0.6, 0.7, 0.8\}$. For the data augmentation setting the only difference was in the higher learning rates: no regularization - $\{0.2, 0.5, 0.8, 1.0, 1.5, 2.0, 3.0, 4.0\}$, weight decay - $\{0.1, 0.4\}$, Cross-Lipschitz - $\{0.2, 1.0\}$.
The results are shown in Figure \ref{exp:ResNets-plain-augm}. Cross-Lipschitz regularization improves the upper bounds on the robustness against adversarial manipulation compare to weight decay and dropout by a factor of 2 to 3 both in the plain setting (right) and with data augmentation (left). This comes at a price
of a slightly worse test performance. However, it shows that Cross-Lipschitz regularization is also effective for deep neural networks. It remains interesting future work to come up also with interesting instance-specific lower bounds (robustness guarantees) for deep neural networks.

\begin{figure}
\centering
\begin{tabular}{c c}
	Adversarial Resistance (Upper Bound)  & Adversarial Resistance (Upper Bound)\\
	wrt to $L_2$-norm (ResNets)                           & wrt to $L_2$-norm (ResNets)\\
	\includegraphics[width=0.45\textwidth]{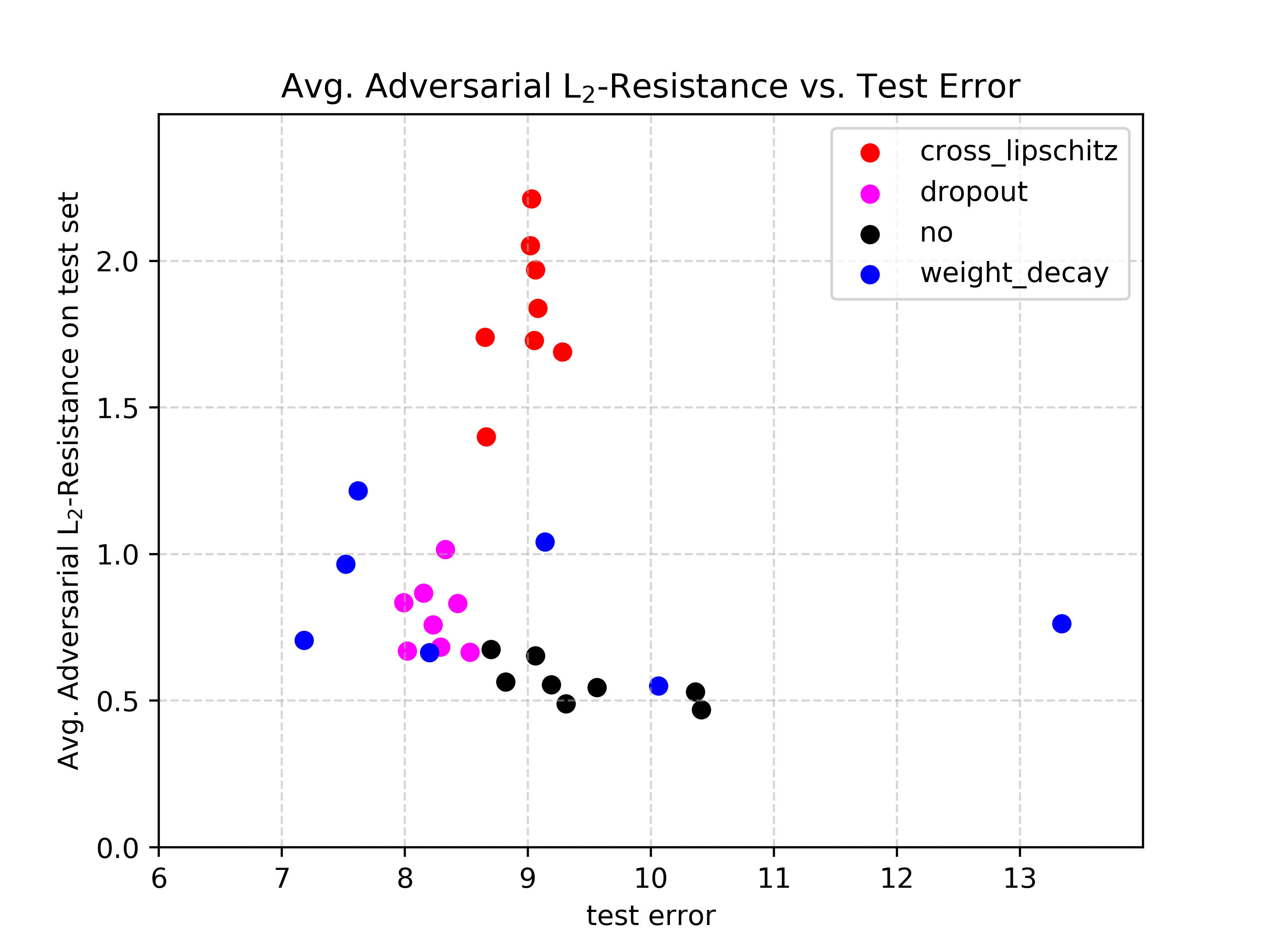}     &\includegraphics[width=0.45\textwidth]{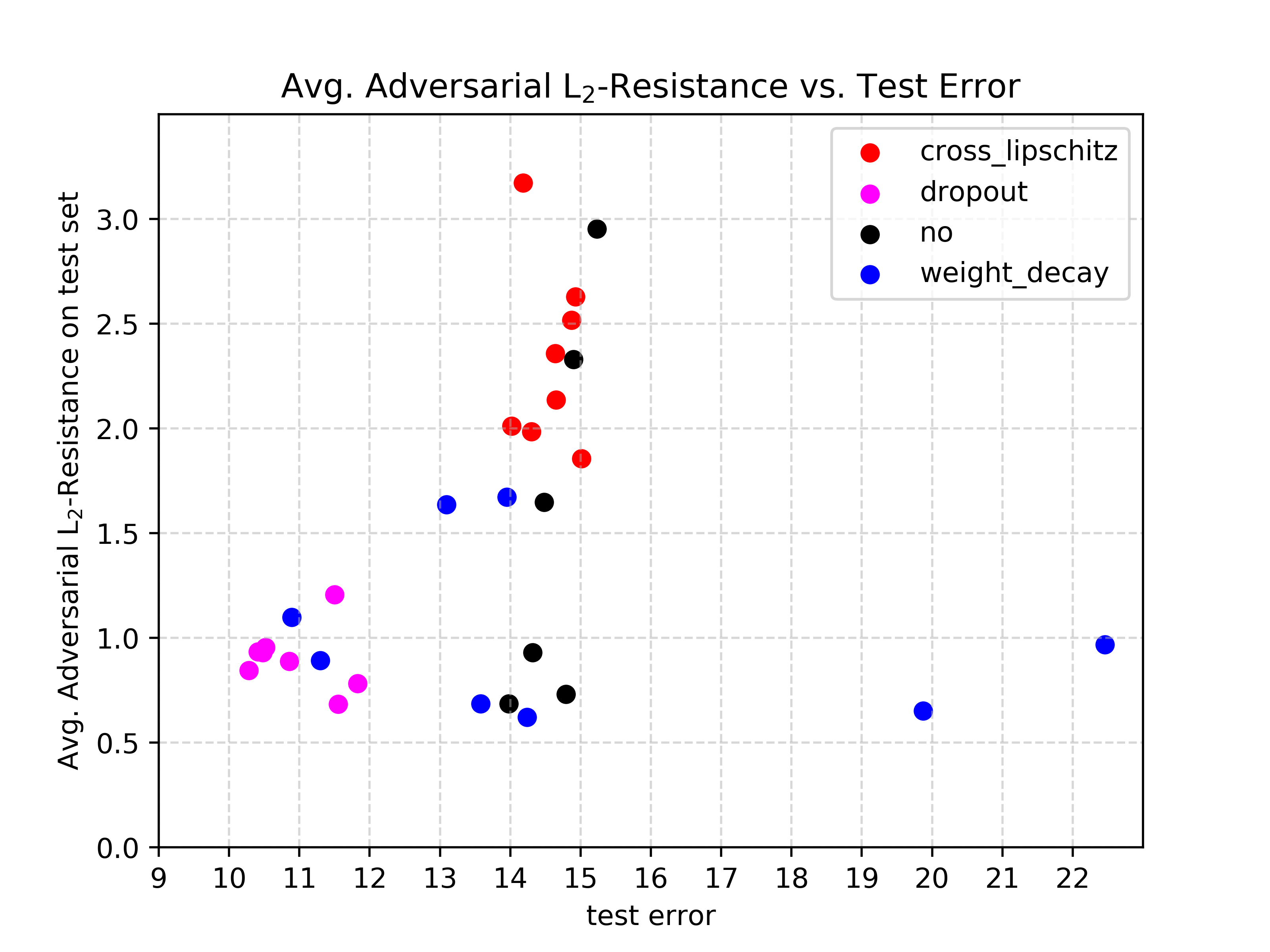}\\
\end{tabular}
\caption{\label{exp:ResNets-plain-augm} Results on CIFAR10 for a residual network with different regularizers. As we only have lower bounds for one hidden layer networks, we can only show upper bounds for adversarial resistance. 
\textbf{Left:} with data augmentation similar to \cite{ZagKom2016}  \textbf{Right:} plain setting }
\end{figure}

%We must note that Cross-Lipschitz regularizer can be successfully combined with other regularization methods. For example, for ResNets without the data augmentation we trained 2 models with both Cross-Lipschitz and dropout. The regularization parameter lambda is in $\{0.001, 0.01\}$ and dropout probability is 0.5. As the result we have both improved the robustness (upper bounds) and the test error \ref{tab:DO-and-CL}.
%\begin{table} 
%	\begin{center}
%		\begin{tabular}{c | c | c}
%			Cross-Lipschitz lambda  & upper bounds & test error \\
%			\hline
%			0.001   & 2.93         & 10.13 \\
%			0.01     & 3.2		   & 10.6
%		\end{tabular}
%	\caption{\label{tab:DO-and-CL} After combining Cross-Lipschitz and dropout (probability 0.5) there are both improvements in the robustness (upper bounds) and the test error}
%	\end{center}
%\end{table}
\fi

\paragraph{Outlook}
Formal guarantees on machine learning systems are becoming increasingly more important as they are used in safety-critical systems. We think that
there should be more research on robustness guarantees (lower bounds), whereas current research is focused on new attacks (upper bounds).
We have argued that our instance-specific guarantees using our local Cross-Lipschitz constant is more effective than using a global one
and leads to lower bounds which are up to 8 times better. A major open problem is to come up with tight lower bounds for deep networks.

\small
\bibliography{PFMH_bib,regul,Literatur}

\begin{thebibliography}{10}

\bibitem{AbaAgaBar2016}
M.~Abadi, A.~Agarwal, P.~Barham, E.~Brevdo, Z.~Chen, C.~Citro, G.~S. Corrado,
  A.~Davis, J.~Dean, M.~Devin, S.~Ghemawat, I.~J. Goodfellow, A.~Harp,
  G.~Irving, M.~Isard, Y.~Jia, R.~J{\'{o}}zefowicz, L.~Kaiser, M.~Kudlur,
  J.~Levenberg, D.~Man{\'{e}}, R.~Monga, S.~Moore, D.~G. Murray, C.~Olah,
  M.~Schuster, J.~Shlens, B.~Steiner, I.~Sutskever, K.~Talwar, P.~A. Tucker,
  V.~Vanhoucke, V.~Vasudevan, F.~B. Vi{\'{e}}gas, O.~Vinyals, P.~Warden,
  M.~Wattenberg, M.~Wicke, Y.~Yu, and X.~Zheng.
\newblock Tensorflow: Large-scale machine learning on heterogeneous distributed
  systems, 2016.

\bibitem{BasEtAl2016}
O.~Bastani, Y.~Ioannou, L.~Lampropoulos, D.~Vytiniotis, A.~Nori, and
  A.~Criminisi.
\newblock Measuring neural net robustness with constraints.
\newblock In {\em NIPS}, 2016.

\bibitem{CarWag2017}
N.~Carlini and D.~Wagner.
\newblock Adversarial examples are not easily detected: Bypassing ten detection
  methods.
\newblock In {\em ACM Workshop on Artificial Intelligence and Security}, 2017.

\bibitem{CisEtAl2017}
M.~Cisse, P.~Bojanowksi, E.~Grave, Y.~Dauphin, and N.~Usunier.
\newblock Parseval networks: Improving robustness to adversarial examples.
\newblock In {\em ICML}, 2017.

\bibitem{DalEtAl2004}
N.~Dalvi, P.~Domingos, Mausam, S.~Sanghai, and D.~Verma.
\newblock Adversarial classification.
\newblock In {\em KDD}, 2004.

\bibitem{DruCun1992}
H.~Drucker and Y.~Le Cun.
\newblock Double backpropagation increasing generalization performance.
\newblock In {\em IJCNN}, 1992.

\bibitem{GooShlSze2015}
I.~J. Goodfellow, J.~Shlens, and C.~Szegedy.
\newblock Explaining and harnessing adversarial examples.
\newblock In {\em ICLR}, 2015.

\bibitem{GuRig2015}
S.~Gu and L.~Rigazio.
\newblock Towards deep neural network architectures robust to adversarial
  examples.
\newblock In {\em ICLR Workshop}, 2015.

\bibitem{HeZhaRen2015}
K.~He, X.~Zhang, S.~Ren, and J.~Sun.
\newblock Deep residual learning for image recognition.
\newblock In {\em CVPR}, pages 770--778, 2016.

\bibitem{HelLon2015}
D.~P. Helmbold and P.~Long.
\newblock On the inductive bias of dropout.
\newblock {\em Journal of Machine Learning Research}, 16:3403--3454, 2015.

\bibitem{SchHoc1995}
S.~Hochreiter and J.~Schmidhuber.
\newblock Simplifying neural nets by discovering flat minima.
\newblock In {\em NIPS}, 1995.

\bibitem{HuaEtAl2016}
R.~Huang, B.~Xu, D.~Schuurmans, and C.~Szepesvari.
\newblock Learning with a strong adversary.
\newblock In {\em ICLR}, 2016.

\bibitem{KosFisSon2017}
J.~Kos, I.~Fischer, and D.~Song.
\newblock Adversarial examples for generative models.
\newblock In {\em ICLR Workshop}, 2017.

\bibitem{KurGooBen2016a}
A.~Kurakin, I.~J. Goodfellow, and S.~Bengio.
\newblock Adversarial examples in the physical world.
\newblock In {\em ICLR Workshop}, 2017.

\bibitem{LiuEtAl2016}
Y.~Liu, X.~Chen, C.~Liu, and D.~Song.
\newblock Delving into transferable adversarial examples and black-box attacks.
\newblock In {\em ICLR}, 2017.

\bibitem{LowMee2005}
D.~Lowd and C.~Meek.
\newblock Adversarial learning.
\newblock In {\em KDD}, 2005.

\bibitem{MooEtAl2016}
S.M. Moosavi-Dezfooli, A.~Fawzi, O.~Fawzi, and P.~Frossard.
\newblock Universal adversarial perturbations.
\newblock In {\em CVPR}, 2017.

\bibitem{PapEtAl2016a}
N.~Papernot, P.~McDonald, X.~Wu, S.~Jha, and A.~Swami.
\newblock Distillation as a defense to adversarial perturbations against deep
  networks.
\newblock In {\em IEEE Symposium on Security \& Privacy}, 2016.

\bibitem{MooFawFro2016}
P.~Frossard S.-M. Moosavi-Dezfooli, A.~Fawzi.
\newblock Deepfool: a simple and accurate method to fool deep neural networks.
\newblock In {\em CVPR}, pages 2574--2582, 2016.

\bibitem{SchSmo2002}
B.~Sch{\"o}lkopf and A.~J. Smola.
\newblock {\em Learning with Kernels}.
\newblock MIT Press, Cambridge, MA, 2002.

\bibitem{ShaYamNeg2016}
U.~Shaham, Y.~Yamada, and S.~Negahban.
\newblock Understanding adversarial training: Increasing local stability of
  neural nets through robust optimization.
\newblock In {\em NIPS}, 2016.

\bibitem{SriEtAl2014}
N.~Srivastava, G.~Hinton, A.~Krizhevsky, I.~Sutskever, and R.~Salakhutdinov.
\newblock Dropout: A simple way to prevent neural networks from overfitting.
\newblock {\em Journal of Machine Learning Research}, 15:1929--1958, 2014.

\bibitem{GTSB2012}
J.~Stallkamp, M.~Schlipsing, J.~Salmen, and C.~Igel.
\newblock Man vs. computer: Benchmarking machine learning algorithms for
  traffic sign recognition.
\newblock {\em Neural Networks}, 32:323--332, 2012.

\bibitem{SzeEtal2014}
C.~Szegedy, W.~Zaremba, I.~Sutskever, J.~Bruna, D.~Erhan, I.~Goodfellow, and
  R.~Fergus.
\newblock Intriguing properties of neural networks.
\newblock In {\em ICLR}, pages 2503--2511, 2014.

\bibitem{ZagKom2016}
S.~Zagoruyko and N.~Komodakis.
\newblock Wide residual networks.
\newblock In {\em BMVC}, pages 87.1--87.12.

\bibitem{ZheEtAl2016}
S.~Zheng, Y.~Song, T.~Leung, and I.~J. Goodfellow.
\newblock Improving the robustness of deep neural networks via stability
  training.
\newblock In {\em CVPR}, 2016.

\end{thebibliography}
\bibliographystyle{plain} 

\end{document}